\documentclass[a4paper,11pt]{article}

\input{macros.tex}

\title{%
  \centerline{%
    \scalebox{0.83}{% adjust this factor as needed
      \bfseries Safe-EF: Error Feedback for Nonsmooth Constrained Optimization
    }%
  }%
}
\date{}
\author[1]{{\bf Rustem~Islamov}}
\author[2,3]{{\bf Yarden~As}}
\author[2,3]{{\bf Ilyas Fatkhullin}}
\affil[1]{University of Basel}
\affil[2]{ETH Z{\"u}rich}
\affil[3]{ETH AI Center}

\begin{document}

\maketitle

\begin{abstract}
    Federated learning faces severe communication bottlenecks due to the high dimensionality of model updates. Communication compression with contractive compressors (e.g., Top-$K$) is often preferable in practice but can degrade performance without proper handling. Error feedback (EF) mitigates such issues but has been largely restricted for smooth, unconstrained problems, limiting its real-world applicability where non-smooth objectives and safety constraints are critical. We advance our understanding of EF in the canonical non-smooth convex setting by establishing new lower complexity bounds for first-order algorithms with contractive compression. Next, we propose \algname{Safe-EF}, a novel algorithm that matches our lower bound (up to a constant) while enforcing safety constraints essential for practical applications. Extending our approach to the stochastic setting, we bridge the gap between theory and practical implementation. Extensive experiments in a reinforcement learning setup, simulating distributed humanoid robot training, validate the effectiveness of \algname{Safe-EF} in ensuring safety and reducing communication complexity.
\end{abstract}

\let\thefootnote\relax
\footnotetext{This work has been accepted to ICML 2025.}

\section{Introduction}\label{sec:intro}

Federated learning is a crucial framework for training machine learning models across distributed environments~\citep{FEDLEARN,FL-big}, where data is naturally stored in a distributed fashion. Formally, such problems can be expressed as
\begin{align}\label{eq:problem}
%\label{eq:problem-wo-constraints}
	&\min_{x \in \cX } f(x) \eqdef \frac{1}{n} \sum_{i=1}^n f_i(x) , 
\end{align}
where $n$ represents the number of workers or machines participating in the training, and $x\in\R^d$ denotes the model parameters to be optimized. The function $f_i\colon \R^d \to \R$ is the local (possibly non-smooth) loss associated with data on worker $i\in [n] \eqdef \{1,\dots,n\}$, and $\cX$ is a subset of $\R^d$. 

This paradigm is particularly valuable in privacy-sensitive and resource-constrained settings, where data remains decentralized, and collaboration is achieved without requiring direct data sharing. For instance, consider a fleet of robots that operate in homes \citep{kalashnikov2018scalable,brohan2022rt}. In such settings, traditional centralized learning approaches are impractical, as transmitting raw sensory data from each robot to a central server would pose severe privacy risks and require enormous bandwidth. Furthermore, these robots must adapt to diverse household environments, necessitating personalized learning while still benefiting from collective experience across the fleet. Despite its advantages, distributed training faces significant communication bottlenecks due to the high dimensionality of model updates. This challenge necessitates the development of communication-efficient algorithms.

\paragraph{Communication compression with Top-$K$.} 
%{\bf Communication compression with Top-$K$.} 
One prominent strategy to reduce communication costs is the communication compression technique, which applies possibly randomized compression to updates prior to transmission. One of the most practical and versatile classes of compression operators are those that satisfy the contractive property:

\begin{equation}
    \E{\|\cC(x) - x\|^2} \le (1-\delta)\|x\|^2 \qquad \text{for all } x \in \R^d , \notag 
\end{equation}
where $\delta \in (0, 1]$ represents the accuracy of the compression. Prominent examples are Top-$K$ sparsifier that preserves $K$ largest components of vector $x$ in magnitude, and random sampling methods such as Rand-$K$ that preserves a subset of $K$ components of $x$ chosen uniformly at random. Although both Top-$K$ and Rand-$K$ are contractive with $\delta \geq K/d$, methods utilizing Top-$K$ operator are often empirically superior due to their greedy nature \citep{you2016asynchronous}. 

\paragraph{Non-smooth challenges.} 
%{\bf Non-smooth challenges.} 
The majority of works focusing on communication compression assume that the objective function is smooth, i.e., differentiable with Lipschitz continuous gradient, simplifying theoretical analysis \citep{stich2018sparsified,richtarik2021ef21}. However, this assumption limits the applicability of developed methods to many real-world problems, where non-smooth functions frequently arise. For instance, consider problems involving ReLU activations \citep{glorot2011deep} or clipped objectives such as those in proximal policy optimization \citep[PPO, ][]{schulman2017proximal}. This motivates the first key question of our study:

\begin{tcolorbox}[colback=white,colframe=green!50!black]	
\textbf{Question 1:} What are the limits of compressed gradient methods in the non-smooth regime?
\end{tcolorbox}

To illustrate the challenges of designing meaningful methods with contractive compressors like Top-$K$, we present a non-convergence example for vanilla compressed gradient descent (\algname{CGD}) in the non-smooth setting. Consider 
\begin{equation}\label{eq:CGD}
\text{\algname{CGD}} \qquad  x^{t+1} = x^t - \frac{\gamma}{n} \sum_{i=1}^n \cC(f_i^{\prime}(x^t) ) ,
\end{equation}
where $f_i^{\prime}(x^t) \in \partial f_i(x^t)$ is a subgradient of $f_i$ and $\gamma \geq 0$ is a stepsize. 

\begin{example}\label{example:div_CGD}
	For any $n\geq 1$, there exists a specific instance of problem \eqref{eq:problem} where $\cX = \R^2$, and $f(x) = \|x\|_1$ is non-smooth, convex, and $1$-Lipschitz continuous. For this instance, with some initial vector $x^0 \in \R^2$, the iterates of \algname{CGD} \eqref{eq:CGD} applied with the Top-$1$ compressor and any stepsize $\gamma \geq 0$, satisfy 
	$$
	f(x^t) - \min_x f(x) = 1 + \frac{\gamma}{2} \qquad \text{for any } t \geq 0. 
	$$
\end{example}
This example implies that running vanilla \algname{CGD} with the Top-$1$ compressor even on a simple non-smooth problem may yield no improvement. It is remarkable that this failure occurs even in the identical data regime $f_i = f$ for all $i\in [n]$, the setting where \algname{CGD} is known to converge in smooth case \citep{nesterov2012efficiency,nutini2015coordinate,beznosikov2023biased}. The idea of the construction in \Cref{example:div_CGD} is that due to a rapid change of the gradients $f^{\prime}$ in consecutive iterations, \algname{CGD} consistently ignores the direction of the second component of $x^t$, which results in a pathological cyclic behavior. See~\Cref{fig:teaser} for an illustration. 

%\citep{lacoste2013block,fercoq2015accelerated,lu2018accelerating}. 

\begin{figure}
    \centering
    \includegraphics{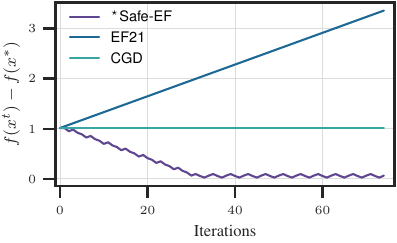}
    %\vspace{-3mm}
    \caption{Non-convergence of \algname{CGD}, divergence of \algname{EF21} and convergence of \algname{Safe-EF} for the problem $f(x) = \|x\|_1$, $i=1, d=2$ used in the proofs of \Cref{example:div_CGD,example:div_ef21} with Top-$1$ compressor. We run all algorithms for $T=10^3$ iterations with $x^0 = (\gamma/2, -1)^{\top}$, $\gamma = 1/\sqrt{T}$, and $v^0 = (1, 1)^{\top}$ (for \algname{EF21}). $^\star$\algname{Safe-EF} coincides with \algname{EF14} \citep{seide20141} in this example.}
    \label{fig:teaser}
\end{figure}

\paragraph{Error feedback can make things worse!}
%{\bf Error feedback can make things worse!} 
A common remedy for non-convergence issues of compressed gradient methods is error feedback (EF), a mechanism that has inspired several variants \cite{seide20141,richtarik2021ef21,fatkhullin2024momentum,gao2024econtrol}. Among these, \algname{EF21} is a recent approach with state-of-the-art performance guarantees in smooth optimization due to \citet{richtarik2021ef21}:
\begin{align}
    \algname{EF21} \qquad &
    \begin{aligned}
        x^{t+1} &= x^t - \gamma \, v^t , \qquad v^t = \frac{1}{n} \sum_{i=1}^n v_i^t ,  \\
        v_i^{t+1} &= v_i^t + \mathcal{C}(f_i^{\prime}(x^{t+1}) - v_i^t) .
    \end{aligned}
    \label{eq:EF21}
\end{align}
where $f_i^{\prime}(x^{t+1}) \in \partial f_i(x^{t+1})$ is a subgradient of $f_i$ and $v_i^t$ is a local gradient estimator at each worker. While \citet{richtarik2021ef21} only analyze this algorithm in the smooth non-convex case, we extend its analysis to smooth convex setup in Appendix~\ref{sec:Proj_EF21}. However, we show that, surprisingly, \algname{EF21} fails to converge on the same problem as \algname{CGD}.

\begin{example}\label{example:div_ef21}
	Consider the problem instance from Example~\ref{example:div_CGD}. For this instance, with some initial vectors $x^0, v^0 \in \R^2$, the iterates of \algname{EF21} \eqref{eq:EF21} applied with the Top-$1$ compressor and any stepsize $\gamma \geq 0$ satisfy
	$$
	f(x^t) - \min_x f(x) = 1 + \frac{\gamma}{2} + t \, \gamma \qquad \text{for any } t \geq 0. 
	$$
\end{example}

This example shows that \algname{EF21} does not converge for non-smooth problems despite achieving an excellent performance in smooth case, see \Cref{thm:Proj_EF21}, and reaching the optimal iteration complexity in smooth non-convex optimization \cite{huang2022lower}. Moreover, if we pick the classical stepsize $\gamma = 1/\sqrt{T}$, \algname{EF21} diverges from the optimum with a rate $\Omega(t / \sqrt{T}) \approx \sqrt{T}$ for $t \approx T$, which is even worse than \algname{CGD}. We show the divergence in \Cref{fig:teaser}, where we also observe that another EF variant, \algname{EF14}\footnote{Our \algname{Safe-EF} method presented in \Cref{alg:ef14} reduces to \algname{EF14} in unconstrained setting with $\cC_{0} = \mathbf{Id}$.}, \citep{seide20141} converges without problems.
We find such stark difference surprising in light of the equivalence of \algname{EF21} and \algname{EF14}, established under additivity assumption of $\cC$ \citep{richtarik2021ef21}. The catch is that Top-$1$ is not additive, and thus the equivalence does not hold here. %\algname{EF21} fails because it "assumes" that when consecutive iterates are close, gradients are also close---an assumption that does not hold without smoothness in general. 
\begin{center}
\textit{Motivated by this fairly toy example, we find it important to understand \\ error feedback in non-smooth setup, and aim to study \algname{EF14}.}
\end{center}

\paragraph{Safety considerations.} 
%{\bf Safety considerations.} 
In addition to these challenges, safety constraints play a critical role in real-world applications \citep{altman1999constrained}. Ensuring solutions satisfy feasibility requirements is essential, particularly in scenarios like federated reinforcement learning (FedRL) \citep{nadiger2019federated,qi2021federated,jin2022federated}. Despite their importance, constrained optimization with communication compression remains under-explored. Although some work develop methods assuming $\cX$ is simple, i.e., using projection~\citep{fatkhullin2021ef21} or linear minimization \citep{nazykov2024stochastic} oracles, they crucially rely on smoothness. Moreover, the applications in Safe FedRL motivate us to pay attention to problems with more complex constraints of the form 
\begin{equation}\label{eq:constraints}\textstyle
\cX \eqdef \left\{x\in\R^d \mid g(x) \eqdef \frac{1}{n} \sum_{i=1}^n g_i(x) \leq 0 \right\} ,
\end{equation}
where $g_i\colon \R^d \to \R$ defines a constraint for worker $i$. 
\begin{tcolorbox}[colback=white,colframe=green!50!black]	
\textbf{Question 2:} Can we design a provably convergent compressed gradient method with a Top-$K$ compressor for non-smooth constrained problems?
\end{tcolorbox}

Perhaps, the most common approach to solve \eqref{eq:problem}  with \eqref{eq:constraints} in non-distributed optimization ($n=1$) is to reformulate it as a saddle point problem, which is then solved by primal-dual methods~\citep{nemirovski2004prox,hamedani2021primal}. This approach is popular in practice \cite{ding2020natural,moskovitz2023reload,ding2024last,muller2024truly} and has rich theory, e.g., \citep{boob2023stochastic,boob2024optimal,zhang2022solving}. However, such methods have several limitations. First, they are known to be sensitive to the tuning of the initial dual variable (e.g., the experiments and discussion in \Cref{sec:primal_dual}) and often require an estimate of the upper bound of the optimal dual variable. Second, their theoretical justification often requires projecting both primal and dual variables onto an unknown bounded set, which is not aligned with practical implementations. In the context of EF-type methods, this projection requirement implies several algorithmic and technical challenges because only certain smooth variants of EF seem to be compatible with projection, e.g., \citep{fatkhullin2021ef21}. An alternative is to adopt a primal only approach, e.g., switching subgradient~\cite{polyak_general_1967, lan2020algorithms, ma2020quadratically, huang2023oracle, jia2022first}, methods based on the velocity field \citep{yu2017online,muehlebach2022constraints,schechtman2022first,kolev2024online}, or level-set methods \citep{lin2018level, boob2024level}. Primal methods have also been used in (non-distributed) RL applications, e.g., \cite{xu2021crpo,chen2021primal,jordan2024independent,li2024faster}. The key advantage of such primal schemes is their simplicity and convergence under mild assumptions without the need for the estimation of dual variables.

\section{Contributions}
\begin{itemize}
    \item First, we establish a $\Omega\left(\frac{M R}{\sqrt{\delta T}} \right)$ convergence lower bound for non-smooth convex distributed optimization with contractive compressors for function suboptimality gap and a constraint violation. Here $T$ is the iteration count, $R$ is the initial distance to the optimum, $M$ bounds the norm of subgradients of $f_i$, and $\delta \in (0, 1]$ is the compression accuracy.
    
    \item Next, we propose \algname{Safe-EF} (\Cref{alg:ef14}), an extension of \algname{EF14} \cite{seide20141} incorporating safety constraints \eqref{eq:constraints} and bidirectional compression including the workers to server compressor $\cC_{0}$. \algname{Safe-EF} provably works in non-smooth distributed settings and efficiently minimizes the objective function, while controlling the constraint violation. We prove the convergence rate of \algname{Safe-EF} matching the above-mentioned lower bound up to a numerical constant under a constant accuracy of the server compression $\cC_0$. It seems our upper bound is new even when $g(x) \equiv 0$ and $\cC_0 = \mathbf{Id}$.
    \item We further study \algname{Safe-EF} in practically relevant stochastic scenarios, where exact subgradients and function evaluations are unavailable and need to be estimated. We establish high probability bounds with a mild logarithmic dependence on failure probability, which is significant even without compression, since our bounds feature the distance to the optimum $R$ instead of the diameter of the set, which is not bounded in our set-up.
    \item Finally, we conduct extensive experiments and ablation studies of \algname{Safe-EF}, putting the method to the test on a challenging task of distributed humanoid robot training and providing important practical insights into the performance of non-smooth EF methods.
\end{itemize}

%\begin{figure*}[!t]
%\begin{minipage}{0.49\textwidth}

%\section{Problem formulation}
\section{Assumptions and Communication Protocol}
%We consider the following distributed constrained optimization problem.
%\begin{align}\label{eq:problem}
%\begin{aligned}
%	&\min_{x \in \cX} f(x) \eqdef \frac{1}{n} \sum_{i=1}^n f_i(x) ,%\\ \text{where } \cX \eqdef &\left\{x\in\R^d \mid g(x) \eqdef \frac{1}{n} \sum_{i=1}^n g_i(x) \leq 0 \right\},   
%\end{aligned}
%\end{align}
%where $\cX$ is defined in \eqref{eq:constraints}. We denote the optimal solution to problem~\eqref{eq:problem} by $x^*.$

%\paragraph{Assumptions.} In this work, we focus on a centralized setting where communication between workers occurs through a central server. Next, we define the class of functions that we consider.

We consider distributed constrained optimization problem \eqref{eq:problem} with a constraint \eqref{eq:constraints}, and denote the optimal solution to this problem by $x^*.$ Unless specified otherwise, we denote by $\|\cdot\|$ the Euclidean norm in $\R^d$.

\begin{assumption}\label{asmp:convexity}
    We assume that $f_i$ and $g_i$ are convex for all $i\in[n],$ namely, for all $x,y\in\R^d$ we have 
    \begin{align}
        f_i(y) \ge f_i(x) + \<f_i^\prime, y-x> \quad \forall f^\prime_i \in \partial f_i(x),\quad 
        g_i(y) \ge g_i(x) + \<g_i^\prime, y-x> \quad \forall g_i^\prime \in \partial g_i(x).
    \end{align}
\end{assumption}
Each worker $i$ has access to the oracles $O_{f_i,i}(x)$ and $O_{g_i,i}(x)$, which return the subgradients $f_i^\prime\in \partial f_i(x)$, $g_i^\prime \in \partial g_i(x)$, and the function values $f_i(x)$, $g_i(x)$ respectively for any $x\in\R^d$. We assume bounded subgradient, which is a common assumption in non-smooth optimization \citep{nesterov2018lectures}

\begin{assumption}\label{asmp:bounded_grad}
    We assume that $f_i$ and $g_i$ have $M$-bounded subgradients, i.e. for any $x\in\R^d$ and $i\in[n]$ we have
    \begin{equation}
       \max\left\{ \|f_i^\prime(x)\|, \, \|g_i^\prime(x)\|  \right\} \le M . \quad
    \end{equation}
\end{assumption}
We let the function classes $\cF_{R,M}$ and $\mathcal G_{R M}$ denote the set of all functions satisfying Assumptions~\ref{asmp:convexity}-\ref{asmp:bounded_grad} for any underlying dimension $d$ and a given initialization $x^0 \in \R^d$ such that $\|x^0 - x^*\|\le R.$ We denote by $\mathcal H_{R,M}$ the class of problems of form \eqref{eq:problem}, \eqref{eq:constraints}, where functions $\{f_i\}_{i=1}^n$ and $\{g_i\}_{i=1}^n$ are taken from $\cF_{R,M}$ and $\mathcal G_{R, M}$ respectively.  

\paragraph{Compression operators.} We focus on the class of algorithms using contractive compressors.

\begin{definition}\label{def:contractive_comp}
    We say that a (possibly randomized) mapping $\cC\colon \R^d \to \R^d$ is a contractive compression operator if for some constant $\delta \in (0,1]$ it holds
    \begin{equation}\label{eq:contractive_comp}
        \E{\|\cC(x) - x\|^2} \le (1-\delta)\|x\|^2.
    \end{equation}
\end{definition}
Beyond Top-$K$ and Rand-$K$ mentioned in \Cref{sec:intro}, examples satisfying \eqref{eq:contractive_comp} include sparsification \citep{alistarh2018convergence, stich2018sparsified,islamov2021distributed} and quantization \citep{wen2017terngrad, bernstein2018signsgd, horvath2022natural, compagnoni2025unbiased} techniques, and low-rank approximations \citep{vogels2019powersgd, qian2021basis, islamov2023distributed}. We refer to \citep{beznosikov2023biased, safaryan2021fednl} for further examples. We denote by $\mathbb{C}(\delta)$ the set of all $\delta$-contractive compressors.

\paragraph{Algorithm class.} 
%{\bf Algorithm class.}
We follow \citet{huang2022lower} to introduce the class of algorithms of interest. We consider a centralized and synchronous algorithm $A$, where: $i)$ workers are restricted to communicating directly with a central server and cannot exchange information with one another directly; $ii)$ all iterations are synchronized, meaning all workers begin each iteration simultaneously. In this setup, each worker $i$ maintains a local copy of the model, denoted as $x_i^t$, at iteration $t$. The output $\hat{x}^t$ of the algorithm $A$ after $t$ iterations can be expressed as any linear combination of all previous local models, namely,
\begin{equation}
    \hat{x}^t \in {\rm span}\left(\{x_i^s \colon 0 \le s \le t, 1 \le i \le n\}\right).
\end{equation}
We additionally require that the algorithm $A$ satisfies the ``zero-respecting'' property \citep{carmon2020lower, lu2021optimal}. This ensures that the number of non-zero entries in a worker's local model can only increase through local subgradient queries, or synchronization with the central server. This property is upheld by a broad range of existing distributed optimization algorithms \citep{tang2019doublesqueeze, xie2020cser, richtarik2021ef21, gao2024econtrol}. In addition to these properties, the algorithm $A$ must support communication compression. To achieve this, 
%the server is equipped with a compressor $\cC_0$, and each worker $i\in[n]$ is equipped with a compressor $\cC_i$. In the case $\cC_0 = {\rm Id}$, the algorithm $A$ performs unidirectional compression on messages from workers to the server. In the case $\cC_0 \neq {\rm Id}$, the algorithm employs bidirectional compression, compressing messages in both directions. The formal definition of this algorithm class with unidirectional/bidirectional compression is provided below.
each worker $i\in[n]$ is equipped with a compressor $\cC_i$. The formal definition of this algorithm class with worker to server compression is provided below, see \Cref{sec:lower_bound_uni_proof} for details. 

%\begin{definition}
%    Given compressors $\{\cC_0, \cC_1, \dots, \cC_n\}$, we denote $\cA_{\{\cC_i\}_{i=0}^n}^B$ as the class of all centralized, synchronous, zero-respecting algorithms that support bidirectional compression, where $i)$ compressor $\cC_i, i\in[n]$, is applied to messages from worker $i$ to the server, and $ii)$ compressor $\cC_0$ is applied to messages from the server to all workers. When $\cC_0 = {\rm Id}$, we use $\cA_{\{\cC_i\}_{i=1}^n}^U$ to denote the class of algorithms that support unidirectional compression for simplicity. 
%\end{definition}

\begin{definition}
    Given compressors $\{\cC_1, \dots, \cC_n\}$, we denote $\cA_{\{\cC_i\}_{i=1}^n}^U$ as the class of all centralized, synchronous, zero-respecting algorithms that support unidirectional compression, where compressor $\cC_i, i\in[n]$, is applied to messages from worker $i$ to the server. 
\end{definition}

\section{Main Results}
%Next, we present the main theoretical results of our work.
We start by presenting our first main contribution, which is the lower iteration/communication complexity bound for a class of first-order compressed gradient methods. 

\subsection{Lower Bound}
Given a problem $h := (\{f_i\}_{i=1}^n, \{g_i\}_{i=1}^n) \subseteq \mathcal H_{R,M}$, subgradient/function value oracles $\{\cO_{f_i,i}\}_{i=1}^n$, $\{\cO_{g_i,i}\}_{i=1}^n$, compressors $\{\cC_i\}_{i=1}^n \subseteq \mathbb{C}(\delta)$, and an algorithm $A \in \cA_{\{\cC_i\}_{i=1}^n}^U$, let $\hat{x}_{A, T} \eqdef \hat{x}_{A, \{f_i\}_{i=1}^n, \{g_i\}_{i=1}^n, \{\cC_i\}_{i=1}^n, T}$ represent the output of algorithm $A$ after at most $T$ oracle queries and communication rounds per worker. We define the minimax convergence measure
\begin{equation*}
     \inf_{A} \sup_{\{\cC_i\}_{i=1}^n} \sup_{h \in \mathcal H_{R,M}} \left\{ \E{f(\hat{x}_{A, T}) - f(x^*)}, \E{g(\hat{x}_{A, T})} \right\} .
\end{equation*}
We do not require operators $\{\cC_i\}_{i=1}^n$ to be neither distinct nor independent, and parameter $\delta$ can be utilized by the algorithm $A.$ Our first contribution is the lower bound for algorithms that support unidirectional compression. 

\begin{theorem}\label{th:lower_bound_uni}
    For any $R, M > 0, n \ge 2,$ $\delta \le 0.3, T \ge \delta^{-2}$ there exists a problem $h \subseteq \mathcal{H}_{R,M}$, oracles $\{\cO_{f_i,i}\}_{i=1}^n$, $\{\cO_{g_i,i}\}_{i=1}^n$, compressors $\{\cC_i\}_{i=1}^n\subseteq \mathbb{C}(\delta)$, and the starting point $x^0 = 0$ such that for any first-order algorithm $A \in \cA_{\{\cC_i\}_{i=1}^n}^U$ run for $T \le d$ iterations from $x^0$, satisfies  
   % \begin{equation}
   %     \hspace{-5pt}\E{f(\hat{x}_{A, \{f_i\}_{i=1}^n, \{\cC_i\}_{i=1}^n, T }) - f(x^*)} \ge \Omega \left(\frac{MR}{\sqrt{\delta T}}\right).
  %  \end{equation}
    \begin{align}\label{eq:LB_unidirectional}
    \begin{aligned}
        \hspace{-5pt}\E{f(\hat{x}_{A, T }) - f(x^*)} &\geq \Omega \left(\frac{RM}{\sqrt{\delta T}}\right),\quad \text{ and } \\ 
       \hspace{-5pt}\E{g(\hat{x}_{A, T })} &\geq \Omega\left(\frac{RM}{\sqrt{\delta T}}\right) . 
    \end{aligned}
    \end{align}
\end{theorem}

When $\delta=1$ and $g \equiv 0$, indicating no compression and no constraints, \eqref{eq:LB_unidirectional} recovers the classical lower bounds for non-smooth convex optimization \citep{nemirovskij1983problem,nemirovski1994efficient,nesterov2014introductory,braun2017lower,scaman2018optimal}. However, when worker to server compression is large, the convergence rate degrades by a factor of $1/\sqrt{\delta}$. Similar degradation appears in the constraint violation. An interesting implication of \Cref{th:lower_bound_uni} is that the convergence rate does not improve when increasing the number of workers $n$, which is different from prior work in smooth stochastic optimization \cite{huang2022lower,he2023lower}. The key idea of the proof is to extend and modify the ``worst-case'' function from \citep{nesterov2014introductory} and account for compression in the distributed setting, specifically, we use for all $i\in [n]$
%\vspace{-1mm}
 \begin{align*}
    \begin{aligned}
        f_i(x) &\eqdef 
            C\cdot \max\limits_{1\le j\le T} x_j + \frac{\mu}{2}\|x\|_2 \cdot \max\left\{ \|x\|_2 ; \frac{R}{2} \right\} ,  \\
        g_i(x) &\eqdef f_i(x) - \min_{x\in \R^d} f_i(x) , 
    \end{aligned}
  \end{align*}
where $C, \mu  > 0$ are some constants depending on the bound of subgradients $M$ and the compression level $\delta$. We refer to \Cref{sec:lower_bound_uni_proof} for the formal proof.

%We modify the second term in definition of $f_i$ to ensure bounded gradients hold without bounded domain, and also define $g_i(x)$ to handle constraint case. 

%Notably, the construction of this ``hard'' function differs significantly from those used in prior studies \citep{arjevani2015communication, scaman2018optimal, kovalev2024lower} on lower bounds in distributed optimization. In the proof of \Cref{th:lower_bound_uni}, we demonstrate that any subgradient-based algorithm can increase the number of non-zero model parameters by one at most after each communication round.

\subsection{\algname{Safe-EF} Method}\label{sec:convergence_theory}

In this section, we describe \algname{Safe-EF}, our main algorithm detailed in \Cref{alg:ef14}, which addresses two main challenges simultaneously: handles \textit{non-smoothness} and \textit{constraints}. The distinct feature of our method is a dynamical switch between the subgradients of the objective $f_i$ and those of the constraints $g_i$ depending on if the constraint violation exceeds a predefined threshold $c.$ To implement this, workers compute the constraint violations $g_i(x^t)$ and communicate them to the server. This process does not increase communication overhead, as it requires transmitting only a single float per iteration. Equipped with this switching rule, we use \algname{EF14} \citep{seide20141} type updates to limit the communication overhead of sub-gradients from workers to server. Furthermore, we additionally enhance \algname{Safe-EF} with server to workers compression using a ``primal'' \algname{EF21} variant, \algname{EF21-P}, due to \citet{gruntkowska2023ef21p}, which compresses the difference between two estimates of the model parameters $w^{t+1}$ and $x^t$.\footnote{In fact, it was noted by \citet{gruntkowska2023ef21p} that a pure \algname{EF21-P} used at the server level can be reformulated as \algname{EF14} on the worker level. However, we only use \algname{EF21-P} formulation for algorithmic presentation and design the convergence proof using \algname{EF14} formulation.}

\begin{figure}[!t]
%\vspace{-6mm}   
    \begin{algorithm}[H]
       \caption{\algname{Safe-EF} with bidirectional compression}
       \label{alg:ef14}
        \begin{algorithmic}[1]
       \State {\bfseries Input:}  $w^0 = x^0$, $\{\cC_i\}_{i=1}^n$, $\gamma, c > 0$,  $e_i^0=0$
       \For{$t=0, \dots, T-1$}
       \For{$i=1, \dots, n$ {\bfseries in parallel}}
       \State Send $g_i(x^t)$ to server  
       \Comment{cheap one float comm.}
       \EndFor
       \State Send $g(x^t) = \frac{1}{n}\sum_{i=1}^n g_i(x^t)$ to workers
       \For{$i=1,\dots,n$ {\bfseries in parallel}}
       \State Compute $h_i^t = f_i^\prime(x^t)$ {\bfseries if} $g(x^t) \le c$ {\bfseries else} $g_i^\prime(x^t)$
       \State Send $v_i^t = \cC_i(e_i^t + h_i^t)$ to server
       \State Compute $e_i^{t+1} = e_i^t + h_i^t - v_i^t$
       \EndFor
       \State Compute $v^t = \frac{1}{n}\sum_{i=1}^n v_i^t$ and $w^{t+1} = w^t - \gamma v^t$
       \State Compute $x^{t+1} = x^t + \cC_0(w^{t+1} - x^t)$ 
       \State Send $\cC_0(w^{t+1} - x^t)$ to workers
       %\State Compute $v^t = \frac{1}{n}\sum_{i=1}^n v_i^t$ and $x^{t+1} = x^t - \gamma v^t$
       \EndFor
    \end{algorithmic}
    \end{algorithm}
\end{figure}

%Second, we incorporate EF mechanisms on both the server and worker sides to effectively manage compression errors. Specifically, we utilize \algname{EF14}-style Error Feedback \citep{seide20141} for both the server and workers. 

%Importantly, EF applied at the server level with contractive compressor $\cC_{0}$ can be reformulated in the same manner as on the workers, and was analyzed in prior works \citep{gruntkowska2023ef21p}.
%\vspace{-1mm}
\subsection{Convergence Upper Bound}

In our next theorem, we provide the convergence guarantees for \algname{Safe-EF} summarized in \Cref{alg:ef14}. The set $\cB$ denotes all iteration counters when the constraint violation is below the threshold $c,$ i.e., $$\cB \eqdef \left\{ t\in [T-1] \mid  g(x^t) \leq c \right\}.$$

\begin{theorem}\label{th:ef14_bidirectional}
    Assume Assumptions~\ref{asmp:convexity}-\ref{asmp:bounded_grad} hold, the server and workers use compressors $\cC_0 \in \mathbb{C}(\delta_{\rm s}), \{\cC_i\}_{i=1}^n \subseteq \mathbb{C}(\delta).$ Then there exist a choice of stepsize $\gamma$ and threshold $c$ such that the iterates of \algname{Safe-EF} with bidirectional compression satisfy 
    \begin{align}\label{eq:ef14_bidirectional}
    \begin{aligned}
        \E{f(\overline{x}^T) - f(x^*)} &\le \cO\left(\frac{RM}{\sqrt{\delta_{\rm s}\delta T}}\right),\quad \text{ and }\\ \E{g(\overline{x}^T)} &\le \cO\left(\frac{RM}{\sqrt{\delta_{\rm s}\delta T}}\right),
    \end{aligned}
    \end{align}
    where $\overline{x}^T \eqdef \frac{1}{|\cB|}\sum_{t\in\cB}x^t.$  %is defined in \eqref{eq:def_x_bar_t_bi}.
\end{theorem}

The proof of the theorem is detailed in \Cref{sec:proof_bidirectional}, where we also give explicit choice of $\gamma$ and $c$. Next, we discuss the obtained result in several special cases as well as the main difficulties in the convergence proof.

\paragraph{Single-node training with no compression.} 
%{\bf Single-node training with no compression.} 
In the special case where $n=1$ and $\delta_s = \delta=1$, corresponding to the non-distributed setting without compression, \eqref{eq:ef14_bidirectional} recovers the rates in  \citep{nesterov2018lectures,lan2020algorithms}. %When compression is applied, the convergence rate for both functional suboptimality and constraint gaps is slowed by a factor of $\sqrt{\delta}.$

\paragraph{No constraints, i.e., $g \equiv 0$, and $\cC_0 \equiv \mathbf{Id}$.} 
%{\bf No constraints, i.e., $g \equiv 0$, and $\cC_0 \equiv \mathbf{Id}$.} 
In this case, our algorithm, \algname{Safe-EF}, simplifies to the well-known \algname{EF14} method \citep{seide20141}. \algname{EF14} was previously analyzed in the non-smooth setting for single-node training ($n=1$) by \citet{Karimireddy_SignSGD}. \Cref{th:ef14_bidirectional} extends the analysis to the distributed setup. Notably, the convergence rate is consistent with that presented in their work in this special case.

\paragraph{Unidirectional compression.} 
%{\bf Unidirectional compression.} 
Next, we consider the setting with unidirectional compression, i.e., $\delta_{\rm s}=1$ and $\cC_0 \equiv \textbf{Id}$. We observe that both the functional suboptimality gap and constraint violation diminish at a rate of $\cO(\nicefrac{1}{\sqrt{\delta T}})$, consistent with the lower bound established in \Cref{th:lower_bound_uni}, thereby confirming the optimality of \algname{Safe-EF} assuming $\delta_{\rm s}$ is a numerical constant independent of $d$ and $K$. 

\paragraph{Bidirectional compression.} 
%{\bf Bidirectional compression.} 
Now we discuss the setting when the compression is applied in both directions. It is worth noting that most prior studies focus on a more restricted class of compressors, such as absolute compressors \citep{tang2019doublesqueeze} or unbiased compressors \citep{philippenko2021preserved, gruntkowska2023ef21p, gruntkowska2024improving,tyurin20232direction}, in the bidirectional setting. In contrast, our work does not impose any additional constraints on the compressors. Other related work considers only server to worker compression \citep{sokolov2024marina}, while often compression in both directions is important. The convergence rate in \eqref{eq:ef14_bidirectional} highlights a slowdown by a factor of $\sqrt{\delta_{\rm s}\delta}$, which aligns with similar dependencies observed in prior works on smooth distributed optimization \citep{fatkhullin2021ef21}. It remains an open question whether the dependence on the compression levels $\delta$ and $\delta_{\rm s}$ can be improved in the non-smooth setting. Perhaps, this dependency could potentially be reduced from $\sqrt{\delta_{\rm s}\delta}$ to $\sqrt{\delta}+\sqrt{\delta_{\rm s}}$ by incorporating multiple communication rounds per iteration, similar to the approach in \citep{huang2022lower}. However, this procedure can be impractical since $\lceil{K/\delta_s\rceil}$ coordinates are communicated at every iteration as observed in \citep{fatkhullin2024momentum}, and we leave the study of this strategy for future work.

\paragraph{Key theoretical challenges.} 
%{\bf Key theoretical challenges.} 
We emphasize that controlling constraints significantly complicates the analysis compared to prior work~\citep{Karimireddy_SignSGD}, which is limited to the unconstrained, unidirectional, non-distributed setting. A key novelty of our analysis lies in demonstrating that an appropriate choice of the stepsize $\gamma$ and threshold $c$ ensures that the number of iteration counters in $\cB$ with constraint violations below $c$ is sufficiently large to guarantee progress in reducing functional suboptimality. In particular, it is not empty and thus $\overline{x}^T$ is well-defined.

\paragraph{Communication complexity with Top-$K$.} 
%{\bf Communication complexity with Top-$K$.} 
In a unidirectional case with $\cC_i$ is Top-$K$ and $\cC_{0} \equiv \mathbf{Id}$, the total communication complexity is\footnote{We omit the numerical constants and logarithmic factors in comparison.}
\begin{equation}
    \underbrace{K}_{\text{floats per iteration}} \times \underbrace{\frac{R^2M^2}{\delta\varepsilon^2}}_{\# \text{ iterations}} \le \frac{KR^2M^2}{\frac{K}{d}\varepsilon^2} = \frac{dR^2M^2}{\varepsilon^2},
\end{equation}
where we utilize the condition $\delta \ge \frac{K}{d}$ for Top-$K$. This finding indicates that the communication complexity of \algname{Safe-EF} aligns with that of parallel switching subgradient method (\algname{Safe-EF} without compression) in the worst-case scenario. However, an improvement is possible when $\delta > \frac{K}{d}$, which occurs if the entries differ substantially in magnitude \citep{beznosikov2023biased}.

%{\bf Key Steps of the Proof.} 
\paragraph{Key Steps of the Proof.} 
Our convergence proof builds on the ``virtual iterates'' construction of \citet{stich2019error} (see \Cref{eq:def_x_hat_t}). In \Cref{lem:lemma2_bi}, we then derive a unified bound controlling both the function sub-optimality and the constraint violation. Crucially, by enforcing appropriate choices of the step size $\gamma$ and threshold $c$, we show that this bound can be made small enough. The same lemma also guarantees that after $T$ iterations, either the number of approximately feasible points are at least $|\cB| \ge \frac{T}{2}$ or the sub-optimality is already below the desired tolerance. Together with the preliminary lemma on the virtual iterates, this yields our full convergence theorem for \algname{Safe-EF}. Finally, in \Cref{cor:stepsize_choice} we verify that the stipulated conditions on $\gamma$ and $c$ are indeed feasible.

%\section{Extensions of \algnametitle{Safe-EF}}

%\subsection{Bidirectional compression}

%\begin{theorem}\label{th:ef14_bidirectional}
%    Assume Assumptions~\ref{asmp:convexity}-\ref{asmp:bounded_grad} hold, the server and workers use compressors $\cC_0 \in \mathbb{C}(\delta_{\rm s}), \{\cC_i\}_{i=1}^n \subseteq \mathbb{C}(\delta).$ Then there exist a choice of stepsize $\gamma$ and threshold $c$ such that the iterates of \algname{Safe-EF} algorithm with bidirectional compression satisfy 
%    \begin{align}\label{eq:ef14_bidirectional}
%    \begin{aligned}
%        \E{f(\overline{x}^T) - f(x^*)} &\le \cO\left(\frac{RM}{\sqrt{\delta_{\rm s}\delta T}}\right),\quad \text{ and }\\ \E{g(\overline{x}^T)} &\le \cO\left(\frac{RM}{\sqrt{\delta_{\rm s}\delta T}}\right),
%    \end{aligned}
%    \end{align}
%    where $\overline{x}^T \eqdef \frac{1}{|\cB|}\sum_{t\in\cB}x^t.$  %is defined in \eqref{eq:def_x_bar_t_bi}.
%\end{theorem}

%We refer to \Cref{sec:proof_bidirectional} for the proof. 

\section{Extension to Stochastic Setting}\label{sec:stochastic}

%In \Cref{sec:convergence_theory}, we consider the deterministic oracles. However, in most of the applications we have access to stochastic oracles only. Below we provide the convergence guarantees for \Cref{alg:ef14} with stochastic subgradients. 
%In \Cref{sec:stochastic_proofs}, we describe and rigorously analyze a natural extension of \algname{Safe-EF} to fully stochastic setting, where both subgradients $f^\prime_i(x,\xi^i) \approx f^\prime_i(x)$, $g^\prime_i(x,\xi^i) \approx g^\prime_i(x)$, and function evaluations $g_i(x,\xi^i) \approx g_i(x)$ can be noisy. Specifically, we assume stochastic subgradients are bounded almost surely, i.e., $\|f^\prime_i(x,\xi^i)\|^2, \|g^\prime_i(x,\xi^i)\|^2 \le M^2$, and constraint violations are evaluated by mini-batch of $\fvN$ samples and estimators $g_i(x,\xi^i)$ are $\fvsigma^2$-sub-Gaussian.

In this section, we consider a stochastic formulation of our the problem \eqref{eq:problem}, \eqref{eq:constraints}, namely, 

\begin{align}\label{eq:problem_stochastic}
 f_i(x) \eqdef \mathbb{E}_{\xi^i\sim\cD_i}\left[f_i(x, \xi^i) \right], 
\end{align}
and 
\begin{equation}\label{eq:constraints_stochastic}
g_i(x) \eqdef \mathbb{E}_{\xi^i\sim\cD_i}\left[g_i(x, \xi^i)\right], 
\end{equation}
where $\cD_i$ is a distribution of local environment (dataset) at worker $i\in[n]$. We assume that the noise follows a sub-Gaussian distribution.

\begin{assumption}\label{asmp:subgaussian_noise}
    Workers have access to $M$-bounded stochastic subgradients and $\fvsigma^2/\fvN$-sub-Gaussian function evaluations of $g_i$, namely, for some $M, \fvsigma^2/\fvN > 0$, any $x\in\R^d$, and any $i\in[n]$ we have 
    \begin{align}
        \|f^\prime_i(x,\xi^i)\|^2, \|g^\prime_i(x,\xi^i)\|^2 &\le M^2, \label{eq:asmp_bounded_subgrad}\\
        \E{\exp\left(\frac{(g_i(x,\xi^i) - g_i(x))^2}{\fvsigma^2/\fvN}\right)} &\le \exp(1),\label{eq:asmp_subgaussian_func}
    \end{align}
    where $\xi^i$ is a sample from the local dataset $\cD_i$. The latter assumption on sub-Gaussian function evaluation can be satisfied by implemented a mini-batch estimation of the constraints with batch-size $\fvN$. Moreover, we assume that the workers compute subgradients and function evaluations independently for any given $x$.
\end{assumption}

\begin{assumption}\label{asmp:stoch_convexity}
    We assume that for all $i\in[n]$ and for all $\xi^i \in \cD_i$ the functions $f_i(x,\xi^i)$ and $g_i(x,\xi^i)$ are convex, i.e. for all $x,y\in\R^d$ we have 
    \begin{align}
        f_i(y,\xi^i) \ge f_i(x,\xi^i) + \<f_i^\prime(x,\xi^i), y-x>,\\
        g_i(y,\xi^i) \ge g_i(x,\xi^i) + \<g_i^\prime(x,\xi^i), y-x>,
    \end{align}
    for all $f_i^\prime(x,\xi^i) \in \partial f_i(x,\xi^i)$ and $g_i^\prime(x,\xi^i) \in \partial g_i(x,\xi^i)$.
\end{assumption}

\begin{remark}
        We highlight that in the special (semi-stochastic) case when subgradient evaluations $f_i^{\prime}(x, \xi^i)$, $g_i^{\prime}(x, \xi^i)$ are stochastic, but the constraint evaluation of $g_i$ is deterministic, the proof significantly simplifies, and convergence analysis can be repeated as in \Cref{sec:proof_bidirectional}. However, the stochastic estimation of constraint violation $g(x)$ poses a significant challenge and we need to use advanced techniques to conduct high probability analysis. 
    \end{remark}

\begin{theorem}\label{thm:stochastic} Let $\beta\in(0,\nicefrac{1}{2})$ be a failure probability and $R \ge \|x^0-x^*\|$. Assume workers use deterministic compressors $\{\cC_i\}_{i=1}^n \subseteq \mathbb{C}(\delta)$. Then there exists a choice of stepsize $\gamma$, threshold $c$, and large enough batch-size $\fvN \ge \wtilde{\cO}(\frac{\fvsigma^2}{nc^2})$ such that the iterates of \algname{Safe-EF} with unidirectional compression satisfy with probability at least $1-2\beta$
\begin{align}
        &f(\overline{x}^T) - f(x^*) \le \cO\left(\frac{(MR + \frac{\fvsigma}{\sqrt{\fvN}})(1+\log\frac{1}{\beta})}{\sqrt{\delta T}}\right),\notag\\
        & g(\overline{x}^T) \le  \cO\left(\frac{(MR + \frac{\fvsigma}{\sqrt{\fvN}})(1+\log\frac{1}{\beta})}{\sqrt{\delta T}}\right).
\end{align}
\end{theorem}

To achieve $\varepsilon$-accuracy, i.e., $f(\overline{x}^T) - f(x^*), g(\overline{x}^T) \le \varepsilon$, \algname{Safe-EF} requires a batch-size of order $\wtilde{\cO}\left(\nicefrac{\fvsigma^2}{n\varepsilon^2}\right)$. The convergence rate matches the lower bound \eqref{eq:LB_unidirectional} up to numerical and logarithmic factors. The proof is deferred to \Cref{sec:stochastic_proofs}. One of the key technical challenges of the above result is that the analysis in the prior (non-distributed) work \citep{lan2020algorithms} relies on bounded domain assumption, while the iterates of our algorithm can be potentially unbounded. To address this issue we use the ideas from \citep{liu2023high} to establish a strong high probability convergence. 

\begin{remark}
        While the iteration (and communication) complexity of the method in the stochastic setting matches the lower bound up to numerical and logarithmic factors, its sample complexity is suboptimal. Taking into account the necessity of $\wtilde\cO(\frac{1}{\varepsilon^2})$ batch-size, the sample complexity of the method becomes $\wtilde\cO(\frac{1}{\varepsilon^4}).$ Nevertheless, this complexity is no worse than the one given by non-distributed gradient switching method \citep{lan2020algorithms}. We use a different technique to conduct high probability analysis than  \citet{lan2020algorithms} because their analysis crucially relies on bounded diameter assumption, which we do not have in our formulation.  
    \end{remark}

    \begin{remark}
        We emphasize that the proof in the stochastic unidirectional setting can be advanced to the bidirectional setting following the derivations of \Cref{th:ef14_bidirectional} and \Cref{thm:stochastic}. The convergence guarantees in the stochastic bidirectional setting matches that in the deterministic up to numerical and logarithmic factors.
    \end{remark}

%$\E{\exp\left(\frac{(g_i(x,\xi^i) - g_i(x))^2}{\fvsigma^2/\fvN}\right)} \le \exp(1)$

\section{Experiments}

Now we test \algname{Safe-EF} in practice. Below we provide experiments on a simple problem with synthetic data which satisfies all our assumptions, and later test our approach in more challenging task of training the Humanoid Robot. We include additional experiments on the classical Cartpole problem and Neyman-Pearson classification in \Cref{sec:additional_experiments}.

\subsection{Synthetic Data}\label{sec:synthetic}

\begin{figure*}[h!]
    \centering
    %\vspace{-1mm}
\includegraphics{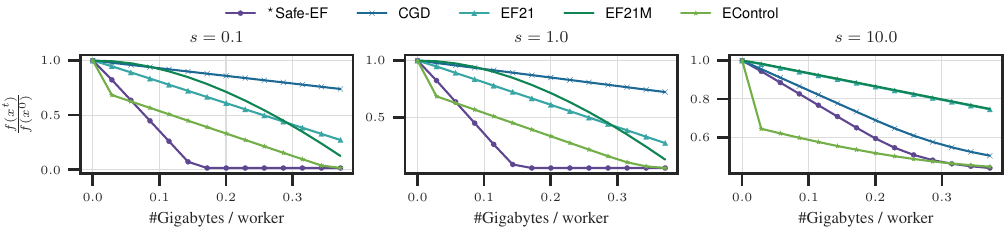}
    %\vspace{-2mm}
    \caption{Comparison of \algname{Safe-EF} against \algname{CGD}, \algname{EF21}, \algname{EF21M}, and \algname{EControl} on synthetic non-smooth problem. $^\star$\algname{Safe-EF} coincides with \algname{EF14} \citep{seide20141} in this problem.}
    \label{fig:synthetic}
    %\vspace{-3mm}
\end{figure*}
We begin with a simple empirical setup designed to easily verify that all assumptions of \algname{Safe-EF} are satisfied. Specifically, we consider the unconstrained problem of the form \eqref{eq:problem}, where $f_i = \|\mA_i x - b_i\|_1$. For this objective, the subgradient $f_i^\prime(x) = \mA_i^\top\sign(\mA_i x-b_i)$ \citep{beck2017first}. This choice ensures that all assumptions required for \algname{Safe-EF} hold. The data $\{\mA_i, b_i\}_{i=1}^n \subseteq \R^{d\times d}\times \R^d$ is synthetically generated, where the parameter 
$s$ controls the variability across local datasets: smaller values of $s$ result in matrices $\mA_i$ that are more similar to each other. We set $n=10, d=1000$, and use the Top-$K$ compressor with $K=\frac{d}{10}$ for all algorithms tested. Details of the data generation process can be found in \Cref{sec:additional_experiment_details}. We compare the proposed \algname{Safe-EF} with \algname{CGD}, \algname{EF21}, \algname{EF21M}~\citep{fatkhullin2024momentum}, and \algname{EControl}~\citep{gao2024econtrol}. For each method, hyper-parameters are tuned (see \Cref{sec:additional_experiment_details} for details) based on function value after $T=1000$ iterations, and performance with the optimal parameters is shown in \Cref{fig:synthetic}. Our results indicate that for $s\in\{0.1, 1.0\}$, \algname{Safe-EF} converges faster than all other algorithms. When heterogeneity is large, $s=10.0$, \algname{EControl} is initially faster; however, \algname{Safe-EF} catches up with \algname{EControl} by the end of the training.

\subsection{Policy Gradients for Humanoid Robot Fleet}
\label{sec:experiment-rl}
In this suite of experiments, we demonstrate an application of \algname{Safe-EF} for reinforcement learning. In this setup, each worker represents a humanoid robot that collects noisy measurements of some utility and constraint functions, to solve a \emph{constrained Markov decision process} \citep[][CMDP]{altman1999constrained}.

%\vspace{-1mm}
%\textbf{Constrained Markov decision processes.} 
\paragraph{Constrained Markov decision processes.}
We define a CMDP as the tuple $(\cS, \cA, r, c, p, \gamma, \rho)$, where $\cS$ describes a state space (e.g. joint positions and velocities) and $\cA$ describes a set of admissible actions (e.g. motor torques). The function $r: \cS \times \cA \rightarrow \R$ describes a reward function that is ought to be maximized, while $c: \cS \times \cA \rightarrow \R$ is a cost signal that must remain bounded. The system dynamics, $p$, describes a probability distribution over the next state, given a state $s \in \cS$ and action $a \in \cA$. States are initially drawn from the distribution $\rho$, and $\gamma$ denotes a discounting factor. In what follows, each robot-worker interacts with a separate CMDP, such that CMDPs differ only in their dynamics, i.e., each robot collects trajectories from a slightly perturbed $p_i$, relative to the nominal model $p$. Collecting trajectories entails carrying out actions determined by a \emph{policy} $\pi(a \mid s)$, a stochastic mapping from states to actions.
The objective and constraint for \emph{each} CMDP are defined as $J_r^i(\pi) \eqdef \EE_{\pi, p_i}\left[{\sum_{t = 0}^{\infty} \gamma^t r(s_t, a_t)}\right]$ and $J_c^i(\pi) \eqdef \EE_{\pi, p_i}\left[{\sum_{t = 0}^{\infty} \gamma^t c(s_t, a_t)}\right]$ where the expectations are w.r.t.$\,p_i, \,\rho$ and $\pi_x$, a policy parameterized by $x \in \R^d$.

%\paragraph{Policy gradient.}
{\bf Policy gradient.}
A common approach for policy search is via the class of \emph{policy gradient} algorithms~\citep{reinforce,schulman2017proximal}. In essence, policy gradient algorithms use Monte Carlo sampling to obtain stochastic gradient estimates of $x$ w.r.t. the objective and constraints by ``rolling~out'' the policy and measuring the returned rewards and costs along several trajectories. In our experiments, each worker collects data independently to obtain these estimates, which are then used to compute the PPO~\citep{schulman2017proximal} loss
\begin{align*}
    f_i(x) = \EE_{s, a \sim \bar{\pi}} \left[\min\left\{\frac{\pi_x(a \mid s)}{\bar{\pi}(a \mid s)} A_{p_i}^{\bar{\pi}}(s, a), \quad \text{clip}\left(\frac{\pi_x(a \mid s)}{\bar{\pi}(a \mid s)}, 1 - \tilde{\epsilon}, 1 + \tilde{\epsilon}\right)A_{p_i}^{\bar{\pi}}(s, a)\right\}\right],
\end{align*}
where, $A_{p_i}^{\bar{\pi}}$ denotes the advantage~\cite{schulman2015high} in terms of cumulative rewards, for picking an action compared to expected action of $\pi_x$, $\bar{\pi}$ is the policy with which the trajectory data was drawn and $\tilde{\epsilon}$ is a hyperparameter. Similarly, a surrogate for the constraint $g_i(x)$ is given by replacing rewards with costs when computing the advantage. Crucially, both $f_i$ and $g_i$ \emph{are non-smooth functions} due to $\text{clip}(x, l, u) \eqdef \max\{l, \min\{x, u\}\}$.

\begin{figure*}[t!]
    \centering
    \begin{minipage}[t]{0.39\textwidth}
        \centering
        %\vspace{-1mm}
        \includegraphics{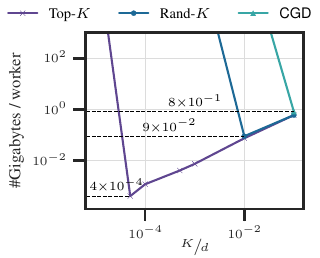}
        %\vspace{-1mm}
        \caption{Gigabytes required to reach a fixed benchmark performance for different compression ratios. Top-$K$ can achieve the same performance as \algname{CGD}, but with approximately two orders of magnitude less gigabytes.}
        \label{fig:k-to-budget}
    \end{minipage}%
    \hfill
    \begin{minipage}[t]{0.59\textwidth}
        \centering
        %\vspace{-1mm}
        \includegraphics{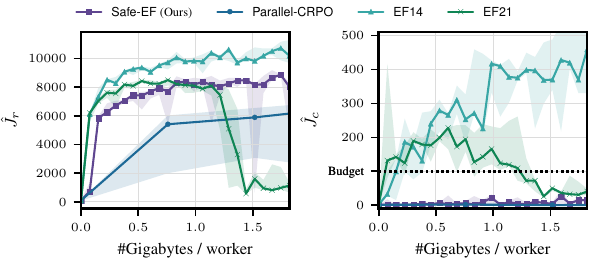}
        %\vspace{-1mm}
        \caption{Objective and constraint during learning. Budget denotes the level below which $J_c$ must remain to satisfy the constraint.
        \algname{EF14}, an unsafe baseline, fails to satisfy the constraint. \algname{EF21}, another unsafe baseline designed for smooth problems, diverges. \algname{Parallel-CRPO}, a safe baseline without compression, suffers from communication overhead. In contrast, \algname{Safe-EF} ensures constraint satisfaction with minimal performance loss.}
        \label{fig:safety}
    \end{minipage}
    %\vspace{-3mm}
\end{figure*}

%\paragraph{Setup.}
{\bf Setup.}
Unless specified otherwise, in all our experiments, the default number of workers is $n = 16$, compression ratio is $\nicefrac{K}{d} = 0.1$ with Top-$K$ compression. We parameterize a neural network policy with $d = 0.2\text{M}$ parameters and use a batch size $\fvN = 1024$ to evaluate $f_i$ and $g_i$. Moreover, we treat the NN parameters as a single ``flat'' vector when compressing, rather than performing layer-wise compression. We run all our experiments for 5  random seed initializations and report the median and a 68$\%$ confidence interval when applicable. Empirical estimates of the objective and constraint are denoted as $\hat{J}_r$ and $\hat{J}_c$ respectively. We use a batch of 128 trajectories to obtain these estimates. Further details, regarding the perturbations of models, the reward and cost functions and additional experiments are provided in \Cref{sec:additional_experiments,sec:additional_experiment_details}.

\paragraph{Experiment 1: Price of communication.}
%{\bf Experiment 1: Robust compression.}
We evaluate \algname{Safe-EF} with Top-$K$ and Rand-$K$ sparsifiers and compare it with a constrained version of \algname{CGD} with a Top-$K$ sparsifier. To adapt \algname{CGD} to enforce the constraint, we follow the same approach as \algname{Safe-EF} and use the switching subgradient method. \Cref{fig:k-to-budget} shows the amount of communication (in gigabytes per worker) required to reach a fixed performance of $\hat{J}_r = 7500$ as the compression ratio $\nicefrac{K}{d}$ increases. As illustrated, both Top-$K$ and Rand-$K$ significantly reduce communication costs compared to \algname{CGD}, with Top-$K$ demonstrating the most robust performance across varying compression rates with about $2000\times$ improvement in communication reduction!

\paragraph{Experiment 2: Safety.}
%{\bf Experiment 2: Safety.}
We study the performance of \algname{Safe-EF} in terms of constraint satisfaction and compare it against the unsafe error feedback algorithms \algname{EF14} \citep{seide20141} and \algname{EF21} \cite{richtarik2021ef21}. Additionally, we compare \algname{Safe-EF} against a parallel variant of CRPO \citep{xu2021crpo}, a CMDP solver that enforces constraints via the subgradient switching method. Our parallel variant of it, indicated as \algname{Parallel-CRPO}, operates independently on each worker and transmits parameters $x$ to the server without compression. The results are presented in \Cref{fig:safety}, where \algname{Safe-EF} satisfies the constraints with a slight performance reduction, while \algname{EF14} violates the constraint. \algname{EF21} diverges, possibly due to non-smoothness of the objective and constraint. Next, given the same communication budget in gigabytes per worker, \algname{Parallel-CRPO} fails to converge. This outcome highlights the non-trivial nature of the task, emphasizing that optimal policies in the unconstrained case are insufficient to meet the constraints.

\paragraph{Experiment 3: Number of workers.}
%{\bf Experiment 3: Number of workers.}
We analyze the performance of \algname{Safe-EF} under varying number of available workers and present our findings in \Cref{fig:ablate-workers}. Our results reveal two key observations. First, the convergence rate decreases significantly when the number of workers is very small. Second, beyond a certain threshold, increasing the number of workers yields diminishing performance gains. The latter aligns with our theoretical lower bounds in \Cref{thm:lower_bound_deterministic}, which establish that no improvement in $n$ is possible in the worst case.

\begin{figure}[ht]
  \centering
  % first plot in a 45%-wide minipage
  \begin{minipage}[t]{0.48\textwidth}
    \centering
    \includegraphics{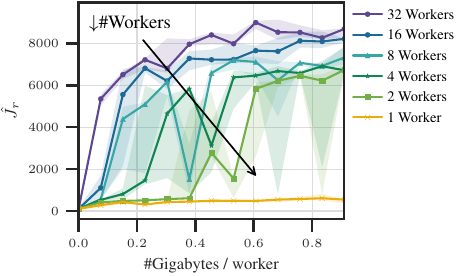}
    \caption{Convergence plots for different number of workers. While increasing the number of workers helps reduce the communication cost, the effect becomes less significant as the number of workers continues to grow.}
    \label{fig:ablate-workers}
  \end{minipage}
  \hfill
  % second plot in a 45%-wide minipage
  \begin{minipage}[t]{0.48\textwidth}
    \centering
    \includegraphics{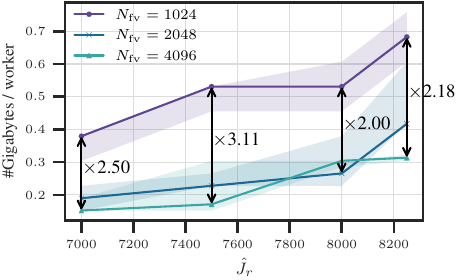}
        %\vspace{-1mm}
        \caption{Communication required to reach a desired performance level for different batch samples $N_{\text{fv}}$.  Beyond a certain batch size, improvement diminishes.}
    \label{fig:ablate-batch}
  \end{minipage}
\end{figure}

%\begin{figure}[h]
%        %\vspace{-1mm}
%        \centering
%        \includegraphics{Figures/ablate-workers.pdf}
        %\vspace{-1mm}
%        \caption{Convergence plots for different number of workers. While increasing the number of workers helps reduce the communication cost, the effect becomes less significant as the number of workers continues to grow.}
%        %\vspace{-1mm}
%        \label{fig:ablate-workers}
%\end{figure}

%\begin{figure}[h]
%    \centering
    %\vspace{-1mm}
%    \includegraphics{Figures/ablate-batch.pdf}
        %\vspace{-1mm}
%        \caption{Communication required to reach a desired performance level for different batch samples $N_{\text{fv}}$.  Beyond a certain batch size, improvement diminishes.}
%    \label{fig:ablate-batch}
%\end{figure}

%\paragraph{Experiment 4: Sampling.}
%\vspace{-3mm}
\paragraph{Experiment 4: Effect of batch-size.}
%{\bf Experiment 4: Effect of batch-size.}
\Cref{thm:stochastic} has a certain requirement of sufficiently large batch-size $\fvN$ due to constraint estimation process. If this requirement is met, the convergence rate is improved when increasing $\fvN$ until it reaches the lower bound in \Cref{th:lower_bound_uni}.
To study this effect in practice, we vary the batch size $\fvN\in\{256, 512, 1024, 2048, 4096\}$. %and measure the number of gigabytes required to reach a certain level of performance.
Our results in \Cref{fig:ablate-batch}, indicate that by increasing the batch size from $\fvN= 1024$ to $2048$, we can see the improvement, however, a further increase from $\fvN=2048$ to $\fvN=4096$ does not yield more improvement. For smaller batch sizes $\fvN \in \{256, 512\}$, \algname{Safe-EF} did not converge, resulting in non-numeric values, and therefore are not presented in \Cref{fig:ablate-batch}. These findings are in line with our large-batch requirement in \Cref{thm:stochastic} and highlight the need to design algorithms that are robust to smaller batch sizes---suggesting an important direction for future work.

%\vspace{-7mm}
\section{Limitations and Future Work}
While we make significant progress in understanding non-smooth EF, there are certain limitations in our work. First, we assume all functions are convex, while \algname{Safe-EF} seems to excel even in challenging, highly non-convex RL tasks. Thus, it is crucial to understand non-convex problems: in general setting \citep[e.g. ][]{boob2023stochastic,jia2022first,grimmer2025goldstein} as well as in structured RL problems \citep[e.g. ][]{agarwal2021theory,xu2021crpo,lan2023policy,fatkhullin2023stochasticpolicy,barakat2023reinforcement}. Second, our noise assumptions are relatively stringent, and can be potentially relaxed using gradient clipping \citep{nazin2019algorithms,gorbunovhigh} or normalization \citep{hubler2024gradient} techniques, although this is non-trivial due to constraint estimation. Finally, our algorithm requires large batch-sizes and is not sample efficient in the stochastic setting due to constraint estimation, and our experiments indicate it is likely the issue of the algorithm. Primal-dual approaches \citep{juditsky2011solving,boob2023stochastic} can be helpful in mitigating this limitation.

%\paragraph{Assumptions.}
%In this setup, there are two assumptions that are not necessarily guaranteed. First, of course, due to the use of a neural network (NN) to parameterize our policy, the functions $f_i$ and $g_i$ are non-convex. Second, we cannot guarantee that the subgradients are bounded. While these assumptions are not rigorously proven, \algname{Safe-EF} demonstrates strong performance in our experiments, even under severe compression, where each worker sends only dozens of parameter per iteration.

%\newpage

\section*{Acknowledgements} 

Rustem Islamov acknowledges the financial support
of the Swiss National Foundation, SNF grant No 207392. Yarden As is supported by the grant of the Hasler foundation (grant no. 21039). Yarden As and Ilyas Fatkhullin are funded by ETH AI Center.

%\section*{Impact Statement}
%This paper presents work whose goal is to advance the field of Machine Learning. There are many potential societal consequences of our work, none of which we feel must be specifically highlighted here.

\bibliography{references.bib}
\bibliographystyle{plainnat}

%%%%%%%%%%%%%%%%%%%%%%%%%%%%%%%%%%%%%%%%%%%%%%%%%%%%%%%%%%%%%%%%%%%%%%%%%%%%%%%
%%%%%%%%%%%%%%%%%%%%%%%%%%%%%%%%%%%%%%%%%%%%%%%%%%%%%%%%%%%%%%%%%%%%%%%%%%%%%%%
% APPENDIX
%%%%%%%%%%%%%%%%%%%%%%%%%%%%%%%%%%%%%%%%%%%%%%%%%%%%%%%%%%%%%%%%%%%%%%%%%%%%%%%
%%%%%%%%%%%%%%%%%%%%%%%%%%%%%%%%%%%%%%%%%%%%%%%%%%%%%%%%%%%%%%%%%%%%%%%%%%%%%%%
\newpage
\appendix
\onecolumn

\tableofcontents
\newpage 

\section{Additional Related Work}
%Literature on (greedy) coordinate descent: \citep{nesterov2012efficiency}, \citep{lacoste2013block,nutini2015coordinate,fercoq2015accelerated,lu2018accelerating}. 

The Error Feedback (EF) mechanism was initially studied in the single-node setting ($n=1$) by \citet{stich2018sparsified,alistarh2017qsgd}. Subsequent research extended its analysis to the smooth convex setting, incorporating additional unbiased compressors \citep{gorbunov2020linearly,stich2020communication,qian2021error}. The \algname{EF21} algorithm, introduced by \citet{richtarik2021ef21}, was the first to establish provable convergence in the large-batch smooth regime without data heterogeneity bounds. Later, \citet{fatkhullin2024momentum} removed this large-batch requirement by integrating a momentum mechanism into the \algname{EF21} framework, achieving an optimal asymptotic rate. An extension of \algname{EF14}, called \algname{EControl}, was proposed by \citet{gao2024econtrol}, demonstrating convergence in both smooth convex and non-convex settings while attaining optimal asymptotic complexity. Recent research has further advanced the analysis of EF, extending it to variational inequalities \citep{beznosikov2022distributed}, decentralized communication graphs \citep{koloskovadecentralized,singh2021squarm,zhao2022beer,islamov2024near}, local updates \citep{huangstochastic}, bilevel optimization \citep{he2024distributed}, and reinforcement learning \citep{mitra2023temporal, adibi2024stochastic, beikmohammadi2024compressed}. Additionally, \citet{richtarik20223pc,makarenko2022adaptive,islamov2023distributed} expanded EF analysis to a broader class of 3PC compression operators, encompassing contractive compressors as a special case. Recent works analyzed the EF mechanism as a special case of biased gradient descent in the single-node setting \citep{ajalloeian2020convergence, demidovich2023guide} while \citet{richtarik2024error} improved the constant dependencies in the rate of \algname{EF21}.

\algname{EF21} variant of EF has been analyzed in the context of $(L_0, L_1)$-smooth optimization \citep{khirirat2024error}, which is different from our non-smoothness since $(L_0, L_1)$-smoothness implies smoothness on any compact set and failure examples as in \Cref{example:div_ef21} cannot happen under such assumption. On the other hand, if not limited to compact set the gradients under $(L_0, L_1)$-smoothness can grow when $x \rightarrow \infty$. 

%Literature on EF. First analysis . Further studies in the smooth convex case and non-convex settings 

%Variational inequalities \cite{beznosikov2022distributed}.

%EF and local updates , 

%More general compression \citep{richtarik20223pc,islamov2023distributed}

%EF in RL \cite{mitra2023temporal, adibi2024stochastic,beikmohammadi2024compressed}. 

\newpage

\section{Failure of \algnametitle{CGD} and \algnametitle{EF21} in Non-smooth Convex Setting}
\paragraph{Proof of \Cref{example:div_CGD}. Non-convergence of \algnametitle{CGD}.}
\begin{proof}
	Consider a $2$-dimensional problem $f_i(x) = \|x\|_1$, $f(x) = \frac{1}{n} \sum_{i=1}^n f_i(x)$ with $f(x_*) = 0$. Set the initial vectors $x^0 = (\gamma/2, -1)^{\top}$ and consider \algname{CGD} \eqref{eq:CGD} with Top-$1$ compressor. 
    
    The proof for the case when $\gamma = 0$ is trivial. We consider the case when $\gamma > 0$. In this case, the function is differentiable at every point of its trajectory, and for any $t\geq 0$ it holds that 
	$$
	x^t = \begin{pmatrix} \frac{\gamma (-1)^t }{2}  \\ -1 \end{pmatrix} , \qquad 	\partial f_i(x^t) = \left\{\begin{pmatrix} (-1)^t  \\ -1  \end{pmatrix} \right\} . 
	$$ 
	The base of induction ($t=0$) is trivial. For the induction step, we make the calculation
	$$
	x^{t+1} = x^t - \gamma g^t = \begin{pmatrix} \frac{\gamma (-1)^t }{2}  \\ -1 \end{pmatrix} - \gamma \operatorname{Top-1} \begin{pmatrix} (-1)^t  \\ 1  \end{pmatrix}   = \begin{pmatrix} \frac{\gamma (-1)^{t+1} }{2}  \\ -1 \end{pmatrix} ,
	$$
	where in the last step, Top-$1$ operator always selects the first coordinate since the entries are equal in absolute value. It remains to compute the function value at these iterates $f(x^t)$ to conclude the proof.
\end{proof}

We remark that divergence issues of gradient methods using biased compressors were previously raised in \citep{Karimireddy_SignSGD}. However, their examples only apply to Sign operator, while we are mainly interested in the behavior of Top-$K$ compressor for distributed optimization. Thus, a different construction is required to capture the interplay of Top-$K$ compressor with non-smoothness of $f$. Another divergence example using Top-$K$ is shown by \citet{beznosikov2023biased}, however, their example is smooth, strongly convex and the key effect is different, since their divergence happens due to heterogeneity. Finally, \citet{fatkhullin2024momentum} show an example of divergence of \algname{EF21} in the stochastic setting, which is also different since their function is smooth, strongly convex and the divergence occurs due to noise.   

\paragraph{Proof of \Cref{example:div_ef21}. Divergence of \algnametitle{EF21}.} 

\begin{proof}
	Similarly to the proof of Example~\ref{example:div_CGD}, we consider a $2$-dimensional problem $f_i(x) = \|x\|_1$, $f(x) = \frac{1}{n} \sum_{i=1}^n f_i(x)$ with $f(x_*) = 0$. Set the initial vectors $x^0 = (\gamma/2, -1)^{\top}$ , $v_i^0 = (1, 1)^{\top}$, and consider \algname{EF21} \eqref{eq:EF21} with Top-$1$ compressor. 
    
    The proof for the case when $\gamma = 0$ is trivial. 
    We consider the case when $\gamma > 0$. In this case the function is differentiable at every point of its trajectory and for any $t\geq 0$ it holds that 
	$$
	x^t = \begin{pmatrix} \frac{\gamma (-1)^t }{2}  \\ -1 - t \, \gamma \end{pmatrix} , \qquad 	\partial f_i(x^t) = \left\{\begin{pmatrix} (-1)^t  \\ -1  \end{pmatrix} \right\} , \qquad 	v_i^t = \begin{pmatrix} (-1)^t  \\ 1  \end{pmatrix}  . 
	$$ 
	The base of induction ($t=0$) is trivial. For the induction step, we make the calculation
	$$
	x^{t+1} = x^t - \gamma \, v^t = \begin{pmatrix} \frac{\gamma (-1)^t }{2}  \\ -1 - t \, \gamma \end{pmatrix} - \gamma \begin{pmatrix} (-1)^t  \\ 1  \end{pmatrix}   = \begin{pmatrix} \frac{\gamma (-1)^{t+1} }{2}  \\ -1 - (t+1) \, \gamma \end{pmatrix} ,
	$$
	$$
	v^{t+1} = v_i^{t+1} = \begin{pmatrix} (-1)^t  \\ 1  \end{pmatrix} + \operatorname{Top-1}\left( \begin{pmatrix} (-1)^{t+1}  \\ -1  \end{pmatrix} - \begin{pmatrix} (-1)^t  \\ 1  \end{pmatrix} \right)  = \begin{pmatrix} (-1)^{t+1}  \\ 1  \end{pmatrix},
	$$	
	where in the last step, Top-$1$ operator selects the first coordinate since the entries are equal in absolute value. It remains to compute the function value at these iterates $f(x^t)$ to conclude the proof.
\end{proof}

\newpage
\section{Convergence Upper Bound for \algnametitle{EF21} in Smooth Convex Setting}\label{sec:Proj_EF21}
In this section, we consider \algname{EF21} method with projection 
\begin{align}
    \text{\algname{Projected-EF21}} \qquad \qquad &
    \begin{aligned}
        x^{t+1} &= \Pi_{\cX}( x^t - \gamma \, v^t) , \qquad v^t = \frac{1}{n} \sum_{i=1}^n v_i^t,  \\
        v_i^{t+1} &= v_i^t + \mathcal{C}(\nabla f_i(x^{t+1}) - v_i^t) .
    \end{aligned}
    \label{eq:Proj_EF21_Proj}
\end{align}

where $\Pi_{\cX}$ is a projection operator on a convex set $\cX$. This method was proposed and analyzed earlier in \cite{fatkhullin2021ef21} for non-convex smooth problems. In \Cref{example:div_ef21}, we showed that this algorithm is not suitable for non-smooth optimization because it diverges even in a simple convex example like $\|x\|_1$. While this algorithm was extensively studied for smooth non-convex problems, we are not aware of any convergence results for this algorithm under convexity (with convergence in the function value). To close this gap and complement the failure example of this method in \Cref{example:div_ef21} in non-smooth convex case, we provide the convergence result for this method in smooth convex setting.

\begin{theorem}\label{thm:Proj_EF21}
Let each $f_i(\cdot)$ be differentiable and $L_i$-smooth on $\cX$ for all $i= 1, \ldots, n$, i.e., $\|\nabla f_i(x) - \nabla f_i(y)\| \leq L_i \|x - y\|$ for all $x, y \in \cX$, and let $f(\cdot)$ be convex over a convex compact set $\cX \subseteq \R^d$ with diameter $R_{\cX}$. Then for any $T \geq 1$ \algname{Projected-EF21} with stepsize $\gamma \leq \frac{\delta}{2 \sqrt{6} L }$ satisfies
$$
\mathbb{E} \left[f(x^T) - f(x^*) \right] \leq \frac{R_{\cX}^2 }{\gamma T} \left( 1 + \log\left(\frac{\gamma \Lambda_0 T}{R_{\cX}^2}\right) \right) , 
$$
where $\Lambda_0 := f(x^0) - f(x^*) + \frac{1}{\sqrt{6} L } \|g^0 - \nabla f(x^0) \|^2$, and $L := \sqrt{\frac{1}{n} \sum_{i=1}^n L_i^2}$.
\end{theorem}

\begin{remark}
    The current stepsize restriction is $\gamma \le \frac{\delta}{2\sqrt{6}L},$ where $L$ is the quadratic mean of the smoothness constants $L_i$. This restriction can be further improved by following the results in \citet{richtarik2024error}, which requires weighting workers' contributions by non-uniform constants. This leads to the improved step-size (and eventually improved rate) of the form $\gamma \le \cO(1/\overline{L})$, where $\overline{L} = \frac{1}{n}\sum_{i=1}^nL_i$, since $\overline{L} \le L$ always holds.
\end{remark}

Before we move to the proof of this result, a few comments are in order. First, if we set $\gamma = \frac{\delta}{2 \sqrt{6} L }$, this theorem implies $\widetilde{\mathcal O}\left( \frac{L R_{\cX}^2 }{\delta T }\right)$ convergence rate for \algname{Projected-EF21}, where $\widetilde{\mathcal O}$ hides numerical constants and a logarithmic term. This convergence rate recovers (up to a logarithmic factor) the rate of subgradient descent when $\delta = 1$ (no compression), and is $1/\delta$ times worse in the presence of compression. This is consistent with rates in non-convex and strongly convex settings~\citep{richtarik2021ef21,fatkhullin2021ef21}. We believe the logarithmic factor can be removed by a more careful choice of parameter $\lambda$ in the proof below. Second, the compactness of the set $\cX$ is critical in the analysis of the method, it would be interesting to explore if this requirement can be removed. Finally, the extension of this method to stochastic setting is possible by replacing $\nabla f_i(x^{t+1})$ with a large mini-batch or momentum estimator, however, a batch-free version of this method may not converge due to a counter-example in \citep{fatkhullin2024momentum}. 

\begin{proof}
Since each $f_i$ is $L_i$-smooth, it follows that $f(x) = \frac{1}{n} \sum_{i=1}^n f_i(x)$ is $L$-smooth with $L = \sqrt{\frac{1}{n} \sum_{i=1}^n L_i^2}$. Next, we follow the proof technique similar to Theorem 8 in \citep{fatkhullin2023stochastic}. By smoothness of $f$, we have for any $z \in \cX$

\begin{eqnarray*}
 f\left(x^{t+1}\right) & \leq& f\left(x^t\right)+\left\langle \nabla f\left(x^t\right), x^{t+1}-x^t\right\rangle+\frac{L}{2}\left\|x^{t+1}-x^t\right\|^2 \\
& = & f\left(x^t\right)+\left\langle v^t, x^{t+1}-x^t\right\rangle+\frac{1}{2 \gamma}\left\|x^{t+1}-x^t\right\|^2 + \langle\nabla f\left(x^t\right)-v^t, x^{t+1}-x^t \rangle -\left(\frac{L}{2}-\frac{1}{2 \gamma}\right)\left\|x^{t+1}-x^t\right\|^2  \\
& \leq &  f\left(x^t\right)+\frac{1}{2 \gamma}\left\|x^t-z\right\|^2-\frac{1}{2 \gamma}\left\|x^{t+1}-z\right\|^2+\left\langle v^t, z-x^t\right\rangle \\
&& \qquad + \left\langle\nabla f\left(x^t\right)-v^t, x^{t+1}-x^t\right\rangle-\left(\frac{L}{2}-\frac{1}{2 \gamma}\right)\left\|x^{t+1}-x^t\right\|^2 =: (*) ,
\end{eqnarray*}
where the last inequality follows by the update rule of the algorithm. Next, rearranging we get 
\begin{eqnarray*}
    (*) &=& f\left(x^t\right) +\frac{1}{2 \gamma}\left\|x^t-z\right\|^2-\frac{1}{2 \gamma}\left\|x^{t+1}-z\right\|^2+\left\langle\nabla f\left(x^t\right), z-x^t\right\rangle \\
&& \qquad +\left\langle\nabla f\left(x^t\right)-v^t, x^{t+1}-z\right\rangle-\left(\frac{L}{2}-\frac{1}{2 \gamma}\right)\left\|x^{t+1}-x^t\right\|^2 \\
& \leq & f\left(x^t\right)+\frac{1}{2 \gamma}\left\|x^t-z\right\|^2-\frac{1}{2 \gamma}\left\|x^{t+1}-z\right\|^2+\left\langle\nabla f\left(x^t\right), z-x^t\right\rangle \\
& &\qquad +\frac{\gamma}{2}\left\|v^t-\nabla f\left(x^t\right)\right\|^2+\frac{1}{2 \gamma}\left\|x^{t+1}-z\right\|^2-\left(\frac{L}{2}-\frac{1}{2 \gamma}\right)\left\|x^{t+1}-x^t\right\|^2  \\
& = & f\left(x^t\right)+\frac{1}{2 \gamma}\left\|x^t-z\right\|^2 + \left\langle\nabla f\left(x^t\right), z-x^t\right\rangle + \frac{\gamma}{2}\left\|v^t-\nabla f\left(x^t\right)\right\|^2 - \left(\frac{L}{2}-\frac{1}{2 \gamma}\right)\left\|x^{t+1}-x^t\right\|^2  ,
\end{eqnarray*}
where we used Young's inequality $\langle a, b \rangle \leq \frac{\gamma}{2} \|a\|^2 + \frac{2}{\gamma} \|b\|^2$ for any $a, b \in \R^d$. Using (lower curvature) smoothness of $f$, we derive

\begin{eqnarray*}
 f\left(x^{t+1}\right) &\leq& f(z) + \left( \frac{1}{2 \gamma} + \frac{L}{2} \right) \left\|x^t-z\right\|^2 + \frac{\gamma}{2}\left\|v^t-\nabla f\left(x^t\right)\right\|^2  - \left(\frac{L}{2}-\frac{1}{2 \gamma}\right)\left\|x^{t+1}-x^t\right\|^2 \\
&\leq& f(z)+\frac{1}{\gamma}\left\|x^t-z\right\|^2+\frac{\gamma}{2}\frac{1}{n}\left\|v_i^t-\nabla f_i\left(x^t\right)\right\|^2-\left(\frac{L}{2}-\frac{1}{2 \gamma}\right)\left\|x^{t+1}-x^t\right\|^2 , 
\end{eqnarray*}
where the last inequality holds since $\gamma \leq 1/L$. Now we fix some $\lambda \in[0,1]$ and select $z=(1-\lambda)\, x^t+\lambda x_* \in \cX$, where $x_* \in \cX_*$. By convexity of $f(\cdot)$, we have 
$$
f(z) \leq (1-\lambda) f(x^t)+\lambda f(x_*) - \frac{\lambda (1-\lambda)}{2L} \|\nabla f(x^t) - \nabla f(x_*)\|^2 \leq (1-\lambda) f(x^t)+\lambda f(x_*) . 
$$
Moreover, $\left\|x^t-z\right\|=\lambda\left\|x^t-x_x\right\| \leq \lambda \, R_{\cX}$, where $R_{\cX}=\max _{x, y \in \cX}\|x-y\|$. Thus, we get for any $\lambda \in[0,1]$
\begin{equation}\label{eq:HC_descent}
f(x^{t+1}) - f(x_*) \leq (1-\lambda) (f(x^{t}) - f(x_*)) + \frac{\lambda^2 R_{\cX}^2}{\gamma}+\frac{\gamma}{2} V_t-\left(\frac{L}{2}-\frac{1}{2 \gamma}\right)\|x^{t+1}-x^t\|^2 .
\end{equation}
For a contractive compressor we have $\mathbb{E}\|\mathcal C(x) - x\|^2 \leq (1-\delta)\|x\|^2$ for some $\delta \in (0,1]$. Let $V_{t,i} := \mathbb{E}\|g_i^t - \nabla f_i(x^t)\|^2$, $V_t := \frac{1}{n} \sum_{i=1}^n V_{t,i}$. Then
\begin{eqnarray*}
    V_{t+1, i} &=& \mathbb{E}\|g_i^{t+1} - \nabla f_i(x^{t+1})\|^2 = \mathbb{E}\|\mathcal C(\nabla f_i(x^{t+1}) - g_i^t) + g_i^t - \nabla f_i(x^{t+1})\|^2 \nonumber\\
    &\leq& (1-\delta)\mathbb{E}\| g_i^t - \nabla f_i(x^{t+1}) \|^2 \nonumber\\
    &\leq& (1-\delta)\left( 1+ \frac{\delta}{2} \right) \mathbb{E}\| g_i^t - \nabla f_i(x^{t}) \|^2 + \left(1+
    \frac{2}{\delta}\right)\mathbb{E}\| \nabla f_i(x^{t+1}) - \nabla f_i(x^t) \|^2 \nonumber\\
    &\leq& \left( 1 - \frac{\delta}{2} \right) V_{t, i} + \frac{3 L_i^2}{\delta}\mathbb{E}\|x^{t+1} - x^t\|^2. 
\end{eqnarray*}
By averaging for $i = 1, \ldots, n$, we get
\begin{equation}\label{eq:ef21_recursion}
V_{t+1} \leq \left( 1 - \frac{\delta}{2} \right) V_{t} + \frac{3 L^2}{\delta}\mathbb{E}\|x^{t+1} - x^t\|^2.
\end{equation}

Define $\Delta_t:= \mathbb{E}[f(x^t)-f(x_*)]$, then adding \eqref{eq:HC_descent} $+\frac{2}{\delta}$ times \eqref{eq:ef21_recursion} and taking $\gamma \leq \frac{\delta}{2 \sqrt{6} L }$, we have
$$
\begin{aligned}
\Lambda_{t+1} &:= \Delta_{t+1}+\frac{2 \gamma}{\delta} V_{t+1}\\
&\leq(1-\lambda)\Delta_t+\frac{\gamma}{2} V_t+\frac{2\gamma}{\delta}\left(1-\frac{\delta}{2}\right)V_t + \frac{\lambda^2}{\gamma}R_{\cX}^2 - \left(\frac{L}{2}-\frac{1}{2\gamma}+\frac{3L^2\cdot 2\gamma}{\delta}\right)\mathbb{E}\|x^{t+1}-x^t\|^2\\
&=(1-\lambda)\Delta_t+\frac{2\gamma}{\delta}\left(1-\frac{\delta}{2}+\frac{\gamma}{2}\frac{\delta}{2\gamma}\right)V_t + \frac{\lambda^2}{\gamma}R_{\cX}^2 - \left(\frac{L}{2}-\frac{1}{2\gamma}+\frac{3L^2}{\delta}\frac{2\gamma}{\delta}\right)\mathbb{E}\|x^{t+1}-x^t\|^2 \\
&\leq (1-\lambda)\Delta_t + \frac{2\gamma}{\delta}\left(1-\frac{\delta}{4}\right) V_t + \frac{\lambda^2}{\gamma}R_{\cX}^2 \\
&\leq (1-\lambda)\Lambda_t + \frac{\lambda^2}{\gamma}R_{\cX}^2 ,
\end{aligned}
$$
where in the last step we assume the choice $\lambda \leq \delta / 4$. Finally, we unroll the recursion for $t=0,1,\ldots,T-1$ and setting $\lambda = \min\left\{ \frac{\delta}{4} ; \frac{1}{N}\log\left(\frac{\gamma \Lambda_0 N}{R_{\cX}^2}\right) \right\}$, we derive
\begin{eqnarray*}
\Lambda_T &\leq& (1-\lambda)^T\Lambda_0 + \left(\sum_{t=0}^{T-1}(1-\lambda)^t\right)\frac{\lambda^2 R_{\cX}^2}{\gamma} \leq  (1-\lambda)^T\Lambda_0 + \frac{\lambda R_{\cX}^2}{\gamma}\\
&=& \exp(T \log(1-\lambda)) \Lambda_0 + \frac{\lambda R_{\cX}^2}{\gamma} \leq \exp\left(-\log\left(\frac{\gamma \Lambda_0 T}{R_{\cX}^2}\right) \right) \Lambda_0 + \frac{\lambda R_{\cX}^2}{\gamma}\\
&\leq& \frac{R_{\cX}^2 }{\gamma T} + \frac{R_{\cX}^2 }{\gamma T} \log\left(\frac{\gamma \Lambda_0 T}{R_{\cX}^2}\right).
\end{eqnarray*}

\end{proof}

\newpage

\section{Convergence Upper Bound for \algnametitle{Safe-EF} with Bidirectional Compression}\label{sec:proof_bidirectional}

The analysis uses the ``virtual iterates'' framework, which is often used in the literature \citep{stich2019error,koloskova2023gradient,mishchenko2023partially,islamov2024asgrad}. Define the virtual iterates $\hat{x}^t \eqdef w^t - \gamma e^t$ with $\hat{x}^0 = x^0.$ Note that then we have $\hat{x}^{t+1} = \hat{x}^t - \gamma h^t$. Indeed, assume that it is true at iteration $t$, then
\begin{equation}\label{eq:def_x_hat_t}
    \hat{x}^{t+1} = w^{t+1} - \gamma e^{t+1} = (w^{t} - \gamma v^t) - \gamma (e^t + h^t - v^t) = (w^t- \gamma e^t) - \gamma h^t = \hat{x}^t - \gamma h^t.
\end{equation}
We additionally define $\hat{e}^t \eqdef w^t - x^t$, an error that appears due to down-link (server to worker) compression.

\begin{lemma}\label{lem:lemma2_bi}
    For any $x\in\R^d$, the following inequality holds
    \begin{align*}
        \sum_{t\in\cB} \gamma(f(x^t)-f(x))
        + \sum_{t\in\cN} \gamma[c - g(x)]
        &\le \frac{1}{2}\|x^0 - x\|^2
        + \frac{1}{2}\sum_{t=0}^{T-1}\gamma^2\|h^t\|^2
        + \sum_{t=0}^{T-1}\gamma^2 \|h^t\|\cdot \|e^t\|\notag\\
        &\quad +\; \sum_{t=0}^{T-1}\gamma \|h^t\|\cdot \|\hat e^t\|
    \end{align*}
\end{lemma}
\begin{proof}
    From the update rule \eqref{eq:def_x_hat_t}, we have
    \begin{equation*}
        \|\hat x^{t+1} - x\|^2 = \|\hat x^t - x\|^2 
        - 2\gamma\<h^t, \hat{x}^t - x>
        + \gamma^2\|h^t\|^2.
    \end{equation*}
    Rewriting the above, we get
   % \begin{align}
    %    2\gamma\<h^t, \hat x^t-x> &= 2\gamma\<h^t, x^t - x>
    %    + 2\gamma\<h^t, w^t - x^t>
    %    + 2\gamma\<h^t, \hat x^t - w^t>\notag \\
     %   &= \|\hat x^t - x\|^2
     %   - \|\hat x^{t+1} - x\|^2
     %   + \gamma^2\|h^t\|^2.
    %\end{align}
    \begin{align*}
        2\gamma\<h^t, x^t-x> &= \|\hat x^t - x\|^2
        - \|\hat x^{t+1} - x\|^2
        + \gamma^2\|h^t\|^2 + 2\gamma\<h^t, x^t - w^t>
        + 2\gamma\<h^t, w^t - \hat x^t>
        \notag \\
        &\leq \|\hat x^t - x\|^2
        - \|\hat x^{t+1} - x\|^2
        + \gamma^2\|h^t\|^2 + 2\gamma^2 \|h^t\| \|e^t\|
        + 2\gamma \| h^t \| \|\hat e^t \| .
        \notag 
        %&\leq \|\hat x^t - x\|^2 - \|\hat x^{t+1} - x\|^2
       % + \gamma^2\|h^t\|^2 + 2\gamma \|h^t\| \|e^t\|
       % + \frac{\gamma^2}{\delta^2} \| h^t \|^2 + \delta^2 \|\hat e^t \|^2
      %  \notag ,
    \end{align*}
    %where the last inequality uses Young's inequality. 
    Summing up both sides, we derive 
   
    \begin{align*}
        2\sum_{t=0}^{T-1} \gamma\<h^t, x^t - x> %&= \sum_{t=0}^{T-1} \|\hat x^t - x\|^2  - \sum_{t=0}^{T-1}\|\hat x^{t+1}-x\|^2  + \sum_{t=0}^{T-1}\gamma^2\|h^t\|^2 + 2\sum_{t=0}^{T-1}\gamma\<h^t, x^t - w^t>\notag \\
        %&\quad +\;  2\sum_{t=0}^{T-1}\gamma\<h^t, w^t - \hat x^t>\notag \\
        &\le \| x^0 - x\|^2 
        - \|\hat x^T - x\|^2 
        + \sum_{t=0}^{T-1}\gamma^2 \|h^t\|^2
        + 2\sum_{t=0}^{T-1}\gamma^2\|h^t\|\cdot \|e^t\| + 2\sum_{t=0}^{T-1}\gamma\|h^t\|\cdot \|\hat e^t\|
        . \notag  
    \end{align*}
   
    Dropping the non-negative term $\|\wtilde x^T - x\|^2$ and using $\hat{x}^0 = x^0$ we obtain
    \begin{equation*}
        2\sum_{t=0}^{T-1} \gamma\<h^t, x^t - x> 
        \le \| x^0 - x\|^2 
        + \sum_{t=0}^{T-1}\gamma^2 \|h^t\|^2
        + 2\sum_{t=0}^{T-1}\gamma^2\|h^t\|\cdot \|e^t\|
        + 2\sum_{t=0}^{T-1}\gamma\|h^t\|\cdot \|\hat e^t\|.
    \end{equation*}
    Now we split the sum over $\cN$ and $\cB$. For $t\in\cB$ we have, $h^t_i = f^\prime_i(x^t),$ i.e. $h^t = f^\prime (x^t),$ and for $t\in\cN$ $h_i^t = g^\prime_i(x^t),$ i.e. $h^t = g^\prime(x^t).$ Therefore, armed with the convexity of $f$ and $g$ we have
    \begin{align*}
        \<f^\prime(x^t), x^t - x> &\ge f(x^t) - f(x), \quad \forall k \in\cB,\\
         \<g^\prime(x^t), x^t - x> &\ge g(x^t) - g(x) \ge c - g(x), \quad \forall k \in\cN.
    \end{align*}
    Therefore, we have
    \begin{align*}
         \sum_{t\in\cB} \gamma(f(x^t)-f(x))
        + \sum_{t\in\cN} \gamma[c - g(x)] 
        &\leq \sum_{t\in\cB}\gamma\<f^\prime(x^t), x^t-x>
        + \sum_{t\in\cN}\gamma\<g^\prime(x^t), x^t-x>\notag \\
        &\le 
        \frac{1}{2}\| x^0 - x\|^2 
        + \frac{1}{2}\sum_{t=0}^{T-1}\gamma^2 \|h^t\|^2
        + \sum_{t=0}^{T-1}\gamma^2\|h^t\|\cdot \|e^t\|
        + \sum_{t=0}^{T-1}\gamma\|h^t\|\cdot \|\hat e^t\|.
    \end{align*}
\end{proof}

We now present the main convergence theorem, providing explicit bounds under appropriate conditions on $\gamma$ and $c.$ To do so, we need to define $\overline{x}^T$ as follows
\begin{eqnarray}\label{eq:def_x_bar_t_bi}
    \overline{x}^T \eqdef \frac{1}{\sum_{t\in\cB}\gamma}\sum_{t\in\cB}\gamma x^t = \frac{1}{|\cB|}\sum_{t\in\cB}x^t.
\end{eqnarray}
\begin{lemma}
    Suppose that the stepsize $\gamma$ and threshold $c$ satisfy
    \begin{equation}\label{eq:dnqwjdnqwqwdw}
        \frac{T}{2}\gamma c > \frac{1}{2}R^2
        + \frac{1}{2}M^2\gamma^2 T
        + M^2\gamma^2\frac{2\sqrt{1-\delta}}{\delta}T
        + M^2\gamma^2\frac{2\sqrt{10(1-\delta_{\rm s})}}{\delta_{\rm s}\delta}T.
    \end{equation}
    Then we have 
    \begin{equation}\label{eq:bkjnkjq3}
        \gamma\E{\sum_{t\in\cB} f(x^t) - f(x^*)} + \gamma\E{\sum_{t\in\cN}c-g(x^*)} \le \frac{1}{2}R^2
        + \frac{1}{2}M^2\gamma^2 T
        + 2M^2\gamma^2\frac{2\sqrt{1-\delta}}{\delta}T
        + 2M^2\gamma^2\frac{\sqrt{10(1-\delta_{\rm s})}}{\delta_{\rm s}\delta}T.
    \end{equation}
    Moreover, suppose that \eqref{eq:bkjnkjq3} holds. Then $\cB$ is non-empty, i.e. $\overline{x}^T$ is well-defined, and one of the two following conditions holds
    \begin{enumerate}
        \item $|\cB| \ge \frac{T}{2},$ or 
        \item $\gamma\E{\sum_{t\in\cB}  f (x^t) - f(x^*)} \le 0.$
    \end{enumerate}
\end{lemma}

\begin{proof} Let us use $x=x^*$ in \Cref{lem:lemma2_bi}. Taking the expectation and using the fact that $\|h^t\| \le M$, we get 
    \begin{align}\label{eq:qwdqklokgmwpnqwe}
        \E{\gamma\sum_{t\in\cB} f(x^t) - f(x^*) } + \E{\gamma\sum_{t\in\cN}c-g(x^t)} &\le \frac{1}{2}R^2 
        + \frac{1}{2}M^2\sum_{t=0}^{T-1}\gamma^2
        + M\sum_{t=0}^{T-1}\gamma^2\E{\|e^t\|}\\
        &\quad +\; M\sum_{t=0}^{T-1}\gamma\E{\|\hat e^t\|}.
    \end{align}
    Using the properties of the compressors $\{\cC_i\}_{i=1}^n$, we get by induction\footnote{The base of induction obviously holds since $\|e_i^0\| = 0$.} that (with the choice $\eta = \frac{\delta}{2(1-\delta)}$)
    \begin{align*}
        \E{\|e^{t+1}\|^2} &= \E{\left\|\frac{1}{n}\sum_{i=1}^ne_i^{t+1}\right\|^2} 
        \le \frac{1}{n}\sum_{i=1}^n \E{\|e_i^{t+1}\|^2}
        = \frac{1}{n}\sum_{i=1}^n\E{\|e_i^t + h_i^t - \cC_i(e_i^t + h_i^t)\|^2}\notag\\
        &\le \frac{1-\delta}{n}\sum_{i=1}^n\E{\|e_i^t + h_i^t\|^2}\notag\\
        &\le (1-\delta)\left(1 + \eta\right)\frac{1}{n}\sum_{i=1}^n\E{\|e^t_i\|^2}
        + (1-\delta)\left(1 + \eta^{-1}\right)M^2\notag\\
        &\le \sum_{l=0}^{t}[(1-\delta)(1+\eta)]^{t-l}(1-\delta)(1+\eta^{-1})M^2\notag\\
        &\le \frac{(1-\delta)(1+\eta^{-1})}{1 - (1-\delta)(1+\eta)}M^2
        = \frac{(1-\delta)(1+\eta^{-1})}{\delta - \eta(1-\delta)}M^2
        = \frac{2(1-\delta)(1+\eta^{-1})}{\delta}M^2 \le \underbrace{\frac{4(1-\delta)}{\delta^2}M^2}_{\eqcolon C^2}.
    \end{align*}
    Similarly, we bound $\E{\|\hat e^t\|^2}$
    \begin{align}
        \E{\|\hat e^{t+1}\|^2} &= \E{\|w^{t+1} - x^{t+1}\|^2}
        = \E{\|w^{t+1} - x^t - \cC(w^{t+1} - x^t)\|^2}\notag\\
        &\le (1-\delta_{\rm s})\E{\|w^{t+1} - x^t\|^2}\notag \\
        &= (1-\delta_{\rm s})\E{\|w^t - \gamma v^t - x^t\|^2} = (1-\delta_{\rm s})\E{\|\hat e^t - \gamma v^t\|^2}\notag\\
        &\le (1-\delta_{\rm s})(1+\hat\eta)\E{\|\hat e^t\|^2} 
        + (1-\delta_{\rm s})(1+\hat\eta^{-1})\gamma^2\E{\|v^t\|^2}.\label{eq:tnjinqfjnfqoqwer}
    \end{align}
    Note that 
    \begin{align*}
         \E{\left\|\frac{1}{n}\sum_{i=1}^n e_i^t + h_i^t\right\|^2} 
        &\le \frac{2}{n}\sum_{i=1}^n\E{\|e_i^t\|^2} + \E{\|h_i^t\|^2}\notag\\
        &\le \frac{2}{n}\sum_{i=1}^n \left(\frac{4(1-\delta)}{\delta^2}M^2 + M^2\right) 
        = 2M^2\frac{4(1-\delta)+\delta^2}{\delta^2}
        \le \frac{10M^2}{\delta^2}.
    \end{align*}
    Therefore, 
    \begin{align*}
        \E{\|v^t\|^2} &\leq \frac{2}{n}\sum_{i=1}^n\E{\left\| v^t_i -  (e_i^t + h_i^t)\right\|^2} + 2  \E{\left\|\frac{1}{n}\sum_{i=1}^n (e_i^t + h_i^t)\right\|^2} \notag \\
        &\le 2 (1-\delta) \frac{1}{n}\sum_{i=1}^n\E{\left\| e_i^t + h_i^t\right\|^2}  + \frac{2}{n}\sum_{i=1}^n\E{\left\| e_i^t + h_i^t\right\|^2}  \notag\\
        &\le \frac{8}{n}\sum_{i=1}^n\E{\| e_i^t\|^2}  + \frac{8}{n}\sum_{i=1}^n \E{\|h_i^t\|^2}  \notag\\
        &\leq \frac{40M^2}{\delta^2}. \notag
    \end{align*}
    Then we continue \eqref{eq:tnjinqfjnfqoqwer} as follows
    \begin{align*}
         \E{\|\hat e^{t+1}\|^2} &\le \sum_{l=0}^t[(1-\delta_{\rm s})(1+\hat{\eta})]^{t-l}(1-\delta_{\rm s})(1+\hat{\eta}^{-1})\gamma^2\cdot \frac{40M^2}{\delta^2}\notag\\
         &\le \frac{(1-\delta_{\rm s})(1+\hat{\eta}^{-1})}{1-(1-\delta_{\rm s})(1+\hat{\eta})}\gamma^2\cdot \frac{40M^2}{\delta^2}\notag\\
         &\le \gamma^2\underbrace{\frac{160(1-\delta_{\rm s})M^2}{\delta_{\rm s}^2\delta^2}}_{\eqdef B^2},
    \end{align*}
    i.e. $\E{\|e^t\|} \le C$ and $\E{\|\hat e^t\|} \le \gamma B.$ Therefore, we continue \eqref{eq:qwdqklokgmwpnqwe} as follows
    \begin{align}\label{eq:ytrewq}
        \E{\gamma\sum_{t\in\cB} f(x^t) - f(x^*) } + \E{\gamma\sum_{t\in\cN}c-g(x^t)} &\le \frac{1}{2}R^2 
        + \frac{1}{2}M^2\gamma^2T
        + M^2\gamma^2\frac{2\sqrt{1-\delta}}{\delta}T\notag\\
        &\quad +\; M^2\gamma^2\frac{4\sqrt{10(1-\delta_{\rm s})}}{\delta_{\rm s}\delta}T.
    \end{align}
    Assume that $\cB=\emptyset$, then we have using the fact that  $g(x^*) \le 0$ 
    \begin{align*}
        T\gamma c \le \frac{1}{2}R^2
        + \frac{1}{2}M^2\sum_{t=0}^{T-1} \gamma^2
        + M \sum_{t=0}^{T-1}\gamma^2\|e^t\|
        + M^2\gamma^2\frac{2\sqrt{1-\delta}}{\delta}T
        + M^2\gamma^2\frac{4\sqrt{10(1-\delta_{\rm s})}}{\delta_{\rm s}\delta}T.
    \end{align*}
    This contradicts the assumption of the lemma \eqref{eq:dnqwjdnqwqwdw}. Therefore, we must have $\cB \neq \emptyset.$
    If we have 
    \begin{equation*}
        \gamma\E{\sum_{t\in\cB} f(x^t) - f(x^*)} \le 0,
    \end{equation*}
    then part 2. holds automatically. If we have the opposite, i.e.
    \begin{eqnarray*}
        \gamma\E{\sum_{t\in\cB} f(x^t) - f(x^*)} > 0,
    \end{eqnarray*}
    then from \eqref{eq:ytrewq} we have 
    \begin{equation*} 
        \gamma\E{\sum_{t\in\cN}(c-g(x^*))} \le \frac{1}{2}R^2 
        + \frac{1}{2}M^2\gamma^2T
        + M^2\gamma^2\frac{2\sqrt{1-\delta}}{\delta}T
        + M^2\gamma^2\frac{4\sqrt{10(1-\delta_{\rm s})}}{\delta_{\rm s}\delta}T.
    \end{equation*}
    Since $g(x^*) \le 0$, we have $c - g(x^*) \ge c$. Therefore, we have 
    \begin{equation}\label{eq:bmvwoefoqef}
        \E{\sum_{t\in\cN}\gamma c} \le \frac{1}{2}R^2 
        + \frac{1}{2}M^2\gamma^2T
        + M^2\gamma^2\frac{2\sqrt{1-\delta}}{\delta}T
        + M^2\gamma^2\frac{4\sqrt{10(1-\delta_{\rm s})}}{\delta_{\rm s}\delta}T.
    \end{equation}
    Assume $|\cB| < \frac{T}{2},$ this means that $|\cN| \ge \frac{T}{2}.$ Therefore, from \eqref{eq:bmvwoefoqef} we derive 
    \begin{equation*}
        \frac{T}{2}\gamma c \le \E{\sum_{t\in\cN}\gamma c} \le \frac{1}{2}R^2 
        + \frac{1}{2}M^2\gamma^2T
        + M^2\gamma^2\frac{2\sqrt{1-\delta}}{\delta}T
        + M^2\gamma^2\frac{4\sqrt{10(1-\delta_{\rm s})}}{\delta_{\rm s}\delta}T,
    \end{equation*}
    which contradicts \eqref{eq:dnqwjdnqwqwdw}. Therefore, $|\cB| \ge \frac{T}{2},$ i.e. part $1.$ holds.

\end{proof}

Now we are ready to prove our main theorem.
    \begin{theorem}\label{th:ef14_bi_proof_in_exp}
        Suppose that $\gamma$ and $c$ are chosen such that \eqref{eq:dnqwjdnqwqwdw} holds. Then we have 
        \begin{align*}
            \E{f(\overline{x}^T) - f(x^*)} &\le \frac{R^2}{\gamma T} 
        + M^2\gamma
        + 4M^2\gamma\frac{\sqrt{1-\delta}}{\delta}
        + 8M^2\gamma\frac{\sqrt{10(1-\delta_{\rm s})}}{\delta_{\rm s}\delta},\\
            \E{g(\overline{x}^T)} \le c.
        \end{align*}
    \end{theorem}
    \begin{proof}
        We start by using the results of \Cref{lem:lemma2_bi}. Using convexity of $f$ and Jensen inequality we get that if part $2.$ of \Cref{lem:lemma2_bi} holds, we have 
        \begin{equation*}
            \E{f(\overline{x}^T) - f(x^*)} \le 0.
        \end{equation*}
        If part $2.$ does not hold, then we must have $|\cB| \ge \frac{T}{2}.$ Since $g(x^*) \le 0$, from \eqref{eq:bkjnkjq3} we obtain
        \begin{equation*}
        \gamma\E{\sum_{t\in\cB}f(x^t) - f(x^*)} \le \frac{1}{2}R^2
        + \frac{1}{2}M^2\gamma^2 T
        + M^2\gamma^2\frac{2\sqrt{1-\delta}}{\delta}T
        + M^2\gamma^2\frac{4\sqrt{10(1-\delta_{\rm s})}}{\delta_{\rm s}\delta}T.
    \end{equation*}
    This implies that 
    \begin{align*}
        \E{f(\overline{x}^T) - f(x^*)} 
        &\le \frac{2}{\gamma T}\left(\frac{1}{2}R^2
        + \frac{1}{2}M^2\gamma^2 T
        + M^2\gamma^2\frac{2\sqrt{1-\delta}}{\delta}T
        + M^2\gamma^2\frac{4\sqrt{10(1-\delta_{\rm s})}}{\delta_{\rm s}\delta}T.\right)\notag\\
        &= \frac{R^2}{\gamma T} 
        + M^2\gamma
        + 4M^2\gamma\frac{\sqrt{1-\delta}}{\delta}
        + 8M^2\gamma\frac{\sqrt{10(1-\delta_{\rm s})}}{\delta_{\rm s}\delta}.
    \end{align*}
    Since $g(x^t) \le c$ for $t\in\cB$ we get from convexity of $g$ and Jensen inequality that 
    \begin{equation*}
        \E{g(\overline{x}^T)} \le c.
    \end{equation*}
    \end{proof}

    \begin{corollary}\label{cor:stepsize_choice}
        If $\gamma = \frac{R\sqrt{\delta_{\rm s}\delta}}{M\sqrt{T}}$ and $c = \frac{32RM}{\sqrt{\delta_{\rm s}\delta T}}$, then we have 
       \begin{align*}
            \E{f(\overline{x}^T) - f(x^*)} &\le \frac{32MR}{\sqrt{\delta T}},\\
            \E{g(\overline{x}^T)} &\le \frac{32MR}{\sqrt{\delta T}}.
        \end{align*}
    \end{corollary}

    \begin{proof}
        Note that $\gamma c = \frac{R\sqrt{\delta_{\rm s}\delta}}{M\sqrt{T}} \frac{32RM}{ \sqrt{\delta_{\rm s}\delta T}} = \frac{32R^2}{T},$ i.e. $\frac{T}{2}\gamma c = 16R^2$, and 
        \begin{align*}
            & \frac{1}{2}R^2
            + \frac{1}{2}M^2\gamma^2 T
            + M^2\gamma^2\frac{2\sqrt{1-\delta}}{\delta}T 
            + M^2\gamma^2\frac{\sqrt{10(1-\delta_{\rm s})}}{\delta_{\rm s}\delta}T \notag\\
            = & \frac{1}{2}R^2 
            + \frac{1}{2}M^2T \frac{R^2\delta\delta_{\rm s}}{M^2T}
            + M^2T\frac{2\sqrt{1-\delta}}{\delta} \frac{R^2\delta\delta_{\rm s}}{M^2T}
            + M^2T\frac{4\sqrt{10(1-\delta_{\rm s})}}{\delta_{\rm s}\delta}\frac{R^2\delta_{\rm s}\delta}{M^2T}\notag\\
            =& \frac{1}{2}R^2 
            + \frac{1}{2}R^2\delta\delta_{\rm s}
            + 2R^2\sqrt{1-\delta}\delta_{\rm s}
            + 4\sqrt{10(1-\delta_{\rm s})}R^2 \le 16R^2.
        \end{align*}
        Therefore, \eqref{eq:dnqwjdnqwqwdw} is satisfied. Hence, we have from \Cref{th:ef14_bi_proof_in_exp}
        \begin{align*}
            \E{f(\overline{x}^T) - f(x^*)} &\le 
            \frac{R^2}{\frac{R\sqrt{\delta_{\rm s}\delta}}{M\sqrt{T}} T} 
            + M^2\frac{R\sqrt{\delta_{\rm s}\delta}}{M\sqrt{T}}
            + 4M^2\frac{R\sqrt{\delta_{\rm s}\delta}}{M\sqrt{T}}\frac{\sqrt{1-\delta}}{\delta}
            + 8M^2\frac{R\sqrt{\delta_{\rm s}\delta}}{M\sqrt{T}}\frac{\sqrt{10(1-\delta_{\rm s})}}{\delta_{\rm s}\delta}\notag\\
            &= \frac{MR}{\sqrt{\delta_{\rm s}\delta T}} 
            + \frac{MR\sqrt{\delta_{\rm s}\delta}}
            {\sqrt{T}}
            + \frac{4MR\sqrt{(1-\delta)\delta_{\rm s}}}{\sqrt{\delta T}}
            + \frac{8MR\sqrt{10(1-\delta_{\rm s})}}{\sqrt{\delta_{\rm s}\delta}}\notag\\
            &\le \frac{32MR}{\sqrt{\delta_{\rm s}\delta T}},
    \end{align*}
    and 
    \begin{align*}
        g(\overline{x}^T) &\le c = \frac{32MR}{\sqrt{\delta_{\rm s}\delta T}}.
    \end{align*}
    \end{proof}

\newpage 

\section{Lower bound under Communication Compression for Non-smooth Convex Setting}\label{sec:lower_bound_uni_proof}

In this section, we establish a lower bound in non-smooth convex setting, assuming workers can compute exact subgradients $f^\prime(x) \in \partial f(x)$ or $g^\prime(x) \in \partial g(x),$ and the compression is the only source of stochasticity in the training. %Appendix A.1 in \citet{huang2022lower} provides a detailed description of zero-respecting property, therefore, we skip it in our work.
First, in the next subsection, we provide some preliminary background on the class of zero-respecting algorithms following the exposition in \citep{huang2022lower}, and justify that our \algname{Safe-EF} method satisfies this general property. In the subsequent \Cref{subsec:lower_bound_unconstrained,subsec:lower_bound_constrained}, we provide the proof of \Cref{th:lower_bound_uni}.

\subsection{Zero-respecting algorithms}

Let $[x]_j$ denote the $j$-th coordinate of a vector $x\in\R^d$ for $j\in[d]$, and define $\prog(x)$ as 
\begin{equation*}
    \prog(x) \eqdef \begin{cases}
        0 &\text{if } x=0;\\
        \max_{1\le j \le d}\{j \colon [x]_j \neq 0\}, &\text{otherwise}.
    \end{cases}
\end{equation*}

Similarly, for a set of points $X = \{x_1,x_2\dots\}$, we define $\prog(X) \eqdef  \max_{x\in X} \prog(x)$. It holds that $\prog(X \cup Y) = \max\{\prog(X), \prog(Y)\}$ for any $X, Y \subseteq \R^d$, and $\prog(X ) \le \prog(\wtilde X)$ for any $X\subseteq \wtilde X \subseteq \R^d$.

We examine a distributed learning framework incorporating communication compression. For each worker $i$ and time step $t\ge 0$, we denote by $y_i^t$ and $z^t_i$ the points at which worker $i$ queries its subgradient (of $f_i$ and/or $g_i$) and function (of $f_i$ and/or $g_i$) oracles, respectively\footnote{We consider deterministic oracles only.}. In more detail, $O_{i,f_i}(y_i^t, z_i^t)$ returns a pair of the subgradient of $f_i^\prime(y_i^t)$ and the function value $f_i(z_i^t),$ namely, 
$$(f_i^\prime(y_i^t), f_i(z_i^t)) \in O_{i,f_i}(y_i^t, z_i^t) \eqdef (O^{\rm sg}_{i,f_i}(y_i^t,z_i^t), O^{\rm fv}_{i,f_i}(y_i^t,z_i^t)) , $$
where $f_i^\prime(y_i^t) \in \partial f_i(y_i^t)$ is an arbitrary selection of subgradient element from subdifferential of $f_i$ at the point $y_i^t.$ We assume similarly the oracle for each constraint function $g_i$, $O_{i,g_i}(y_i^t, z_i^t)$ which returns a pair $(g_i^\prime(y_i^t), g_i(z_i^t)),$ where $g_i^\prime(y_i^t) \in \partial g_i(y_i^t).$ Additionally, $x^t_i$ represents the local model updated by worker $i$ after the $t$-th query. It is important to note that $y^t_i$ and $z^t_i$ are not necessarily equal to the previous local model $x^{t-1}_i$; instead, they may serve as auxiliary vectors.

Between the $(t-1)$-th and $t$-th gradient queries, each worker is allowed to communicate with the server
by transmitting (compressed) vectors. For worker $i$, we let $\cV_{w_i\to s}^t$ denote the set of vectors that worker $i$
aims to send to the server, i.e., the vectors before compression. Due to communication compression, the
vectors received by the server from worker $i$, which we denote by $\cV^{t,\star}_{w_i\to s}$, are the compressed version of $\cV^{t}_{w_i\to s} = \cC_i(\cV^{t,\star}_{w_i\to s})$\footnote{The compression is performed vector-wise.} with some underlying compressors $\cC_i.$ Note that $\cV^t_{w_i\to s}$ is a set that may include
multiple vectors, and its cardinality equals the rounds of communication. After receiving the compressed vectors from all workers, the server will broadcast some vectors back to all workers. We let $\cV^{t}_{s\to w}$ denote the set of vectors that the server aims to send to workers.  Since we consider the setting with unidirectional compression only, then $\cV^{t}_{s\to w}\equiv \cV^{t,\star}_{s\to w}$\footnote{$\cV_{s\to w}^{t,\star} = \cC_0(\cV_{s\to w}^{t})$.}.

We now extend the zero-respecting property \citep{huang2022lower} to distributed learning with communication compression with functional constraints.
\begin{definition}
    We say a distributed algorithm $A$ is zero-respecting if for any $t \ge 0$ and $1 \le  k \le d$, the following requirements are satisfied:
    \begin{enumerate}
        \item If worker $i$ queries at $y^t_i$ and $z^t_i$ with  $[y^t_i]_k \neq 0$, then one of the following must be true:
        \begin{equation*}
            \begin{cases}
                \text{there exists some } 0\le s < t \text{ such that } [x_i^s]_k \neq 0;\\
                \text{there exists some } 1\le s < t \text{ such that } [O_{f_i,i}{}(y_i^s)]_k \neq 0 \text{ or } [O_{g_i,i}(y_i^s)]_k \neq 0;\\
                \text{there exists some } 1\le s < t \text{ such that worker } i \text{ has received some } v\in\cV_{s\to w}^t \text{ with } [v]_k \neq 0;\\
                \text{there exists some } 1\le s < t \text{ such that worker } i \text{ has compressed some } v\in\cV_{w_i\to s}^t \text{ with } [v]_k \neq 0;
            \end{cases}
        \end{equation*}

        \item If the local model $x_i^t$ of worker $i$, after $t$-th query, has $[x_i^t]_k\neq 0$, then one of the following must be true:
        \begin{equation*}
            \begin{cases}
                \text{there exists some } 0\le s < t \text{ such that } [x_i^s]_k \neq 0;\\
                \text{there exists some } 1\le s < t \text{ such that } [O_{f_i,i}{}(y_i^s)]_k \neq 0 \text{ or } [O_{g_i,i}(y_i^s)]_k \neq 0;\\
                \text{there exists some } 1\le s < t \text{ such that worker } i \text{ has received some } v\in\cV_{s\to w}^t \text{ with } [v]_k \neq 0;\\
                \text{there exists some } 1\le s < t \text{ such that worker } i \text{ has compressed some } v\in\cV_{w_i\to s}^t \text{ with } [v]_k \neq 0;
            \end{cases}
        \end{equation*}

        \item If worker $i$ aims to send some $v\in\cV_{w_i\to s}^t$ with $[v]_k\neq 0$, then one of the following must be true:
        \begin{equation*}
            \begin{cases}
                \text{there exists some } 0\le s < t \text{ such that } [x_i^s]_k \neq 0;\\
                \text{there exists some } 1\le s < t \text{ such that } [O_{i,f_i}{}(y_i^s)]_k \neq 0 \text{ or } [O_{i,g_i}(y_i^s)]_k \neq 0;\\
                \text{there exists some } 1\le s < t \text{ such that worker } i \text{ has received some } v^\prime \in\cV_{s\to w}^t \text{ with } [v^\prime ]_k \neq 0;\\
                \text{there exists some } 1\le s < t \text{ such that worker } i \text{ has compressed some } v^\prime \in\cV_{w_i\to s}^t \text{ with } [v^\prime ]_k \neq 0;
            \end{cases}
        \end{equation*}

        \item If the server aims to broadcast some $v\in\cV_{s\to w}^{t}$ with $[v]_k\neq 0$, then one of the following must be true:
        \begin{equation*}
            \begin{cases}
                \text{there exists some } 1\le s < t \text{ and } 1\le i \le n \text{ such that the server} \text{ has received some } v^\prime \in\cV_{w_i\to s}^s \text{ with } [v^\prime ]_k \neq 0;
            \end{cases}
        \end{equation*}
    \end{enumerate}

\end{definition}

\paragraph{\algname{Safe-EF} is zero-respecting.} Fundamentally, the zero-respecting property ensures that any increase in the number of nonzero coordinates in $x^t_i$, $y_i^t$, or other related vectors at worker $i$ stems from its past local gradient updates, local compression operations, or synchronization with the server. Likewise, any expansion of nonzero coordinates in the server's vectors must result from receiving compressed messages from workers. Notably, this definition explicitly prohibits expanding the set of nonzero entries through function value queries of $f_i$ and/or $g_i$. Therefore, our algorithm class excludes zero-order methods. Nevertheless, function values can be used to set a stepsize or coefficients in linear combination to compute local model $x^t_i.$ For instance, in \algname{Safe-EF} function evaluation of $g_i$ are used to define an update direction:
\begin{equation*}
    h_i^t = f_i^\prime(x^t) \, \1(g(x^t) \le c) + g_i^\prime(x^t) \, (g(x^t) > c).
\end{equation*}
In this case, function values are only used to choose which of the directions, $f_i^\prime(x^t)$ or $g_i^\prime(x^t)$, to follow, but they cannot be used to compute the update direction itself.

\subsection{Lower bound in unconstrained case}\label{subsec:lower_bound_unconstrained} 

We first establish the lower bound in unconstrained setting when $g(x) \equiv 0$, which is the most challenging part of the proof. Without loss of generality, we assume that $x^0 = 0$. Given local loss functions $\{f_i\}_{i=1}^n\subseteq \cF_{R,M}$, compressors $\{\cC_i\}_{i=1}^n \subseteq \cC(\delta)$, and an algorithm $A\in\cA_{\{\cC\}_{i=1}^n}^U$ to solve problem~\eqref{eq:problem}, we let $\hat{x}_{A, \{f_i\}_{i=1}^n, \{\cC_i\}_{i=1}^n, T}$ denote the output of algorithm $A$ using no more than $T$ subgradient queries and rounds of communication by each worker node. Let us define the minimax measure in unconstrained case as 
\begin{equation}\label{eq:minimax_measure}
    \inf_{A\in\cA}\sup_{\{\cC_i\}_{i=1}^n \subseteq \cC(\delta)}\sup_{\{f_i\}_{i=1}^n\subseteq \cF_{R,M}} \E{f(\hat{x}_{A, \{f_i\}_{i=1}^n, \{\cC_i\}_{i=1}^n, T}) - f^*}.
\end{equation}
In \eqref{eq:minimax_measure}, we do not require the compressors $\{\cC_i\}_{i=1}^n$ to be distinct or independent. We allow the compression parameter $\delta$ to be accessible by algorithm $A$. Let $[x]_j$ denote the $j$-th coordinate of a vector $x\in\R^d$ for $j\in[d],$ and define $\prog(x)$ as 
\begin{equation*}
    \prog(x) \eqdef \begin{cases}
        0 & \text{ if } x = 0;\\
        \max_{1\le j \le d}\{j: [x]_j \neq 0\} & \text{ otherwise}.
    \end{cases}
\end{equation*}
In other words, $\prog(x)$ outputs the largest coordinate of input $x$ that corresponds to a non-zero entry. Importantly, $\prog(x)$ satisfies $\prog(\cX\cup \cY) = \max\{\prog(\cX), \prog(\cY)\}$ for any $\cX,\cY\in\R^d$, and $\prog(\cX) \le \prog(\wtilde{\cX})$ for any $\cX \subseteq \wtilde{\cX} \subseteq \R^d$ (see, e.g., \citep{huang2022lower}). Now we are ready to state and prove the lower bound stated in the unconstrained setting.

\begin{theorem}[Unconstrained setting]\label{thm:lower_bound_deterministic}
    For any $R, M > 0, n \ge 2,$ $\delta \le 0.3, T \ge \delta^{-2}$ there exist functions $\{f_i\}_{i=1}^n \subseteq \cF_{R,M}$, compressors $\{\cC_i\}_{i=1}^n \subseteq \cC(\delta),$ oracles $\{\cO_{f_i,i}\}_{i=1}^n$, and the starting point $x^0 = 0$ such that for any first-order algorithm $A \in \cA_{\{\cC_i\}_{i=1}^n}^U$ run for $T \le d$ iterations from $x^0$, satisfies
    \begin{equation*}
        \E{f(\hat{x}_{A,\{f_i\}_{i=1}^n, \{\cC\}_{i=1}^n, T } ) - \min_{x\in \R^d}f(x)} \ge \Omega\left(\frac{MR}{\sqrt{\delta T}}\right).
    \end{equation*}
\end{theorem}

\begin{proof}
    
    \textbf{Step 1.} Let us fix some $R$ and define $\cS \eqdef \left\{x\in\R^d \mid \|x\|_2\le \frac{R}{2}\right\}$. Let $h\colon \R^d \to \R$ be defined as 

    \begin{equation*}
        h(x) \eqdef \begin{cases}
            C\cdot \max\limits_{1\le j\le T} x_j + \frac{\mu}{2}\|x\|^2_2 & \text{ if } x\in\cS,\\
            C\cdot \max\limits_{1\le j\le T} x_j + \frac{\mu R}{4}\|x\|_2 & \text{ if } x\notin\cS.
        \end{cases}
    \end{equation*}
    Here we assume that $T \le d$. The constant $C = \frac{M\sqrt{T}}{1+\sqrt{\delta T}}$ and $\mu = \frac{2M}{R(1+\sqrt{\delta T})}$. This implies that $C = \frac{R\mu\sqrt{T}}{2}.$ Note that it is never optimal to have $[x^*]_j \neq 0$ for $T < j \le d$, and by symmetry, we know that 
    \[
        [x^\star]_1 = \dots = [x^\star]_T.
    \]
    Thus, as long as $C \le \frac{R\mu\sqrt{T}}{2}$ the optimal solution $x^*$ and optimal value of the problem $f^* \eqdef \min_x f(x)$ are given by 
    \[
        [x^*]_j = \begin{cases}
            -\frac{C}{\mu T} & \text{ for } 1\le j \le T,\\
            0 & \text{ for } T < j \le d,
        \end{cases}
        \quad \text{and}
        \quad f^* = - \frac{C^2}{2\mu T}.
    \]
    One can show that the function $h$ is convex. Indeed, this is because taking $\max$ and/or a sum of convex functions preserves convexity. We consider the following subgradient oracle $O_h$ 
    \[
    h^\prime(x) = \begin{cases}
        \mu x + Ce_k & \text{ if } x\in\cS,\\
        \mu R\frac{x}{4\|x\|} + Ce_k & \text{ otherwise},
    \end{cases}
    \]
    where $k$ is the smallest index such that $[x]_k = \max\limits_{1\le j \le T}[x]_j.$  We set $f_i\equiv h$ with $O_i \equiv O_h$ for all $i\in[n].$ Note that the first part of the subgradient (either $\mu x$ or $\mu R\frac{x}{2\|x\|}$) is proportional to $x.$ Therefore, the algorithms are hampered by oracle $O_i$ to reach more non-zero coordinates due to the second part $Ce_k$ only. However, it might increase $\prog(O_i(x))$ at most by one, namely, 
    \begin{equation}\label{eq:nwejfnwekjfnefSD}
        \prog(O_i(x)) \le \prog(x) + 1.
    \end{equation}

    \textbf{Step 2.} Next, we assume that each worker $i$ uses Rand-$K$ compressor with $K = \lceil d\delta\rceil.$ Moreover, we assume that the randomness of the compressors is shared among workers. Then this compressor belongs to $\mathbb{C}(\delta).$ This step ensures there is no speedup of the final rate in the number of workers $n.$

    \textbf{Step 3.} We let $v_{s\to w}^{t}$ be the vector that workers receive from the central server in the $t$-th communication (similar definition is used for $v_{w_i\to s}^t$) and let $x_i^{t}$ be the local model that worker $i$ produces after the $t$-th communication round. Recall that algorithms satisfy the zero-respecting property. Therefore, we find that each worker can only achieve one more non-zero coordinate in the local model by local subgradient updates based on the received messages from the central server. Thus, we have that
    \begin{equation}\label{eq:ghbdwqkodnqw}
        \prog(x_i^{t}) \le \max\limits_{1\le s\le t} \prog(v_{s\to w}^{s}) + 1.
    \end{equation}
    By further noting that vector $v_{s\to w}^{t}$ sent by the central server can be traced back to past vectors received from all workers, we have 
    \begin{equation}\label{eq:sqsdsqscascascasca}
        \prog(v_{s\to w}^{t}) \le \max\limits_{1\le s \le t}\max\limits_{1\le i \le n}\prog(v_{w_i\to s}^{s}).
    \end{equation}
    Combining \eqref{eq:ghbdwqkodnqw} and \eqref{eq:sqsdsqscascascasca}, we reach
    \begin{equation}\label{eq:jfkwdjefwfwe}
        \prog(x_i^{t}) \le \max\limits_{1\le s\le t}\max\limits_{1\le i \le n}\prog(v_{w_i \to s}^{s}) + 1.
    \end{equation}

    \textbf{Step 4.} Let 
    \[
    \hat{x} \in {\rm span}\left(\left\{x_i^{t} \mid 0 \le t \le T, 1 \le i \le n\right\}\right).
    \]
    be the final algorithm output after $T$ subgradient queries on each worker. By \eqref{eq:jfkwdjefwfwe}, we have 
    \begin{equation*}
        \prog(\hat{x}) \le \max\limits_{1\le t \le T}\max_{1\le i \le n}\prog(v_{w_i\to s}^{t}) + 1.
    \end{equation*}
    By \Cref{lem:lemma1_lower_bound_uni}, we have 
    \begin{equation*}
        \Prob(\max\limits_{1\le t\le T}\max\limits_{1\le i \le n}\prog(v_{w_i\to s}^{t}) \ge T - 1) \le \exp\left((e-1)T\lceil d\delta\rceil/ - T + 1\right).
    \end{equation*}
    Note that if $\prog(\hat{x}) < T$ then we have 
    \[
    f(\hat{x}) \ge 0 \Leftrightarrow f(\hat{x}) - f^\star \ge - f^\star = \frac{C^2}{2\mu T}.
    \]
    Therefore, we have 
    \begin{equation*}
        \E{f(\hat{x}) - f^\star} \ge (1-\exp\left((e-1)T\lceil d\delta\rceil/d - T + 1\right))\frac{C^2}{2\mu T}.
    \end{equation*}
    If we let $d = \lfloor 5T\delta \rfloor$ and $T$ to be no less than $\frac{1}{\delta^2}$, we have 
    \[
    d = \lfloor 5T\delta\rfloor \ge 5T\delta - 1 \ge 4T\delta + \frac{1}{\delta^2}\delta - 1 \ge 4T\delta \ge \frac{4}{\delta} \ge 4.
    \]
    Then it is easy to verify 
    \begin{align*}
    (e-1)T\lceil d\delta\rceil/d + 1 - T &\le (e-1)T(d\delta + 1)/d + 1 - T\notag\\
    &= (e-1)T\delta + \frac{(e-1)T}{d} + 1 - T\notag\\
    &\le (e-1)T\delta + (e-1)\frac{T\delta}{4} + 1 - T\notag\\
    &= (e-1)\frac{5T\delta}{4} + 1 - T.
    \end{align*}
    Note that since $\delta \le 0.3$ and $T \ge \frac{1}{\delta^2}$ we have 
    \begin{align*}
        (e-1)\frac{5T\delta}{4} + 1 - T \le -1 \Leftrightarrow T\left(1 - \frac{(e-1)5\delta}{4}\right) \ge 2 \Leftarrow 6.25\cdot \left(1 - \frac{(e-1)5\cdot 0.4}{4}\right) \approx 3.95 > 2.
    \end{align*}
    Since the last inequality holds, then we have $(e-1)\frac{5T\delta}{4} + 1 - T < -1.$ Therefore, this leads to 
    \[
    \E{f(\hat{x}) - f^\star} \ge \Omega\left(\frac{C^2}{2\mu T}\right) = \Omega\left(\frac{M^2T}{(1+\sqrt{\delta T})^2} \frac{1}{2T} \frac{R(1+\sqrt{\delta T})}{2M}\right) = \Omega\left(\frac{MR}{1+\sqrt{\delta T}}\right).
    \]
\end{proof} 

\begin{lemma}[Technical lemma]\label{lem:lemma1_lower_bound_uni}
    In example used in the proof of Theorem~\ref{thm:lower_bound_deterministic}, it holds that 
    \begin{equation*}
        \Prob(\max\limits_{1\le t\le T}\max\limits_{1\le i \le n}\prog(v_{w_i\to s}^{t}) \ge T-1) \le \exp\left((e-1)T\lceil d\delta \rceil/d -T + 1\right).
    \end{equation*}
\end{lemma}
\begin{proof}
    Note that at the $t$-th round of communication where $1\le t\le T$, the non-zero coordinates of $v_{w_i\to s}^{(t,\star)}$, {\it the vector that is to be transmitted by worker $i$ to the server before compression}, are achieved by utilizing previously received vectors $\{v_{s\to w}^{(s)}\colon 1\le s \le t-1\}$ and local subgradient queries. Following the argument in Step 3 of Theorem~\ref{thm:lower_bound_deterministic}, we find that worker $i$ can only achieve one more non-zero coordinate in $v_{w_i\to s}^{(t,\star)}$ by local subgradient updates based on received vectors $\{v_{s\to w}^{(s)} \mid 1 \le s \le t - 1\}.$ Therefore, it holds that 
    \begin{equation}\label{eq:njibfwihebfwiefb}
        \prog(v_{w_i\to s}^{(t,\star)}) 
        \le \max\limits_{1\le s\le t-1}\prog(v_{s\to w}^{(s)}) + 1 
        \le \max\limits_{1\le s \le t-1} \max\limits_{1\le i \le n} \prog(v_{w_i\to s}^{s}) + 1 \eqcolon B^{(t-1)}.
    \end{equation}
    We additionally define $B^{(0)} = 1$. By the definition of $B^{(t)}$ and that 
    \begin{equation}\label{eq:vdcsgerwevwe}
        \prog(v_{w_i\to s}^{t}) \le \prog(v_{w_i\to s}^{(t,\star)}),
    \end{equation} 
    it naturally holds that 
    \begin{align}
        B^{(t-1)} \le B^{(t)} &= \max\limits_{1\le s \le t}\max\limits_{1\le i \le n} \prog(v_{w_i\to s}^{s}) + 1\notag \\
        &= \max\left\{B^{(t-1)}, \max\limits_{1\le i \le n} \prog(v_{w_i\to s}^{t}) + 1\right\}\notag \\
        &\overset{\eqref{eq:vdcsgerwevwe}}{\le} \max\left\{B^{(t-1)}, \max\limits_{1\le i \le n} \prog(v_{w_i\to s}^{(t,\star)}) + 1\right\}\notag\\
        &\overset{\eqref{eq:njibfwihebfwiefb}}{\le} \max\left\{B^{(t-1)}, B^{(t-1)} + 1\right\} \le B^{(t-1)} + 1.\label{eq:efnwefdwfdew}
    \end{align}
    Therefore, one round of communication can increase $B^{(t)}$ at most by $1.$ Moreover, \eqref{eq:efnwefdwfdew} implies that $B^{(t)} = B^{(t-1)} + 1$ only if $\max\limits_{1\le i \le n}\prog (v_{w_i\to s}^{(t,\star)}) = \max\limits_{1\le i \le n} \prog(v_{w_i\to s}^{t}).$ Let $k = \max\limits_{1\le i \le n} \prog(v_{w_i\to s}^{(t,\star)}).$ Recall that the compressors $\{\cC\}_{i=1}^n$ share the randomness, we therefore conclude that having $\max\limits_{1\le i \le n} \prog(v_{w_i\to s}^{t}) = \max\limits_{1\le i \le n}\prog (v_{w_i\to s}^{(t,\star)}) = k$ is equivalent to that coordinate index $k$ is chosen to communicate in communication round $t$, which happens with probability $\frac{K}{d}.$ Therefore, we have 
    \begin{align*}
        \Prob(B^{(t)} = B^{(t-1)} + 1) &\le \Prob(\max\limits_{1\le i \le n} \prog(v_{w_i\to s}^{t}) = \max\limits_{1\le i \le n}\prog (v_{w_i\to s}^{(t,\star)}))\\
        &= \Prob\left(\text{the coordinate index } \max\limits_{1\le i \le n}\prog(v_{w_i\to s}^{(t,\star)}) \text{ is chosen at round } t\right) = \frac{K}{d}.
    \end{align*}
    Let us define the event $E^t = \{\text{the coordinate index } \max\limits_{1\le i \le n}\prog(v_{w_i\to s}^{(t,\star)}) \text{ is chosen at round } t\}$. Since the compression happens uniformly at random, we have $\1(E^{(1)}), \dots, \1(E^t)$ are i.i.d. Be$(\frac{K}{d})$ random variables where $\1(\cdot)$ is the indicator function. By the above argument, we also have $B^{(t)} - B^{(t-1)} \le \1(E^t)$ for any $1\le t\le T.$ As a result, we reach by Markov's inequality 
    \begin{align*}
        \Prob(B^{(T)} \ge T) &= \Prob(e^{B^{(T)}} \ge e^{T})\\
        &\le e^{-T}\E{\exp\left(B^{(T)}\right)}\\
        &=e^{-T}\E{\exp\left(B^{(0)} + \sum_{t=1}^T(B^{(t)} - B^{(t-1)})\right)}\\
        &\le e^{-T}\E{e^{B^{(0)}}}\prod_{t=1}^T\E{\exp\left(\1(E^t)\right)}\\
        &= e^{-(T-1)}\prod_{t=1}^T\left((1-\frac{K}{d})\cdot 1 + \frac{K}{d}\cdot e\right)\\
        &=e^{-(T-1)}\prod_{t=1}^T\left(1 + \frac{K}{d}( e-1)\right)\\
        &\le e^{-(T-1)}\prod_{t=1}^Te^{(e-1)K/d}\\
        &= e^{(e-1)TK/d - T + 1}.
    \end{align*}
    This concludes the proof of the lemma.
\end{proof}

\subsection{Proof of \Cref{th:lower_bound_uni} (constrained case)}\label{subsec:lower_bound_constrained} 

Now we are ready to extend the proof of \Cref{th:lower_bound_uni} to constrained setting based on the construction in \Cref{thm:lower_bound_deterministic}. Notice that the function classes $\cF_{R,M}$ and $\mathcal G_{R, M}$ for objective and constraints have the same properties: convex with $M$-bounded subgradients. Moreover, in the construction of \Cref{thm:lower_bound_deterministic}, all functions $f_i$ are identical and equal to $f$. Thus, we can set $g_i(x) \eqdef f(x) - \min_{y\in \R^d} f(y)$ for all $i\in [n].$ Then such problem is in the class $\mathcal H_{R, M}$ by construction, and it has a unique feasible point, $x^*$, which also coincides with the solution to unconstrained problem $\min_{x\in \R^d} f(x)$. Since $\partial f_i(x) = \partial g_i(x)$ for any $x\in \R^d$ and all $i\in [n]$, the trajectory of zero-respecting algorithm on the unconstrained problem $\min_{x\in \R^d} f(x)$ and the constrained problem $$\min_{x\in \R^d} f(x) \qquad \text{s.t. }\qquad  g(x) \leq 0$$  are identical. Therefore, the statement of \Cref{thm:lower_bound_deterministic} implies the lower bound in \Cref{th:lower_bound_uni}.

\newpage 

\section{Convergence Upper Bound for \algname{Safe-EF} in Stochastic Setting}\label{sec:stochastic_proofs}

We first recall a standard concentration inequality result for sub-Gaussian random vector.

\begin{lemma}[Lemma C.3 from \citet{gorbunov2019optimal}]\label{lem:concentration_lemma} Let $\{\xi_k\}_{k=1}^N$ be the sequence of random vectors with values in $\R^n$ such that 
\begin{equation*}
    \E{\xi_k \mid \xi_{k-1}, \dots, \xi_1} = 0 \text{ a.s. } \forall~k\in \{1,\dots, N\},
\end{equation*}
and set $S_N \eqdef \sum_{k=1}^N\xi_k$. Assume that the sequence is $\{\xi_k\}_{k=1}^N$ are sub-Gaussian, i.e., 
\begin{equation*}
    \E{\exp(\nicefrac{\|\xi_k\|^2}{\sigma_k^2}\mid \xi_{k-1}, \dots, \xi_1} \le \exp(1) \text{ a.s. } \forall~k\in\{1,\dots, N\},
\end{equation*}
where $\sigma_1, \dots, \sigma_N$ are some positive numbers. Then for all $b \ge 0$ we have 
\begin{equation*}
    \Prob\left(\|S_N\| \ge (\sqrt{2}+\sqrt{2}b)\sqrt{\sum_{k=1}^N\sigma^2_k}\right) \le \exp(-\nicefrac{b^2}{3}).
\end{equation*}
    
\end{lemma}

We first establish several lemmas. 

\begin{lemma}\label{lem:bound_et}
    Assume that \Cref{asmp:subgaussian_noise} holds. Assume that the compressors $\{\cC_i\}_{i=1}^n$ are deterministic (e.g., Top-$K$). Then for all $t\ge 0$ and $i\in[n]$ we have $\|e_i^t\|^2 \le \frac{4(1-\delta)}{\delta^2}M^2.$
\end{lemma}

\begin{proof}
    Using the properties of the compressors $\{\cC_i\}_{i=1}^n$, we get by induction\footnote{The base of induction obviously holds since $\|e_i^0\| = 0$.} that (with the choice $\eta = \frac{\delta}{2(1-\delta)}$)
    \begin{align*}
        \|e^{t+1}\|^2 &= \left\|\frac{1}{n}\sum_{i=1}^ne_i^{t+1}\right\|^2
        \le \frac{1}{n}\sum_{i=1}^n \|e_i^{t+1}\|^2
        = \frac{1}{n}\sum_{i=1}^n\|e_i^t + h_i^t - \cC_i(e_i^t + h_i^t)\|^2\notag\\
        &\le \frac{1-\delta}{n}\sum_{i=1}^n\|e_i^t + h_i^t\|^2\notag\\
        &\le (1-\delta)\left(1 + \eta\right)\frac{1}{n}\sum_{i=1}^n\|e^t_i\|^2
        + (1-\delta)\left(1 + \eta^{-1}\right)M^2\notag\\
        &\le \sum_{l=0}^{t}[(1-\delta)(1+\eta)]^{t-l}(1-\delta)(1+\eta^{-1})M^2\notag\\
        &\le \frac{(1-\delta)(1+\eta^{-1})}{1 - (1-\delta)(1+\eta)}M^2
        = \frac{(1-\delta)(1+\eta^{-1})}{\delta - \eta(1-\delta)}M^2
        = \frac{2(1-\delta)(1+\eta^{-1})}{\delta}M^2 \le \frac{4(1-\delta)}{\delta^2}M^2,
    \end{align*}
    which concludes the proof.
\end{proof}

\begin{theorem}\label{th:theorem_stoch}
    Let Assumptions~\ref{asmp:subgaussian_noise} and \ref{asmp:stoch_convexity} hold. Let $\beta\in (0,1)$ be the failure probability. Suppose $\gamma^2w_t \le \frac{n}{32M^2}.$\footnote{This restriction is needed to apply Lemma 2.2 from \citet{liu2023high}.} For every $0\le t \le T-1$ we have 
    \begin{equation*}
        \E{\exp(S_t) \mid \cF_t} \le \exp\left(48M^2\sum_{l=t}^{T-1}\gamma^2w_l + \frac{8\fvsigma^2}{n\fvN}\sum_{l=t}^{T-1}w_l^2\gamma^2\right),
    \end{equation*}
    where $S_t$ is defined in \eqref{eq:def_St}.
\end{theorem}
\begin{proof}
        We use the same definition of $\wtilde x^t$ established in \eqref{eq:def_x_hat_t}:
        \begin{equation}\label{eq:def_x_hat_t_stoch}
            \wtilde x^{t} = x^{t} - \gamma e^t \quad \text{where} \quad \wtilde x^0 = x^0.
        \end{equation}
        We start by extending the norm of $\|\wtilde x^{t+1} - x\|^2$:
    \begin{align*}
        \|\wtilde x^{t+1} - x\|^2 &= 
        \|\wtilde x^t - x^*\|^2 
        - 2\gamma\<h^t, \wtilde x^t-x^*> 
        + \gamma^2\|h^t\|^2\notag\\
        &= \|\wtilde x^t - x^*\|^2 
        - 2\gamma\<h^t, x^t-x^*> 
        - 2\gamma\<h^t, \wtilde x^t - x^t>
        + \gamma^2\|h^t\|^2.
    \end{align*}

    Rearranging terms gives us 
    \begin{align}\label{eq:nofjafaq}
        2\gamma\<h^t, x^t-x> &= \|\wtilde x^t - x\|^2
        - \|\wtilde x^{t+1} - x\|^2
        - 2\gamma\<h^t, \wtilde x^t - x^t>
        + \gamma^2\|h^t\|^2.
    \end{align}
    Note that for $t\in\cN$ we have $h^t = \frac{1}{n}\sum_{i=1}^n g_i^\prime(x^t,\xi^t_i)$, and for $t\in\cB$ we have $h^t = \frac{1}{n}\sum_{i=1}^n f_i^\prime(x^t,\xi^t_i).$ Therefore, we get from \eqref{eq:nofjafaq}
    \begin{align}\label{eq:egnowjnefq}
        &\frac{2\gamma}{n}\sum_{i=1}^n \<g_i^\prime(x^t,\xi^t_i), x^t-x>\1(t\in\cN)
        + \frac{2\gamma}{n}\sum_{i=1}^n \<f_i^\prime(x^t,\xi^t_i), x^t-x>\1(t\in\cB)\notag\\
        &\quad \le \|\wtilde x^t - x\|^2
        - \|\wtilde x^{t+1} - x\|^2
        - 2\gamma\<h^t, \wtilde x^t - x^t>\notag\\
        &\qquad +\; \frac{\gamma^2}{n}\sum_{i=1}^n\|g_i^\prime(x^t,\xi^t_i)\|^2\1(t\in\cN)
        + \frac{\gamma^2}{n}\sum_{i=1}^n\|f_i^\prime(x^t,\xi^t_i)\|^2\1(t\in\cB)\notag\\
        &\quad \quad \le \|\wtilde x^t - x\|^2
        - \|\wtilde x^{t+1} - x\|^2
        - 2\gamma\<h^t, \wtilde x^t - x^t>\notag\\
        &\qquad +\; \frac{2\gamma^2}{n}\sum_{i=1}^n\|g_i^\prime(x^t)\|^2\1(t\in\cN)
        + \frac{2\gamma^2}{n}\sum_{i=1}^n\|g_i^\prime(x^t) - g_i^\prime(x^t,\xi^t_i)\|^2\1(t\in\cN)\notag\\
        &\qquad +\; \frac{2\gamma^2}{n}\sum_{i=1}^n\|f_i^\prime(x^t)\|^2\1(t\in\cB)
        + \frac{2\gamma^2}{n}\sum_{i=1}^n\|f_i^\prime(x^t) - f_i^\prime(x^t,\xi^t_i)\|^2\1(t\in\cB).
    \end{align}
    Note that we have 
    \begin{align*}
        |\<h^t,\wtilde{x}^t - x^t>| &\le \|h^t\|\cdot \gamma\|e^t\|\notag\\
        &\le  M \cdot \gamma\frac{2\sqrt{1-\delta}}{\delta}M=\frac{2\sqrt{1-\delta}}{\delta}\gamma M^2.
    \end{align*}
    Therefore, we continue from \eqref{eq:egnowjnefq} as follows
    \begin{align*}
        &\frac{2\gamma}{n}\sum_{i=1}^n \<g_i^\prime(x^t,\xi^t_i), x^t-x>\1(t\in\cN)
        + \frac{2\gamma}{n}\sum_{i=1}^n \<f_i^\prime(x^t,\xi^t_i), x^t-x>\1(t\in\cB)\notag\\
        &\quad \le \|\wtilde x^t - x\|^2
        - \|\wtilde x^{t+1} - x\|^2
        + \frac{4\sqrt{1-\delta}}{\delta}\gamma^2M^2
        + 2\gamma^2M^2\notag\\
        &\quad + \frac{2\gamma^2}{n}\sum_{i=1}^n\|g_i^\prime(x^t) - g_i^\prime(x^t,\xi^t_i)\|^2\1(t\in\cN)
        + \frac{2\gamma^2}{n}\sum_{i=1}^n\|f_i^\prime(x^t) - f_i^\prime(x^t,\xi^t_i)\|^2\1(t\in\cB).
    \end{align*}
    We add and subtract full subgradients and derive
    \begin{align*}
        &\frac{2\gamma}{n}\sum_{i=1}^n \<g_i^\prime(x^t), x^t-x>\1(t\in\cN)
        + \frac{2\gamma}{n}\sum_{i=1}^n \<f_i^\prime(x^t), x^t-x>\1(t\in\cB)\notag\\
        &\quad \le \|\wtilde x^t - x\|^2
        - \|\wtilde x^{t+1} - x\|^2
        + \frac{4\sqrt{1-\delta}}{\delta}\gamma^2M^2
        + 2\gamma^2M^2\notag\\
        &\quad + \; \frac{2\gamma}{n}\sum_{i=1}^n \<g_i(x^t)-g_i^\prime(x^t,\xi^t_i), x^t-x>\1(t\in\cN)
        + \frac{2\gamma}{n}\sum_{i=1}^n \<f_i^\prime(x^t) - f_i^\prime(x^t,\xi^t_i), x^t-x>\1(t\in\cB)\notag\\
        &\quad + \frac{2\gamma^2}{n}\sum_{i=1}^n\|g_i^\prime(x^t) - g_i^\prime(x^t,\xi^t_i)\|^2\1(t\in\cN)
        + \frac{2\gamma^2}{n}\sum_{i=1}^n\|f_i^\prime(x^t) - f_i^\prime(x^t,\xi^t_i)\|^2\1(t\in\cB).
    \end{align*}
    Now we use convexity of $g_i$ and $f_i$ to derive 
    \begin{align*}
        &\frac{2\gamma}{n}\sum_{i=1}^n (g_i(x^t)-g_i(x))\1(t\in\cN)
        + \frac{2\gamma}{n}\sum_{i=1}^n (f_i(x^t)- f_i(x))\1(t\in\cB)\notag\\
        &\quad \le \|\wtilde x^t - x\|^2
        - \|\wtilde x^{t+1} - x\|^2
        + \frac{4\sqrt{1-\delta}}{\delta}\gamma^2M^2
        + 2\gamma^2M^2\notag\\
        &\quad + \; \frac{2\gamma}{n}\sum_{i=1}^n \<g_i(x^t)-g_i^\prime(x^t,\xi^t_i), x^t-x>\1(t\in\cN)
        + \frac{2\gamma}{n}\sum_{i=1}^n \<f_i^\prime(x^t) - f_i^\prime(x^t,\xi^t_i), x^t-x>\1(t\in\cB)\notag\\
        &\quad + \frac{2\gamma^2}{n}\sum_{i=1}^n\|g_i^\prime(x^t) - g_i^\prime(x^t,\xi^t_i)\|^2\1(t\in\cN)
        + \frac{2\gamma^2}{n}\sum_{i=1}^n\|f_i^\prime(x^t) - f_i^\prime(x^t,\xi^t_i)\|^2\1(t\in\cB).
    \end{align*}
    We add and substract $\frac{1}{n}\sum_{i=1}^ng_i(x^t,\xi^t_i)$ to obtain
    \begin{align*}
        &\frac{2\gamma}{n}\sum_{i=1}^n (g_i(x^t,\xi^t_i)-g_i(x))\1(t\in\cN)
        + \frac{2\gamma}{n}\sum_{i=1}^n (f_i(x^t)- f_i(x))\1(t\in\cB)\notag\\
        &\quad \le \|\wtilde x^t - x\|^2
        - \|\wtilde x^{t+1} - x\|^2
        + \frac{4\sqrt{1-\delta}}{\delta}\gamma^2M^2
        + 2\gamma^2M^2\notag\\
        &\quad + \; \frac{2\gamma}{n}\sum_{i=1}^n \<g_i(x^t)-g_i^\prime(x^t,\xi^t_i), x^t-x>\1(t\in\cN)
        + \frac{2\gamma}{n}\sum_{i=1}^n \<f_i^\prime(x^t) - f_i^\prime(x^t,\xi^t_i), x^t-x>\1(t\in\cB)\notag\\
        &\quad + \; \frac{2\gamma}{n}\sum_{i=1}^n(g_i(x^t,\xi^t_i) - g_i(x^t))\1(t\in\cN)\notag\\
        &\quad + \frac{2\gamma^2}{n}\sum_{i=1}^n\|g_i^\prime(x^t) - g_i^\prime(x^t,\xi^t_i)\|^2\1(t\in\cN)
        + \frac{2\gamma^2}{n}\sum_{i=1}^n\|f_i^\prime(x^t) - f_i^\prime(x^t,\xi^t_i)\|^2\1(t\in\cB).
    \end{align*}
    Now we set $x=x^*$. Since $\frac{1}{n}\sum_{i=1}^ng_i(x^t,\xi^t_i)\ge c$ for $t\in\cN$ and $g(x^*) \le 0$ we get
    \begin{align*}
        &2\gamma c\1(t\in\cN)
        + \frac{2\gamma}{n}\sum_{i=1}^n (f_i(x^t)- f_i(x))\1(t\in\cB)
        - \|\wtilde x^t - x\|^2
        + \|\wtilde x^{t+1} - x\|^2
        - \frac{4\sqrt{1-\delta}}{\delta}\gamma^2M^2
        - 2\gamma^2M^2\notag\\
        \le \;& \frac{2\gamma}{n}\sum_{i=1}^n \<g_i(x^t)-g_i^\prime(x^t,\xi^t_i), x^t-x>\1(t\in\cN)
        + \frac{2\gamma}{n}\sum_{i=1}^n \<f_i^\prime(x^t) - f_i^\prime(x^t,\xi^t_i), x^t-x>\1(t\in\cB)\notag\\
        &\quad + \; \frac{2\gamma}{n}\sum_{i=1}^n(g_i(x^t,\xi^t_i) - g_i(x^t))\1(t\in\cN)\notag\\
        &\quad + \frac{2\gamma^2}{n}\sum_{i=1}^n\|g_i^\prime(x^t) - g_i^\prime(x^t,\xi^t_i)\|^2\1(t\in\cN)
        + \frac{2\gamma^2}{n}\sum_{i=1}^n\|f_i^\prime(x^t) - f_i^\prime(x^t,\xi^t_i)\|^2\1(t\in\cB).
    \end{align*}
    Let us denote $\omega^t_i \eqdef g_i^\prime(x^t) - g_i^\prime(x^t,\xi^t)$ and $\nu^t_i \eqdef f_i^\prime(x^t) - f_i^\prime(x^,\xi^t_i)$. Then we have 
    \begin{align*}
        &2\gamma c \1(t\in\cN)
        + \frac{2\gamma}{n}\sum_{i=1}^n (f_i(x^t)- f_i(x^*))\1(t\in\cB)
        - \|\wtilde x^t - x^*\|^2
        + \|\wtilde x^{t+1} - x^*\|^2
        - \frac{4\sqrt{1-\delta}}{\delta}\gamma^2M^2
        - 2\gamma^2M^2\notag\\
        \le\;& \frac{2\gamma}{n}\sum_{i=1}^n \<\omega^t_i, x^t-x^*>\1(t\in\cN)
        + \frac{2\gamma}{n}\sum_{i=1}^n \<\nu^t_i, x^t-x^*>\1(t\in\cB)\notag\\
        &\quad + \; \frac{2\gamma}{n}\sum_{i=1}^n(g_i(x^t,\xi^t_i) - g_i(x^t))\1(t\in\cN)
        + \frac{2\gamma^2}{n}\sum_{i=1}^n\|\omega_i^t\|^2\1(t\in\cN)
        + \frac{2\gamma^2}{n}\sum_{i=1}^n\|\nu_i^t\|^2\1(t\in\cB).
    \end{align*}
    We add and subtract $\wtilde x^t$ in some terms to obtain
    \begin{align*}
        &2\gamma c \1(t\in\cN)
        + \frac{2\gamma}{n}\sum_{i=1}^n (f_i(x^t)- f_i(x^*))\1(t\in\cB)
        - \|\wtilde x^t - x^*\|^2
        + \|\wtilde x^{t+1} - x^*\|^2
        - \frac{4\sqrt{1-\delta}}{\delta}\gamma^2M^2
        - 2\gamma^2M^2\notag\\
        \le \;& \frac{2\gamma}{n}\sum_{i=1}^n \<\omega^t_i, \wtilde x^t-x^*>\1(t\in\cN)
        + \frac{2\gamma}{n}\sum_{i=1}^n \<\omega^t_i,x^t - \wtilde x^t>\1(t\in\cN)\notag\\
        &\quad +\; \frac{2\gamma}{n}\sum_{i=1}^n \<\nu^t_i, \wtilde x^t-x^*>\1(t\in\cB)
        + \frac{2\gamma}{n}\sum_{i=1}^n \<\nu^t_i,x^t- \wtilde x^t>\1(t\in\cB)\notag\\
        &\quad + \; \frac{2\gamma}{n}\sum_{i=1}^n(g_i(x^t,\xi^t_i) - g_i(x^t))\1(t\in\cN)
        + \frac{2\gamma^2}{n}\sum_{i=1}^n\|\omega_i^t\|^2\1(t\in\cN)
        + \frac{2\gamma^2}{n}\sum_{i=1}^n\|\nu_i^t\|^2\1(t\in\cB).
    \end{align*}
    Using \eqref{eq:def_x_hat_t_stoch} we derive
    \begin{align*}
        &2\gamma c \1(t\in\cN)
        + \frac{2\gamma}{n}\sum_{i=1}^n (f_i(x^t)- f_i(x^*))\1(t\in\cB)
        - \|\wtilde x^t - x^*\|^2
        + \|\wtilde x^{t+1} - x^*\|^2
        - \frac{4\sqrt{1-\delta}}{\delta}\gamma^2M^2
        - 2\gamma^2M^2\notag\\
        \le \;& \frac{2\gamma}{n}\sum_{i=1}^n \<\omega^t_i, \wtilde x^t-x^*>\1(t\in\cN)
        + \frac{2\gamma^2}{n}\sum_{i=1}^n \<\omega^t_i,e^t>\1(t\in\cN)\notag\\
        &\quad +\; \frac{2\gamma}{n}\sum_{i=1}^n \<\nu^t_i, \wtilde x^t-x^*>\1(t\in\cB)
        + \frac{2\gamma^2}{n}\sum_{i=1}^n \<\nu^t_i,e^t>\1(t\in\cB)\notag\\
        &\quad + \; \frac{2\gamma}{n}\sum_{i=1}^n(g_i(x^t,\xi^t_i) - g_i(x^t))\1(t\in\cN)
        + \frac{2\gamma^2}{n}\sum_{i=1}^n\|\omega_i^t\|^2\1(t\in\cN)
        + \frac{2\gamma^2}{n}\sum_{i=1}^n\|\nu_i^t\|^2\1(t\in\cB).
    \end{align*}
    Since $\|\omega_i^t\|, \|\nu_i^t\|\le 2M$ we get from \Cref{lem:bound_et}
    \begin{align*}
        &2\gamma c \1(t\in\cN)
        + \frac{2\gamma}{n}\sum_{i=1}^n (f_i(x^t)- f_i(x^*))\1(t\in\cB)
        - \|\wtilde x^t - x^*\|^2
        + \|\wtilde x^{t+1} - x^*\|^2
        - \frac{4\sqrt{1-\delta}}{\delta}\gamma^2M^2
        - 2\gamma^2M^2\notag\\
        \le\;&\frac{2\gamma}{n}\sum_{i=1}^n \<\omega^t_i, \wtilde x^t-x^*>\1(t\in\cN)
        + \frac{2\gamma^2}{n}\sum_{i=1}^n 2M\cdot\frac{2\sqrt{1-\delta}}{\delta}M\1(t\in\cN)\notag\\
        &\quad +\; \frac{2\gamma}{n}\sum_{i=1}^n \<\nu^t_i, \wtilde x^t-x^*>\1(t\in\cB)
        + \frac{2\gamma^2}{n}\sum_{i=1}^n 2M\cdot \frac{2\sqrt{1-\delta}}{\delta}M\1(t\in\cB)
        + \frac{2\gamma}{n}\sum_{i=1}^n(g_i(x^t,\xi^t_i) - g_i(x^t))\1(t\in\cN)\notag\\
        &\quad + \; 
        \frac{2\gamma^2}{n}\sum_{i=1}^n\|\omega_i^t\|^2\1(t\in\cN)
        + \frac{2\gamma^2}{n}\sum_{i=1}^n\|\nu_i^t\|^2\1(t\in\cB).
    \end{align*}
    Rearranging terms, we obtain
    \begin{align*}
        &2\gamma c \1(t\in\cN)
        + \frac{2\gamma}{n}\sum_{i=1}^n (f_i(x^t)- f_i(x^*))\1(t\in\cB)
        - \|\wtilde x^t - x^*\|^2
        + \|\wtilde x^{t+1} - x^*\|^2
        - \frac{12\sqrt{1-\delta}}{\delta}\gamma^2M^2
        - 2\gamma^2M^2\notag\\
        \le \;& \frac{2\gamma}{n}\sum_{i=1}^n \<\omega^t_i, \wtilde x^t-x^*>\1(t\in\cN)
        + \frac{2\gamma}{n}\sum_{i=1}^n \<\nu^t_i, \wtilde x^t-x^*>\1(t\in\cB)
        + \frac{2\gamma}{n}\sum_{i=1}^n(g_i(x^t,\xi^t_i) - g_i(x^t))\1(t\in\cN)\notag\\
        &\quad + \;
        \frac{2\gamma^2}{n}\sum_{i=1}^n\|\omega_i^t\|^2\1(t\in\cN)
        + \frac{2\gamma^2}{n}\sum_{i=1}^n\|\nu_i^t\|^2\1(t\in\cB).
    \end{align*}
    Now we define 
    \begin{equation}\label{eq:def_At}
        A_{t} \eqdef 2\gamma c \1(t\in\cN)
        + \frac{2\gamma}{n}\sum_{i=1}^n (f_i(x^t)- f_i(x^*))\1(t\in\cB)
        - \|\wtilde x^t - x^*\|^2
        + \|\wtilde x^{t+1} - x^*\|^2
        - \frac{12\sqrt{1-\delta}}{\delta}\gamma^2M^2
        - 2\gamma^2M^2.
    \end{equation}
    In the case $t\in\cN$, we have
    \begin{align*}
        A_{t} &= \frac{2\gamma c}{n} 
        - \|\wtilde x^t-x^*\|^2 
        + \|\wtilde x^{t+1} - x^*\|^2 
        - \frac{12\sqrt{1-\delta}}{\delta}\gamma^2M^2
        - 2\gamma^2M^2\notag\\
        &\le \;\frac{2\gamma}{n}\sum_{i=1}^n \<\omega^t_i, \wtilde x^t-x^*>
        + \frac{2\gamma}{n}\sum_{i=1}^n(g_i(x^t,\xi^t_i) - g_i(x^t))
        + \frac{2\gamma^2}{n}\sum_{i=1}^n\|\omega_i^t\|^2.
    \end{align*}
    In the case $t\in\cB$, we have
    \begin{align*}
        A_{t} &= \frac{2\gamma}{n}\sum_{i=1}^n (f_i(x^t)- f_i(x^*)) 
        - \|\wtilde x^t-x^*\|^2 
        + \|\wtilde x^{t+1} - x^*\|^2 
        - \frac{12\sqrt{1-\delta}}{\delta}\gamma^2M^2
        - 2\gamma^2M^2\notag\\
        &\le \;\frac{2\gamma}{n}\sum_{i=1}^n \<\nu^t_i, \wtilde x^t-x^*>
        + \frac{2\gamma^2}{n}\sum_{i=1}^n\|\nu_i^t\|^2.
    \end{align*}
    Following \citet{liu2023high} we define $Z_t$ as follows
    \begin{equation}\label{eq:def_Zt}
        Z_t \eqdef w_t A_t - v_t\|\wtilde x^t-x^*\|^2,
    \end{equation}
    where $w_t$ and $v_t$ will be defined later. Next, we define 
    \begin{equation}\label{eq:def_St}
        S_t \eqdef \sum_{l=t}^{T-1}Z_t.
    \end{equation} Let us define the natural filtration $\cF_t \eqdef\sigma(\xi_0, \dots, \xi_{t-1}).$ We will show by induction that
    \begin{equation*}
        \E{\exp(S_t) \mid \cF_t} \le \exp\left(48M^2\sum_{l=t}^{T-1}w_l\gamma^2 + \frac{8\fvsigma^2}{n\fvN}\sum_{l=t}^{T-1}w_l^2\gamma^2\right).
    \end{equation*}
    The base of induction is trivial for $t=T$ since $S_T = 0.$ Assume that the statement holds for $t\in \{0,\dots, T-1\}.$ We have 
    \begin{align*}
        \E{\exp(S_t)\mid \cF_t} &= \E{\exp(S_{t+1} + Z_t) \mid \cF_t}\notag\\
        &= \E{\E{\exp(S_{t+1} + Z_t \mid \cF_{t+1} } \mid \cF_t}.
    \end{align*}
    We now analyze the inner expectation. Conditioned on $\cF_{t+1}$ we have $Z_t$ fixed. Using the inductive hypothesis, we derive
    \begin{equation*}
        \E{\exp(Z_t + S_{t+1}) \mid \cF_{t+1}} \le \exp(Z_t)\exp\left(48M^2\sum_{l=t+1}^{T-1}w_l\gamma^2\right).
    \end{equation*}
    Therefore, 
    \begin{equation}\label{eq:wfFWQEfwe}
        \E{\exp(Z_t + S_{t+1}) \mid \cF_{t}} \le \E{\exp(Z_t)\mid\cF_t}\exp\left(48M^2\sum_{l=t+1}^{T-1}w_l\gamma^2\right).
    \end{equation}
    From \eqref{eq:def_At}, \eqref{eq:def_Zt}, and assuming that $t\in\cN$ we have the following bound
    \begin{align*}
        \exp(Z_t) &= \exp\left(w_t\frac{2\gamma c}{n} - w_t\|\wtilde x^t - x^*\|^2 + w_t\|\wtilde x^{t+1} - x^*\|^2 -w_t\left(2+\frac{12\sqrt{1-\delta}}{\delta}\right)\gamma^2M^2 -v_t\|\wtilde x^t-x^*\|^2\right) \notag\\
        &\le\; \exp\left(\frac{2\gamma w_t}{n}\sum_{i=1}^n \<\omega_i^t,\wtilde x^t-x^*> 
        + \frac{2\gamma^2 w_t}{n}\sum_{i=1}^n \|\omega_i^t\|^2
        + \frac{2\gamma w_t}{n}\sum_{i=1}^n (g_i(x^t,\xi^t_i)
        - g_i(x^t))\right)\exp(-v_t\|\wtilde x^t-x^*\|^2).\notag
    \end{align*}
    Next, we use Lemma 2.2 from \citet{liu2023high} (with $a=\frac{2\gamma w_t}{n}(\wtilde x^t-x^*)$ and $b^2=\frac{2\gamma^2w_t}{n}$ for the terms with $\omega_i^t$, and with $a=\frac{2\gamma w_t}{n}\cdot 1$ for the terms with $g_i(x^t,\xi^t_i) - g_i(x^t)$) and independence of function and subgradient evaluations
    \begin{align}\label{eq:ndqijwdqwdqwdqw}
        &\E{\exp\left(\frac{2\gamma w_t}{n}\sum_{i=1}^n \<\omega_i^t,\wtilde x^t-x^*> 
        + \frac{2\gamma^2 w_t}{n}\sum_{i=1}^n \|\omega_i^t\|^2
        + \frac{2\gamma w_t}{n}\sum_{i=1}^n (g_i(x^t,\xi^t_i) - g_i(x^t))\right)\mid \cF_t, t\in\cN}\notag\\
        \le \;& \exp\left(n\cdot \left[3\left\{\frac{4\gamma^2 w_t^2}{n^2}\cdot 4M^2\|\wtilde x^t-x^*\|^2
        + \frac{2\gamma^2w_t}{n}\cdot 4M^2\right\} + 2\frac{4\gamma^2w_t^2}{n^2}\frac{\fvsigma^2}{\fvN}\right]\right)\notag\\
        =\;& \exp\left(\frac{48\gamma^2 w_t^2}{n}M^2\|\wtilde x^t-x^*\|^2 +24\gamma^2w_tM^2 + \frac{8\gamma^2w_t^2}{n}\frac{\fvsigma^2}{\fvN}\right).
    \end{align}
    Therefore, from \eqref{eq:wfFWQEfwe} we derive using the definition of $v_t \eqdef \frac{48\gamma^2w_t^2}{n}M^2$
    \begin{align*}
        \E{\exp(S_t) \mid \cF_t} &\le \exp\left(\left[\frac{48\gamma^2w_t^2}{n}M^2 - v_t\right]\|\wtilde x^t-x^*\|^2 
        + 24M^2\sum_{l=t}^{T-1}w_l\gamma^2
        + \frac{8\fvsigma^2}{n\fvN}\sum_{l=t}^{T-1}w_l^2\gamma^2\right)\notag\\
        &= \exp\left(48M^2\sum_{l=t}^{T-1}w_l\gamma^2
        + \frac{8\fvsigma^2}{n\fvN}\sum_{l=t}^{T-1}w_l^2\gamma^2\right).
    \end{align*}
    This concludes the transition step in the case $t\in\cN.$

    Now we move on to the case $t\in\cB.$ The derivations are similar, but we do not have function values. Therefore, instead of 
    \begin{equation*}
        \exp\left(\frac{48\gamma^2w_t^2}{n}M^2\|\wtilde x^t-x^*\|^2 
        + 24M^2w_t\gamma^2
        + \frac{8\fvsigma^2}{n\fvN}w_t^2\gamma^2\right)
    \end{equation*}
    in \eqref{eq:ndqijwdqwdqwdqw} we get
    \begin{equation*}
        \exp\left(\frac{48\gamma^2w_t^2}{n}M^2\|\wtilde x^t-x^*\|^2 
        + 24M^2w_t\gamma^2\right).
    \end{equation*}
    Therefore, the transition step holds in both cases.
    \end{proof}

    \begin{corollary}\label{cor:g5}
        Let $\beta\in (0,1)$ be a failure probability. Suppose the sequence $\{w_t\}$ satisfy the restrictions of \Cref{th:theorem_stoch} and $w_t + \underbrace{\frac{48\gamma^2w_t^2}{n}M^2}_{=v_t} \le w_{t-1}.$ Let the stepsize $\gamma = \frac{\wtilde \gamma}{\sqrt{T}}$. Then with probability at least $1-\beta$
        \begin{align*}
            \sum_{t\in\cN}\gamma c + \sum_{t\in\cB}\gamma(f(x^t) - f(x^*)) 
            &\le C_1\log\frac{1}{\beta}
            + \|\wtilde x^0-x^*\|^2 
            + \gamma^2M^2\left(50+\frac{12\sqrt{1-\delta}}{\delta}\right)T\\
            &\quad + \frac{8\fvsigma^2}{C_1n\fvN}T\gamma^2,
        \end{align*}
        where $C_1 \eqdef \frac{48\wtilde\gamma^2M^2}{n}.$
    \end{corollary}
    \begin{proof}
        Let $T=48M^2\sum_{t=0}^{T-1}w_t\gamma^2 + \frac{8\fvsigma^2}{n\fvN} \sum_{t=0}^{T-1}w_t^2\gamma^2+ \log\frac{1}{\beta}$. By \Cref{th:theorem_stoch} and Markov's inequality, we have
        \begin{align*}
            \Prob(S_0 \ge T) &\le \Prob(\exp(S_0) \ge \exp(T))\notag\\
            &\le \exp(-T)\E{\exp(S_0)}\notag\\
            &\le \exp(-T)\exp\left(48M^2\sum_{t=0}^{T-1}\gamma^2w_t
            + \frac{8\fvsigma^2}{n\fvN}\sum_{t=0}^{T-1}\gamma^2w_t^2\right)\notag\\
            &= \beta.
        \end{align*}
        Note that since $w_t + v_t \le w_{t-1}$ by the assumption of the lemma
        \begin{align*}
            S_0 &= \sum_{t=0}^{T-1} Z_t\notag\\
            &=\sum_{t=0}^{T-1}\bigg[w_t\left(2\gamma c\1(t\in\cN) + 2\gamma(f(x^t) - f(x^*))\1(t\in\cB)\right) - (v_t+w_t)\|\wtilde x^t-x^*\|^2 + w_t\|\wtilde x^{t+1}-x^*\|^2\notag\\
            &\quad \left. -w_t\left(2+\frac{12\sqrt{1-\delta}}{\delta}\right)\gamma^2M^2\right]\notag\\
            &\ge \sum_{t=0}^{T-1}\bigg[2\gamma w_t\left(c\1(t\in\cN) + (f(x^t) - f(x^*))\1(t\in\cB)\right) 
            -\sum_{t=0}^{T-1} \left(w_{t-1}\|\wtilde x^t-x^*\|^2 - w_t\|\wtilde x^{t+1}-x^*\|^2\right)\notag\\
            &\quad \left. -\sum_{t=0}^{T-1}w_t\left(2+\frac{12\sqrt{1-\delta}}{\delta}\right)\gamma^2M^2 \right]\\
            &\ge  \sum_{t=0}^{T-1}2\gamma w_t\left(c\1(t\in\cN) + (f(x^t) - f(x^*))\1(t\in\cB)\right) 
            -w_{0}\| x^0-x^*\|^2 + w_{T-1}\|\wtilde x^{T}-x^*\|^2\notag\\
            &\quad -\sum_{t=0}^{T-1}w_t\left(2+\frac{12\sqrt{1-\delta}}{\delta}\right)\gamma^2M^2\notag\\
            &\ge  \sum_{t=0}^{T-1}2\gamma w_t\left(c\1(t\in\cN) + (f(x^t) - f(x^*))\1(t\in\cB)\right) 
            -w_{0}\| x^0-x^*\|^2 + w_{T-1}\|\wtilde x^{T}-x^*\|^2\notag\\
            &\quad  -\sum_{t=0}^{T-1}w_t\left(2+\frac{12\sqrt{1-\delta}}{\delta}\right)\gamma^2M^2.
        \end{align*}
        Therefore, with a probability of at least $1-\beta$ we have 
        \begin{align*}
            & \sum_{t\in\cN}2\gamma w_tc
            + \sum_{t\in\cB}2\gamma w_t(f(x^t) - f(x^*)) 
            + w_{T-1}\|\wtilde x^{T}-x^*\|^2\notag\\
            \le\;& S_0 
            + w_{0}\| x^0-x^*\|^2 
            + \sum_{t=0}^{T-1}w_t\left(2+\frac{12\sqrt{1-\delta}}{\delta}\right)\gamma^2M^2\\
            \le\;& \log\frac{1}{\beta}
            + w_{0}\| x^0-x^*\|^2 
            + \gamma^2\left(48M^2
            + 2M^2+\frac{12\sqrt{1-\delta}}{\delta}M^2\right)\sum_{t=0}^{T-1}w_t
            + \frac{8\fvsigma^2}{n\fvN}\sum_{t=0}^{T-1}w_t^2\gamma^2.
        \end{align*}
        We need to satisfy the following restrictions on $w_t$:
        \begin{align*}
            &w_t \le \frac{n}{32\gamma^2M^2}\\
            &w_t + \frac{48\gamma^2}{n}w_t^2 \le w_{t-1}.
        \end{align*}
        Let 
        \begin{equation}\label{eq:def_C1}
            C_1 \eqdef \frac{48\wtilde\gamma^2M^2}{n}.
        \end{equation} 
        Then we set $w_{T-1} = \frac{1}{C_1+\frac{48\wtilde\gamma^2M^2}{n}} = \frac{1}{2C_1}$. Next, we set $w_{t-1}$ such that the second inequality holds with equality, namely, 
        \begin{equation*}
            w_{t-1} = w_t + \frac{48\gamma^2M^2}{n}w_t^2 = w_t + \frac{C_1}{T}w_t^2.
        \end{equation*}
        We can show by induction that $w_t \le \frac{1}{C_1 + \frac{C_1}{T}t}.$ Indeed, the base of induction holds by the choice of $w_{T-1}.$ Assume it holds at $t,$ let us show that it holds at $t-1$ as well:
        \begin{align*}
            w_{t-1} &= w_t + \frac{C_1}{T}w_t^2\\
            &\le \frac{1}{C_1 + \frac{C_1}{T}t} 
            + \frac{C_1}{T(C_1+\frac{C_1}{T}t)^2}\\
            &\le \frac{1}{C_1 + \frac{C_1}{T}t} + \frac{(C_1 +\frac{C_1}{T} t) - (C_1 + \frac{C_1}{T}(t-1))}{(C_1+\frac{C_1}{T}(t-1))(C_1+\frac{C_1}{T}t)}\\
            &= \frac{1}{C_1 + \frac{C_1}{T}t}\left(\frac{C_1+\frac{C_1}{T}(t-1)}{C_1+\frac{C_1}{T}(t-1)} + \frac{C_1+\frac{C_1}{T}t - (C_1+\frac{C_1}{T}(t-1))}{C_1+\frac{C_1}{T}(t-1)}\right) = \frac{1}{C_1+\frac{C_1}{T}(t-1)}.
        \end{align*}
        Now we show that the first condition is satisfied as well
        \begin{equation*}
            w_t\gamma^2 = w_t\frac{\wtilde\gamma^2}{T} \le  \frac{1}{\frac{C_1}{T}t}\frac{\wtilde\gamma^2}{T}
            = \frac{1}{\frac{48\wtilde\gamma^2M^2}{n}t}\wtilde\gamma^2
            = \frac{n}{48M^2t} \le \frac{n}{32M^2}.
        \end{equation*}
        Therefore, with a probability at least $1-\beta$, we have 
        \begin{align}
            & \sum_{t\in\cN}2\gamma w_tc
            + \sum_{t\in\cB}2\gamma w_t(f(x^t) - f(x^*)) 
            + w_{T-1}\|\wtilde x^{T}-x^*\|^2\notag\\
            \le\;& \log\frac{1}{\beta}
            + w_{0}\| x^0-x^*\|^2 
            + \gamma^2\left(48M^2 
            + 2M^2+\frac{12\sqrt{1-\delta}}{\delta}M^2\right)\sum_{t=0}^{T-1}w_t
            + \frac{8\fvsigma^2}{n\fvN}\sum_{t=0}^{T-1}w_t^2\gamma^2.
        \end{align}
        Since $w_{T-1} = \frac{1}{2C_1}$ and $\frac{1}{2C_1} \le w_t \le \frac{1}{C_1}$ we have with probability at least $1-\beta$
        \begin{align}
            & \frac{1}{C_1}\sum_{t\in\cN}\gamma c+ \frac{1}{C_1}\sum_{t\in\cB}\gamma(f(x^t) - f(x^*)) 
            \notag\\
            \le\;& \log\frac{1}{\beta}
            + \frac{1}{C_1}\| x^0-x^*\|^2 
            + \gamma^2\left(50M^2+\frac{12\sqrt{1-\delta}}{\delta}M^2\right)\sum_{t=0}^{T-1}w_t
            + \frac{8\fvsigma^2}{n\fvN}\sum_{t=0}^{T-1}w_t^2\gamma^2.
        \end{align}
        We estimate the sums $\sum_{t=0}^{T-1}w_t \le \frac{T}{C_1}$ and $\sum_{t=0}^{T-1}w_t^2 \le \frac{T}{C_1^2}.$ Therefore, we derive
        \begin{align}
            & \frac{1}{C_1}\sum_{t\in\cN}\gamma c+ \frac{1}{C_1}\sum_{t\in\cB}\gamma(f(x^t) - f(x^*)) 
            \notag\\
            \le\;& \log\frac{1}{\beta}
            + \frac{1}{C_1}\| x^0-x^*\|^2 
            + \gamma^2M^2\left(50+\frac{12\sqrt{1-\delta}}{\delta}\right)\frac{T}{C_1}
            + \frac{8\fvsigma^2}{n\fvN}\frac{T}{C_1^2}\gamma^2.
        \end{align}
        Canceling $C_1$ in both sides, we finally obtain
        \begin{align}
            & \sum_{t\in\cN}\gamma c + \sum_{t\in\cB}\gamma(f(x^t) - f(x^*)) 
            \notag\\
            \le\;& C_1\log\frac{1}{\beta}
            + \| x^0-x^*\|^2 
            + \gamma^2M^2\left(50+\frac{12\sqrt{1-\delta}}{\delta}\right)T
            + \frac{8\fvsigma^2}{C_1n\fvN}T\gamma^2.
        \end{align}
    \end{proof}

    \begin{lemma}\label{lem:g6}Let $\beta\in(0,1)$ be the failure probability and $C_1$ be defined as in\eqref{eq:def_C1}. Suppose that the stepsize $\gamma=\frac{\wtilde\gamma}{\sqrt{T}}$ and threshold $c$ satisfy
    \begin{equation}\label{eq:nofenolqef}
        \frac{T}{2}\gamma c > C_1\log\frac{1}{\beta}
            + \| x^0-x^*\|^2 
            + \gamma^2M^2\left(50+\frac{12\sqrt{1-\delta}}{\delta}\right)T
            + \frac{8\fvsigma^2}{C_1n\fvN}T\gamma^2.
    \end{equation}
    Then we have with probability at least $1-\beta$
    \begin{align}\label{eq:mqokenfqokenfq}
            & \sum_{t\in\cN}\gamma c + \sum_{t\in\cB}\gamma(f(x^t) - f(x^*)) 
            \notag\\
            \le\;& C_1\log\frac{1}{\beta}
            + \| x^0-x^*\|^2 
            + \gamma^2M^2\left(50+\frac{12\sqrt{1-\delta}}{\delta}\right)T
            + \frac{8\fvsigma^2}{C_1n\fvN}T\gamma^2.
        \end{align}
        Moreover, assume that \eqref{eq:mqokenfqokenfq} holds. Then $\cB$ is non-empty, i.e. $\overline{x}^T = \frac{1}{|\cB|}\sum_{t\in\cB}x^t$ is well-defined, and one of the following conditions holds
        \begin{enumerate}
            \item $|\cB| \ge \frac{T}{2},$ or 
            \item $\gamma\sum_{t\in\cB} f(x^t) - f(x^*) \le 0.$
        \end{enumerate}
    \end{lemma} 

    \begin{proof}
        Assume that $\cB=\emptyset.$ Then from \Cref{cor:g5} we have that with probability at least $1-\beta$ we have
        \begin{align*}
            T\gamma c 
            &\le C_1\log\frac{1}{\beta}
            + \| x^0-x^*\|^2 
            + \gamma^2M^2\left(50 +\frac{12\sqrt{1-\delta}}{\delta}\right)T
            + \frac{8\fvsigma^2}{C_1n\fvN}T\gamma^2,
        \end{align*}
        This contradicts the assumption of the lemma. Hence, we must have $\cB\neq \emptyset.$ Now assume that \eqref{eq:mqokenfqokenfq} holds. If we have $\gamma\sum_{t\in\cB} f(x^t) - f(x^*) \le 0,$ then the second condition holds. Assume that $\gamma\sum_{t\in\cB} f(x^t) - f(x^*) > 0,$ then from \eqref{eq:mqokenfqokenfq} we obtain
        \begin{align*}
            & \sum_{t\in\cN}\gamma c  \le  C_1\log\frac{1}{\beta}
            + \| x^0-x^*\|^2 
            + \gamma^2M^2\left(50+\frac{12\sqrt{1-\delta}}{\delta}\right)T
            + \frac{8\fvsigma^2}{C_1n\fvN}T\gamma^2.
        \end{align*}
        Assume that $|\cB| < \frac{T}{2},$ this means that $|\cN| \ge \frac{T}{2}.$ Therefore, we have 
        \begin{align*}
            \frac{T}{2}\gamma c\le \sum_{t\in\cN}\gamma c  \le  C_1\log\frac{1}{\beta}
            + \| x^0-x^*\|^2 
            + \gamma^2M^2\left(50+\frac{12\sqrt{1-\delta}}{\delta}\right)T
            + \frac{8\fvsigma^2}{C_1n\fvN}T\gamma^2,
        \end{align*}
        which contradicts \eqref{eq:nofenolqef}. Hence, if $\gamma\sum_{t\in\cB}(f(x^t) - f(x^*)) >0,$ then $|\cB|\ge \frac{T}{2}.$
    
    \end{proof}

    Now we are ready to establish our main convergence result in the stochastic setting.

    \begin{theorem}\label{th:ef14_stoch_appx} Let $\beta\in(0,1)$ be the failure probability and $C_1$ be defined as in \eqref{eq:def_C1}. Suppose that the choice of $\gamma$ and $c$ are chosen such that \eqref{eq:nofenolqef} holds. Then we have with a probability of at least $1-\beta$ that
    \begin{align*}
        f(\overline{x}^T) - f(x^*) &\le \frac{2C_1\log\frac{1}{\beta} + 2\|x^0-x^*\|^2}{\gamma T}
        + 2\gamma M^2\left(50+\frac{12\sqrt{1-\delta}}{\delta}\right)
        + \frac{16\fvsigma^2}{C_1n\fvN}\gamma.
    \end{align*}
    \end{theorem}

    \begin{proof}
        We start by using the results \Cref{lem:g6}. Using the convexity of $f$ and Jensen's inequality we get that if part 2. holds, then with a probability of at least $1-\beta$ we have 
        \[
        f(\overline{x}^T) - f(x^*) \le 0.
        \]
        If part 2. does not hold, then $|\cB| \ge \frac{T}{2}.$ Therefore, from \eqref{eq:nofenolqef} we obtain
        \begin{align*}
            f(\overline{x}^T) - f(x^*) &\le \frac{2}{\gamma T}\left(C_1\log\frac{1}{\beta}
            + \| x^0-x^*\|^2 
            + \gamma^2M^2\left(50+\frac{12\sqrt{1-\delta}}{\delta}\right)
            + \frac{8\fvsigma^2}{C_1n\fvN}T\gamma^2\right)\\
            &= \frac{2C_1\log\frac{1}{\beta}+2\|x^0-x^*\|^2}{\gamma T}
            + 2\gamma M^2\left(50+\frac{12\sqrt{1-\delta}}{\delta}\right)T
            + \frac{16\fvsigma^2}{C_1n\fvN}\gamma.
        \end{align*}
    \end{proof}

    \begin{corollary}
        Let $\beta\in(0,1/2)$ be the failure probability. Let
        \begin{align*}
            R^2 \ge \|x^0-x^*\|^2 + \frac{\fvsigma^2/\fvN}{6M^2}.
        \end{align*}
        If $\gamma = \frac{\wtilde\gamma}{\sqrt{T}} = \frac{R\sqrt{\delta}}{M\sqrt{T}}$, i.e., $\wtilde{\gamma}=\frac{R\sqrt{\delta}}{M}$ and $c = \frac{128RM(1+\log\nicefrac{1}{\beta})}{\sqrt{\delta T}}$,
        then we have with a probability of at least $1-2\beta$
        \begin{align*}
            f(\overline{x}^T) - f(x^*) &\le \frac{MR}{\sqrt{\delta T}}\left(48\log\frac{1}{\beta} + 128\right),\\
            g(\overline{x}^T) &\le \frac{256RM(1+\log\nicefrac{1}{\beta})}{\sqrt{\delta T}}.
        \end{align*}
    \end{corollary}
    \begin{proof}
        First, we check that the stepsize $\gamma$ and threshold $c$ satisfy \eqref{eq:nofenolqef}. We have with $C_1 = \frac{48\wtilde\gamma^2M^2}{n}$
        \begin{align*}
            &\frac{48\frac{R^2\delta}{M^2}M^2}{n}\log\frac{1}{\beta} 
            + \|x^0-x^*\|^2 
            + \frac{R^2\delta}{M^2 T}M^2\left(50 + \frac{12}{\delta}\right)T
            + \frac{8\fvsigma^2}{n\fvN} \frac{n}{48\frac{R^2\delta}{ M^2}M^2} T \frac{R^2\delta}{ M^2 T}\\
            \le\;& \frac{48R^2\delta M^2}{n M^2}\log\frac{1}{\beta} 
            + \|x^0-x^*\|^2
            + 50R^2\delta
            + 12R^2
            + \frac{\fvsigma^2/\fvN}{6M^2}\\
            \le\;& \frac{48R^2\delta}{n}\log\frac{1}{\beta} 
            + \|x^0-x^*\|^2
            + 62R^2
            + R^2\\
            \le\; & 64R^2\log\frac{1}{\beta} + 64R^2.
        \end{align*}
        At the same time, we have
        \begin{align*}
            \frac{T}{2}\gamma c = \frac{T}{2} \frac{R\sqrt{\delta}}{ M \sqrt{T}} \frac{128R M(1+\log\nicefrac{1}{\beta})}{\sqrt{\delta T}} = 64R^2(1+\log\nicefrac{1}{\beta}).
        \end{align*}
        Therefore, with a probability of at least $1-\beta$ we have
        \begin{align}\label{eq:21}
            f(\overline{x}^T) - f(x^*) &\le \frac{48R^2\delta\log\frac{1}{\beta} + 2\|x^0-x^*\|^2}{T}\frac{M\sqrt{T}}{R\sqrt{\delta}}
            + 2\frac{R\sqrt{\delta}}{M\sqrt{T}}M^2\left(50 +\frac{12}{\delta}\right)
            + \frac{16\fvsigma^2/\fvN}{48\frac{R^2\delta}{M^2n}n}\frac{R\sqrt{\delta}}{M\sqrt{T}}\notag\\
            &= (48\delta\log\frac{1}{\beta} + 2)\frac{MR}{\sqrt{\delta T}}
            + 100\frac{RM\sqrt{\delta}}{M\sqrt{T}}
            + 24\frac{RM}{\sqrt{\delta T}}
            + 2\frac{MR\sqrt{\delta}}{\sqrt{\delta T}}\notag\\
            &= \frac{MR}{\sqrt{\delta T}}\left(48\log\frac{1}{\beta} + 128\right).
        \end{align}
        For the constraint violation we have that 
        \begin{align*}
            g(\overline{x}^T) \le \frac{1}{|\cB|}\sum_{t\in\cB}g(x^t) \le \max_{t\in\cB}g(x^t).
        \end{align*}
        Moreover, from \eqref{eq:asmp_subgaussian_func} and \Cref{lem:concentration_lemma}  we have 
        \begin{align*}
            \Prob\left(\left|\sum_{i=1}^n g_i(x^t) - g_i(x^t,\xi^t_i)\right| > (\sqrt{2}+\sqrt{2}b)\sqrt{\sum_{i=1}^n \frac{\fvsigma^2}{\fvN}}\right) \le \exp(-\nicefrac{b^2}{3}).
        \end{align*}
        This implies that 
        \begin{align*}
            \Prob\left(g(x^t) > \frac{1}{n}\sum_{i=1}^ng_i(x^t,\xi^t_i)+ (\sqrt{2}+\sqrt{2}b)\frac{\fvsigma}{\sqrt{n \fvN}}\right) \le \exp(-\nicefrac{b^2}{3}).
        \end{align*}`
        Since for $t\in\cB$ we have $\frac{1}{n}\sum_{i=1}^ng_i(x^t,\xi^t_i) \le c$, then we get
        \begin{align*}
            \Prob\left(g(\overline{x}^T) \le c + (\sqrt{2}+\sqrt{2}b)\frac{\fvsigma}{\sqrt{n \fvN}}\right) \ge 1 - T\exp(-\nicefrac{b^2}{3}).
        \end{align*}
        Choosing $b^2 = 3\log\frac{T}{\beta}$ we obtain
        \begin{align*}
            \Prob\left(g(\overline{x}^T) \le c + (\sqrt{2}+\sqrt{2}b)\frac{\fvsigma}{\sqrt{n \fvN}}\right) \ge 1 - \beta.
        \end{align*}
        Now we choose $\fvN \ge (\sqrt{2}+\sqrt{2}b)^2\frac{\fvsigma^2}{nc^2}$ we obtain
        \begin{align}\label{eq:12}
            \Prob\left(g(\overline{x}^T) \le 2c\right) \ge 1 - \beta.
        \end{align}
        Thus with probability at least $1-2\beta$ we have both \eqref{eq:21} and \eqref{eq:12} hold. The batch-size $\fvN$ depends on the problem constants as follows
        \begin{align*}
            \fvN \ge (\sqrt{2} + \sqrt{2}b)\frac{\fvsigma^2}{nc^2} = \wtilde\cO\left(\frac{\fvsigma^2}{n\frac{R^2M^2}{\delta T}}\right) = \wtilde\cO\left(\frac{\fvsigma^2\delta T}{nR^2M^2}\right).
        \end{align*}
        The number of iterations of \algname{Safe-EF} to converge to $\varepsilon$-accuracy is
        \begin{align*}
            T = \wtilde\cO\left(\frac{R^2M^2}{\delta\varepsilon^2}\right).
        \end{align*}
        Therefore, the batch-size required in the stochastic setting is of order
        \begin{align*}
            \fvN \ge \wtilde{\cO}\left(\frac{\fvsigma^2\delta \frac{R^2M^2}{\delta\varepsilon^2}}{nR^2M^2}\right) = \wtilde\cO\left(\frac{\fvsigma^2}{n\varepsilon^2}\right).
        \end{align*}
        
        This concludes the proof.
        
    \end{proof}

\newpage

\section{Primal-dual Methods}\label{sec:primal_dual}
\paragraph{A short primer on primal-dual methods.}
In \Cref{sec:intro}, we briefly mentioned the primal-dual approach to solving the constrained problem \eqref{eq:problem}, \eqref{eq:constraints}, here we elaborate more on this direction. Consider the Lagrangian with non-negative multiplier $\lambda$:
$$
\mathcal L(x, \lambda) := f(x) + \lambda \, g(x) = \frac{1}{n} \sum_{i=1}^n f_i(x) + \frac{\lambda}{n} \sum_{i=1}^n g_i(x) .
$$
Primal-dual schemes aim to find the saddle-point of this Lagrangian. If Slater's conditions hold, i.e., $f(x)$ is convex and there exists a strictly feasible solution $g(x) < 0$, then the strong duality holds, that is
$$
\min_x \max_{\lambda \geq 0} \, \mathcal L(x, \lambda) = \max_{\lambda \geq 0} \min_x \, \mathcal L(x, \lambda) , 
$$
and general purpose methods for minimizing the primal-dual gap, $\operatorname{Gap}(x^t, \lambda^t) := \max_{\lambda \geq 0} \mathcal L(\lambda, x^t) - \min_{x} \mathcal L(\lambda^t, x)$, can be used. The basic variant of such a scheme is Gradient Descent Ascent:
\begin{align} 
    \text{\algname{Primal-dual}} \qquad &
    \begin{aligned}
        x^{t+1} &= x^t - \gamma_t \, (f^{\prime}(x^t) + \lambda^t g^{\prime}(x^t) ) , \\
        \lambda^{t+1} &= \Pi_{\lambda \geq 0}(\lambda^t + \eta_t \, g(x^{t+1}) ) ,
    \end{aligned}
    \label{eq:Primal-dual}
\end{align}
where  $\{\gamma_t\}$, $\{\eta_t\}$ are primal and dual stepsizes respectively, and $\Pi_{\lambda \geq 0}$ denotes the projection onto the non-negative ray. Similarly to the design of \algname{Safe-EF}, we can write down an error feedback variant of this method for distributed optimization \Cref{alg:pd_ef14_dist_summary}. The intuitive justification of \Cref{alg:pd_ef14_dist_summary} is similar to that of \algname{Safe-EF} in \Cref{sec:proof_bidirectional}. However, a rigorous convergence analysis of $\operatorname{Gap}(x^t, \lambda^t)$ for \Cref{alg:pd_ef14_dist_summary} remains open since even the analysis of \eqref{eq:Primal-dual} (special case of \Cref{alg:pd_ef14_dist_summary} in case of no compression) typically requires the projection step in $x^t$ variable. This is problematic for EF analysis because the virtual iterates $\hat x^t$ defined in \eqref{eq:def_x_hat_t} do not have such simple form anymore. 

\begin{algorithm*}[ht]
\caption{Primal-dual Error Feedback for Constrained Optimization with Bidirectional Compression}
\label{alg:pd_ef14_dist_summary}
\begin{algorithmic}[1]
\State \textbf{Input:} initial point $x^0, \lambda^0 \in \R^d,$ stepsizes $\{\gamma_t\}$, $\{\eta_t\}$, compressors $\cC$ and $\cC_{s}$ at the workers and the server
    \For{$t=0, \ldots, T-1$}
        \For{$i = 1, \ldots, n$}
            \State Compute $h^t_i = f_i^\prime(x^t) + \lambda_t g_i^\prime(x^t) $
            \State Compute $v^{t}_i = \cC(e^{t}_i + h_i^t)$ and send to server
            \State Compute $e^{t+1}_i = e^t_i + h^t_i - v^{t}_i$            
        \EndFor
        \State Compute $v^t = \frac{1}{n}\sum_{i=1}^n v_i^t$ 
        \State Compute $w^{t+1} = w^t - \gamma_t v^t$ 
        \State Compute $x^{t+1} = x^t + \cC_s(w^{t+1} - x^t)$ and send $\cC_s(w^{t+1} - x^t)$ to workers
        \For{$i = 1, \ldots, n$}
        \State Compute $x^{t+1} = x^t + \cC_s(w^{t+1} - x^t)$ 
        \State Compute $g_i(x^{t+1})$ and send to server 
        \Comment{Cheap communication of one float}
        \EndFor
        \State Compute $u^{t+1} = \frac{1}{n}\sum_{i=1}^n g_i(x^{t+1})$
        \State Compute $\lambda^{t+1} = \Pi_{\lambda \geq 0} (\lambda^t + \eta_t u^{t+1})$ 
    \EndFor
\end{algorithmic}	
\end{algorithm*}

\paragraph{Experiments.}
Although a rigorous convergence analysis for \algname{Primal-dual} remains open, we investigate its practical performance through empirical evaluation. We follow the same experimental setup as before and compare \algname{Safe-EF} with \Cref{alg:pd_ef14_dist_summary}, analyzing its sensitivity to different dual initializations $\lambda^0$. We present our results in \Cref{fig:primal-dual}, where we compare the objective and constraint after 500M samples, the number of samples required for \algname{Safe-EF} to converge. As shown, different values of $\lambda^0$ have significant impact on the performance of \algname{Primal-dual}. In contrast, \algname{Safe-EF} that does not require additional tuning of hyperparameters and only slightly underperforms \algname{Primal-dual} when $\lambda^0 = 2$.

\begin{figure}[ht]
    \centering
    \includegraphics{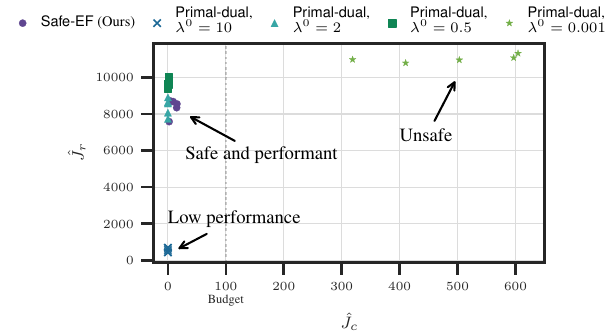}
    \caption{Objective and constraint values of \algname{Safe-EF} compared to \algname{Primal-dual} with different initialization values of $\lambda^0$. Each point represents a distinct experiment trial with a different random seed. \algname{Safe-EF} ensures safety and achieves solid performance without requiring additional hyperparameter tuning.}
    \label{fig:primal-dual}
\end{figure}

\section{Additional Experiments}
\label{sec:additional_experiments}
\paragraph{Cartpole.}
We repeat our safety experiment using the Cartpole environment from Brax \citep{brax2021github}, with the exception of using $\nicefrac{K}{d} = 0.01$ instead of $\nicefrac{K}{d} = 0.1$. As before, we compare \algname{Safe-EF} with \algname{EF14}~\cite{seide20141}, \algname{EF21}~\citep{richtarik2021ef21} and \algname{Parallel-CRPO}. The results are presented in \Cref{fig:cartpole}. Similarly to the experiments with the Humanoid, \algname{Safe-EF} rapidly satisfies the constraints with only a slight performance reduction in the objective. \algname{EF14} outperforms \algname{Safe-EF}, however violates the constraints. Further, \algname{EF21} diverges during the last part of training. Finally, as \algname{Parallel-CRPO} does not employ compression at all, it requires significantly more gigabytes per worker to converge.
\begin{figure}[h!]
    \centering
    \includegraphics{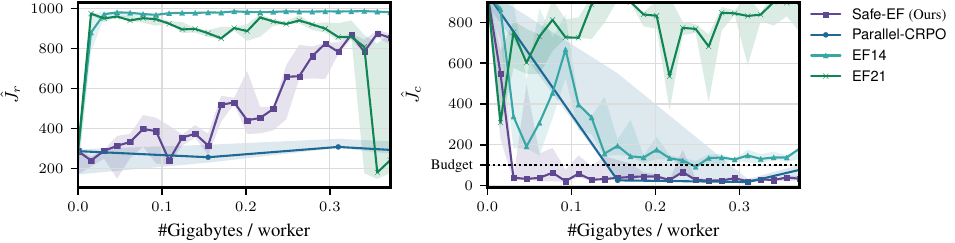}
    \caption{Objective and constraint in the Cartpole environment. \algname{Safe-EF} satisfies the constraints while maintaining competitive performance.}
    \label{fig:cartpole}
\end{figure}

\paragraph{Price of compression.}
We follow the same evaluation protocol used in \Cref{fig:k-to-budget} however now, instead of measuring how many gigabytes are required to reach a certain benchmark performance, we use a fixed sample ``budget'', and evaluate the performance achieved by each algorithm under this budget. Accordingly, we record $\hat{J}_r$ after 100M and 500M samples, corresponding to 4883 and 24415 iterations respectively, for different values of $\nicefrac{K}{d}$. We present the results in \Cref{fig:ablate-k}. As shown, both Top-$K$ and Rand-$K$ perform well under diminishing values of $\nicefrac{K}{d}$ after 500M samples. For a training budget of 100M samples, Top-$K$ significantly surpasses \algname{CGD} and Rand-$K$.
\begin{figure}[h!]
    \centering
    \includegraphics{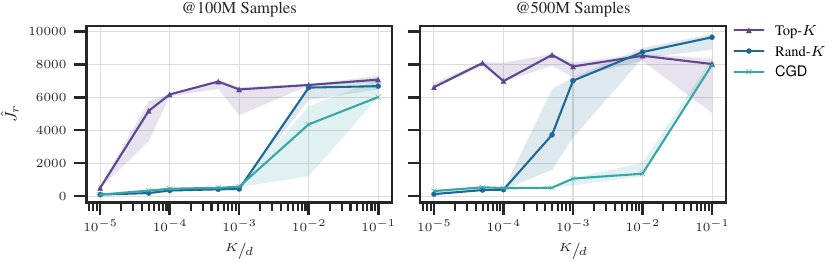}
    \caption{Performance for different compression ratios. \algname{Safe-EF} with Top-$K$ and Rand-$K$ strategies outperform the \algname{CGD} baseline. For a training budget of 500M samples, Top-$K$ reaches adequate performance, even under severe compression.}
    \label{fig:ablate-k}
\end{figure}

\paragraph{Non-distributed baseline.}
We show that \algname{Safe-EF} is able to find a non-trivial policy, by comparing it against \algname{Parallel-CRPO} and its non-distributed variant, \algname{CRPO}, where the latter is trained and evaluated only on the nominal model $p$. We present our results in \Cref{fig:simple}.
\begin{figure}[h!]
    \centering
    \includegraphics{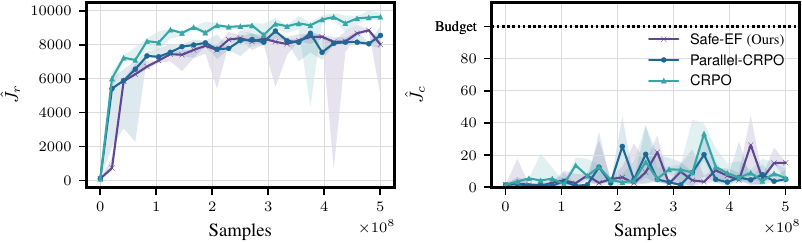}
    \caption{\algname{Safe-EF} performance is only slightly degraded compared to a non-distributed baseline in terms of sample efficiency. However, in the distributed setup, as we observed in \Cref{fig:safety}, \algname{Safe-EF} significantly outperforms \algname{Parallel-CRPO} in communication efficiency.}
    \label{fig:simple}
\end{figure}

\paragraph{Learning curves.}
In \Cref{fig:all-runs} we provide the full learning curves of the experiment trials used for \Cref{fig:ablate-k,fig:k-to-budget}.

\begin{figure}[h!]
    \centering
    \includegraphics{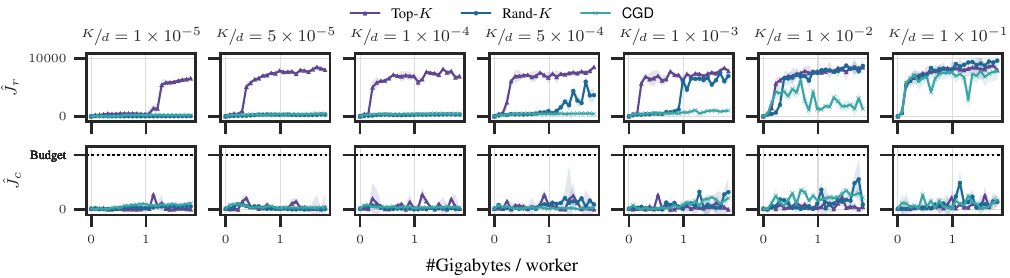}
    \caption{Objective and constraint learning curves for different compression ratio. \algname{Safe-EF} with Top-$K$ outperforms Rand-$K$ and \algname{CGD}, even under small compression values.}
    \label{fig:all-runs}
\end{figure}

\paragraph{Neyman-Pearson classification.}
We test \algname{Safe-EF} on Neyman-Pearson (NP) classification problem following the work of \citet{he2024distributed}. This statistical formulation aims to minimize type II error while enforcing an upper bound on type I error, making it particularly relevant for applications with asymmetric misclassification costs, such as medical diagnosis. The NP classification is
 \begin{equation*}
    \min_xf(x)=\frac{1}{n_0}\sum_{i=1}^{n_0}\phi(h_{x}, z_{i,0}),\text{ s.t. }\quad g(x)=\frac{1}{n_1}\sum_{i=1}^{n_1}\phi(h_{x}, z_{i,1})\le c,    
\end{equation*}
where $f_x$ is a classifier parameterized by $x$ (3 layers MLP with 64 units in each layer and ReLu activation); $\phi$ is a cross-entropy loss; $\{z_{i,0}\}_{i=1}^{n_0}$ and $\{z_{i,1}\}_{i=1}^{n_1}$ are training samples from class 0 and class 1, respectively. The constraint ensures that the classification error for class 1 does not exceed a predefined threshold $c$. Our results are presented in \Cref{fig:np-classification}.

\begin{figure}[h!]
    \centering
    \includegraphics{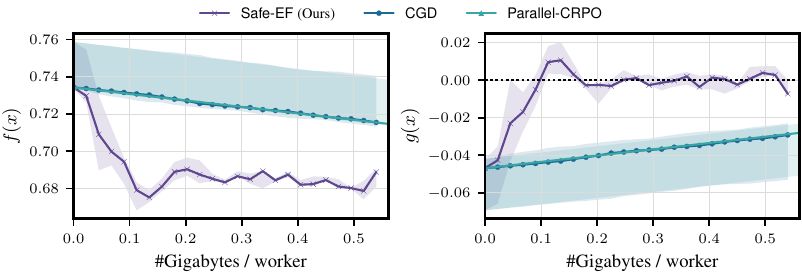}
    \caption{Objective and constraint for Neyman-Pearson classification. Compared to the \algname{CGD} and \algname{Parallel-CRPO} baselines, \algname{Safe-EF} both satisfies the constraint and minimizes the loss while requiring significantly less communication overhead.}
    \label{fig:np-classification}
\end{figure}
This experiment further supports the argument that \algname{Safe-EF} is useful for federated learning by showing its effectiveness in a well-established classification framework.

\section{Additional Details on the Experimental Setup}
\label{sec:additional_experiment_details}

\paragraph{Data generation.} We generate matrices $\{\mA_i\}_{i=1}^n$ and shifts $\{b_i\}_{i=1}^n$ according to \Cref{alg:data_generation}. Here parameter $s$ controls how different the matrices $\mA_i$ are from each other. In our experiments, we vary $s\in\{0.1, 1.0, 10.0\}$ and set $\zeta = 10^{-3}.$

\paragraph{Hyper-parameter tuning for \Cref{sec:synthetic}.}

For all algorithms mentioned in \Cref{sec:synthetic}, we tune the stepsize $\gamma \in \{0.01, 0.003, 0.001, 0.0003, 0.0001, 0.00003\}$. For \algname{EF21M} we tune the momentum parameter $\beta\in\{0.0001, 0.001, 0.01, 0.1, 0.5, 0.9\}$, and for \algname{EControl}, we tune $\eta \in \{0.0001, 0.001, 0.01, 0.1, 0.5, 0.9\}.$ The best hyper-parameters are reported in \Cref{tab:hyperparameters}.

\begin{table*}[t]
    \centering
    \caption{The algorithms' hyperparameters used in the training from \Cref{sec:synthetic}. Here $\gamma$ denotes the stepsize for all algorithms, $\beta$ is the momentum parameter for \algname{EF21M}, and $\eta$ is the control stepsize for \algname{EControl}.}
    \label{tab:hyperparameters}
    \begin{tabular}{c|c|c|c|c|c}
        &
         \algname{Safe-EF} &
         \algname{CGD} &
         \algname{EF21} &
         \algname{EF21M} & 
         \algname{EControl}\\
         \toprule
         $s=0.1$ &
         $\gamma=0.01$ &
         $\gamma = 0.01$ &
         $\gamma = 0.003$ &
         $\gamma = 0.01, \beta = 0.001$ &
         $\gamma = 0.003, \eta = 0.01$ \\ \midrule

         $s=1.0$ &
         $\gamma=0.01$ &
         $\gamma = 0.01$ &
         $\gamma = 0.003$ &
         $\gamma = 0.01, \beta = 0.001$ &
         $\gamma = 0.003, \eta = 0.01$ \\ \midrule

         $s=1.0$ &
         $\gamma=0.003$ &
         $\gamma = 0.01$ &
         $\gamma = 0.001$ &
         $\gamma = 0.001, \beta = 0.1$ &
         $\gamma = 0.001, \eta = 0.1$ \\

         \bottomrule
         
    \end{tabular}
\end{table*}

\begin{figure}[!t]
    \begin{algorithm}[H]
       \caption{Synthetic data generation mechanism}
       \label{alg:data_generation}
        \begin{algorithmic}[1]
       \State {\bfseries Parameters:}  number of nodes $n,$ dimension $d,$ noise scalers $\zeta$ and $s$ 
       \State Generate $\mA \sim \cN(0, \mI) \in \R^{d\times d}$ and $x_0 \sim\cN(0, \mI)\in \R^d$
       \State Normalize $\mA \leftarrow \mA/\|\mA\|_{\rm F}$
       \For{$i=1, \dots, n$}
       \State Generate $\mA_i \sim \cN(0, \mI) \in \R^{d\times d}$
       \State Normalize $\mA_i \leftarrow \mA_i / \|\mA_i\|_{\rm F}$
       \State Shift $\mA_i \leftarrow \mA + s\mA_i$
       \State Sample independently $\xi \sim \cN(0, 1) \in \R^{d}$
       \State Compute $b_i = \mA_i x_0 + \zeta \xi$
       \EndFor
       \State {\bfseries Return} $\{\mA_i, b_i\}_{i=1}^n$
    \end{algorithmic}
    \end{algorithm}
\end{figure}

\paragraph{Humanoid.}
We use the Humanoid environment implementation from Brax \citep{brax2021github} and extend it with an indicator cost function for whenever any one of the joint angles goes outside of a predefined limits. We perturb the dynamics $p_i$ of each worker by sampling the ground's friction coefficient and the gear parameter of the joints' motors. Sampling is done with a uniform distribution, with a symmetric interval centered around the nominal value given in Brax.

\paragraph{Cartpole.}
As with the Humanoid, we use the environment implementation provided by Brax. The cost function is an indicator for whenever the `cart' exceeds a predefined distance from the center position. The dynamics are perturbed in the same fashion as the Humanoid, using a uniform distribution centered around nominal values. However in this experiment, we perturb the mass of the `pole' and the gear parameter of the cart's motor.

\paragraph{Hyper-parameters tuning for \Cref{sec:experiment-rl}.}
As mentioned before, our implementation of \algname{Safe-EF} builds on PPO \citep{schulman2017proximal}. We follow the standard follow the standard implementation provided in Brax, including their default hyper-parameters used for the Humanoid environment. Notably, in all of our experiments, we keep the default value $\gamma = 0.0003$, with Adam as optimizer \citep{kingma2014adam}. In practice, we found the default set of parameters to work well with \algname{Safe-EF}. The only deviation from these parameters is the entropy regularization coefficient, which we set to $0.01$ from $0.001$.

For more specific details, please use our open-source implementation \url{https://github.com/yardenas/safe-ef}.

\end{document}